%% file: main.tex
\documentclass[11pt]{article}

\usepackage[table,xcdraw,dvipsnames]{xcolor}

\usepackage{newpxmath}
\usepackage{newpxtext}
\input{includes}
\usepackage{todonotes}
\usepackage{fullpage}
\allowdisplaybreaks
\input{commands}
\usepackage{tikz}

\usepackage[margin=1in]{geometry}

\allowdisplaybreaks

\usepackage[
backref=true,
natbib=true,
style=trad-alpha,
maxalphanames=8,
sorting=nyt
]{biblatex}
\numberwithin{equation}{section}

\usepackage{subcaption}

\title{Distributionally Robust Linear Regression With Block Lewis Weights}
\author{Naren Sarayu Manoj\thanks{TTIC. Email: \url{nsm@ttic.edu}. Supported by NSF Award ECCS-2216899} \and Kumar Kshitij Patel\thanks{Institute for Foundations of Data Science, Yale University. Email: \url{kumarkshitij.patel@yale.edu}. This work was partly done while the author was a Fellow at the Simons Institute for the Theory of Computing.}}

\begin{document}
\input{body}
\newpage
\printbibliography
\end{document}

%% file: includes.tex
\usepackage{parskip}
\usepackage{anyfontsize}

\usepackage{hyperref}
\usepackage{amsmath}
\usepackage{amsthm}
\usepackage{thm-restate}
\usepackage{mathtools}

\usepackage{booktabs}   
\usepackage{array}      
\usepackage{colortbl}   
\usepackage{makecell}   
\usepackage{xcolor}     

\usepackage{amsfonts}
\usepackage{bm}
\usepackage[inline]{enumitem}
\usepackage{color}
\usepackage{comment}
\usepackage[capitalize]{cleveref}
\usepackage{float}
\usepackage[T1]{fontenc}
\usepackage[textsize=tiny, textwidth=2.5cm, disable]{todonotes}
\makeatletter 
\@mparswitchfalse%
\makeatother
\normalmarginpar 

\usepackage{asymptote}
\usepackage{mdframed}
\usepackage[most]{tcolorbox}
\usepackage{hyperref}
\usepackage{enumitem}
\usepackage{framed}
\usepackage{mdframed}
\usepackage{bbm}

\usepackage{wrapfig}
\usepackage{multicol}
\usepackage{cutwin}

\usepackage[capitalize]{cleveref}

\usepackage{mathrsfs}

\usepackage[T1]{fontenc}

\definecolor{asparagus}{rgb}{0.53, 0.66, 0.42}
\definecolor{sacramentostategreen}{rgb}{0.0, 0.3, 0.15}
\definecolor{teal}{rgb}{0.0, 0.5, 0.5}
\definecolor{forestgreen}{rgb}{0.13, 0.6, 0.13}

\usepackage{hyperref}
\hypersetup{
    colorlinks=true,
    linkcolor=sacramentostategreen,
    filecolor=sacramentostategreen,
    citecolor=teal,
    urlcolor=teal,
}

\newtheorem{theorem}{Theorem}[section]
\newtheorem{mainthm}{Theorem}
\newtheorem{definition}[theorem]{Definition}

\newtheorem{lemma}[theorem]{Lemma}

\newtheorem{remark}[theorem]{Remark}

\newtheorem*{problem*}{Problem}

\theoremstyle{remark}

\Crefname{theorem}{Theorem}{Theorems}
\Crefname{claim}{Claim}{Claims}
\Crefname{lemma}{Lemma}{Lemmas}
\Crefname{proposition}{Proposition}{Propositions}\usepackage[symbol]{footmisc}
\Crefname{corollary}{Corollary}{Corollaries}
\Crefname{definition}{Definition}{Definitions}

\newcommand{\E}{\mathbb{E}}

\newcommand{\R}{\mathbb{R}}

\newcommand{\cD}{\mathcal{D}}
\newcommand{\cE}{\mathcal{E}}

\newcommand{\cG}{\mathcal{G}}

\newcommand{\cO}{\mathcal{O}}

\newcommand{\cS}{\mathcal{S}}

\newcommand{\mA}{\mathbf{A}}
\newcommand{\mB}{\mathbf{B}}

\newcommand{\mD}{\mathbf{D}}

\newcommand{\mI}{\mathbf{I}}

\newcommand{\mM}{\mathbf{M}}

\newcommand{\mW}{\mathbf{W}}

\newcommand{\va}{\bm{a}}
\newcommand{\vb}{\bm{b}}

\newcommand{\vd}{\bm{d}}
\newcommand{\ve}{\bm{e}}

\newcommand{\vg}{\bm{g}}

\newcommand{\vq}{\bm{q}}

\providecommand{\vv}{}
\renewcommand{\vv}{\bm{v}}
\newcommand{\vw}{\bm{w}}
\newcommand{\vx}{\bm{x}}
\newcommand{\vy}{\bm{y}}
\newcommand{\vz}{\bm{z}}

\def\to{\rightarrow}
\def\eps{\varepsilon}

\def\eps{\varepsilon}

\renewcommand{\bar}{\overline}

\usepackage{algorithm}
\usepackage{algpseudocodex}

\newcommand{\argmin}[1]{\underset{#1}{\mathrm{argmin}}\ }

\newcommand{\suchthat}{{\;\; : \;\;}}

\newcommand{\inparen}[1]{\left(#1\right)}             
\newcommand{\inbraces}[1]{\left\{#1\right\}}           
\newcommand{\insquare}[1]{\left[#1\right]}             

\newcommand{\abs}[1]{\ensuremath{\left\lvert #1 \right\rvert}}

\newcommand{\norm}[1]{\ensuremath{\left\lVert #1 \right\rVert}}
\newcommand{\opnorm}[1]{\ensuremath{\left\lVert #1 \right\rVert_{\mathrm{op}}}}
\newcommand{\maxnorm}[1]{\ensuremath{\left\lVert #1 \right\rVert_{\infty}}}

\newcommand{\ip}[1]{\left\langle #1 \right\rangle}

\newcommand{\myfunc}[3]{#1\colon#2\to#3}

\usepackage{nicefrac}

\newcommand{\exv}[1]{\E\insquare{#1}}

\newcommand{\expv}[1]{\mathrm{exp}\inparen{#1}}

\newcommand{\var}[1]{\mathsf{Var}\insquare{#1}}
\newcommand{\cov}[1]{\mathsf{Cov}\insquare{#1}}
\newcommand{\logv}[1]{\log\inparen{#1}}

\makeatletter
\newcommand{\customlabel}[2]{%
   \protected@write \@auxout {}{\string \newlabel {#1}{{#2}{\thepage}{#2}{#1}{}} }%
   \hypertarget{#1}{#2}
}
\makeatother

\newcounter{casenum}

\newcounter{subcasenum}

\newcounter{casenump}

\newcommand{\casep}[2]{
    \ifthenelse{\equal{\value{casenump}}{0}}{
    \vskip.5\baselineskip\par\noindent
    }{}
    {\it Case \arabic{casenump}:} {\it #1}
    \vskip0.1\baselineskip
    \begin{addmargin}[1.5em]{1em}
    #2
    \end{addmargin}
    \addtocounter{casenump}{1}
}

\newcounter{subcasenump}

\newcommand{\lecture}[4]{
  \pagestyle{myheadings}
  \thispagestyle{plain}
  \newpage
  \setcounter{page}{1}
  \noindent
  \begin{center}
    \framebox{
      \vbox{\vspace{2mm}
      \hbox to 6.28in {{\bf Hypergraph Sparsification
                          \hfill Note #4}}
      \vspace{6mm}
      \hbox to 6.28in {{\Large \hfill #1  \hfill }}
      \vspace{5mm}
      \hbox to 6.28in {{#2 \hfill #3}}
      \vspace{2mm}}
    }
  \end{center}
  \markboth{#1}{#1}
  \vspace*{4mm}
}

\newcommand{\rank}[1]{\mathsf{rank}\inparen{#1}}

\usepackage{xfrac}

%% file: commands.tex
\usepackage{cleveref}

\usepackage{ulem}

\newcommand{\xhat}{\widehat{\vx}}
\newcommand{\xstar}{\vx^{\star}}

\newcommand{\vlambda}{\bm{\lambda}}
\newcommand{\mLambda}{\mathbf{\Lambda}}

\newcommand{\Fstar}{F^{\star}}

\newcommand{\gnorm}[2]{\norm{#1}_{\cG_{#2}}}

\newcommand{\lse}{\mathsf{lse}}
\newcommand{\opt}{\mathsf{OPT}}

\newcommand{\lsetext}{\textsc{LogSumExp}}
\newcommand{\ftilde}{\widetilde{f}}
\newcommand{\xtilde}{\widetilde{\vx}}
\newcommand{\fhat}{\widehat{f}}
\newcommand{\oprox}{\cO_{\mathsf{prox}}}

\newcommand{\oms}{\cO_{\mathsf{MS}}}

\definecolor{ForestGreen}{RGB}{34, 139, 34}

\definecolor{forest}{RGB}{0,155,85}

\newcommand{\fixout}{\bgroup\markoverwith{\textcolor{forest}{\rule[0.5ex]{2pt}{0.4pt}}}\ULon}

\renewcommand{\phi}{\varphi}

\newcommand{\smallpar}[1]{\textbf{#1\quad}}

\newif\ifnotes
\makeatletter
\newcommand{\nnote}[1]{\@bsphack\ifnotes{$\ll$\textsf{\color{blue} Naren: { #1}}$\gg$}\fi\@esphack}
\newcommand{\knote}[1]{\@bsphack\ifnotes{$\ll$\textsf{\color{magenta} Kshitij: { #1}}$\gg$}\fi\@esphack}
\makeatother
\notesfalse

%% file: body.tex
\allowdisplaybreaks

\maketitle

\begin{abstract}%
We present an algorithm for the group distributionally robust (GDR) least squares problem. Given $m$ groups, a parameter vector in $\mathbb{R}^d$, and stacked design matrices and responses $\mathbf{A}$ and $\bm{b}$, our algorithm obtains a $(1+\varepsilon)$-multiplicative optimal solution using $\widetilde{O}(\min\{\mathsf{rank}(\mathbf{A}),m\}^{1/3}\varepsilon^{-2/3})$ linear-system-solves of matrices of the form $\mathbf{A}^{\top}\mathbf{B}\mathbf{A}$ for block-diagonal $\mathbf{B}$. Our technical methods follow from a recent geometric construction, block Lewis weights, that relates the empirical GDR problem to a carefully chosen least squares problem and an application of accelerated proximal methods. Our algorithm improves over known interior point methods for moderate accuracy regimes and matches the state-of-the-art guarantees for the special case of $\ell_{\infty}$ regression. We also give algorithms that smoothly interpolate between minimizing the average least squares loss and the distributionally robust loss.
\end{abstract}

\newpage
\tableofcontents
\newpage

\section{Introduction}\label{sec:gp_reg_intro}
\input{intro}



\section{Technical Overview}\label{sec:gp_reg_overview}
In this section, we sketch our proofs for \Cref{mainthm:fair_regression_iteration_complexity} and \Cref{mainthm:gp_regression_iteration_complexity}. 

\paragraph{Notation.} Here and in the rest of the paper, we ignore the dataset size normalization factors $1/\sqrt{n_i}$ as we can fold this into $\mA_{S_i}$ and $\vb_{S_i}$. Additionally, let $f(\vx) \coloneqq \sum_{i=1}^m \norm{\mA_{S_i}\vx-\vb_{S_i}}_2^p$ if $2 \le p < \infty$ and let $f(\vx) \coloneqq \max_{1 \le i \le m} \norm{\mA_{S_i}\vx-\vb_{S_i}}_2$ if $p=\infty$. Note that in the $2 \le p < \infty$ case, we let $f(\vx)$ be the $p$th power of the objective written in \Cref{mainthm:gp_regression_iteration_complexity}; this is to make future calculations easier and makes a difference of only polynomial factors in $p$ in the iteration complexity. Without loss of generality (by rescaling), let $\opt = 1$, where $\opt \coloneqq f(\xstar)$. So, it is enough to get an $\eps$-additive optimal solution $\xhat$. Also without loss of generality, let $\mA$ be such that $\rank{\mA}=d$. For a positive semidefinite $\mM\in\R^{d\times d}$, denote $\norm{\vx}_{\mM} \coloneqq \sqrt{\vx^{\top}\mM\vx}$. As shorthand, for $\vy\in\R^n$, we will often refer to the norm $\gnorm{\vy}{p} \coloneqq \inparen{\sum_{i=1}^m \norm{\vy_{S_i}}_2^p}^{1/p}$ for $p \ge 1$, where with a slight abuse of notation $\vy_{S_i}$ denotes the coordinates of $\vy$ indexed by $S_i$. Finally, in an abuse of notation, for symmetric matrices $\mM$, let $\mM^{-1}$ denote the pseudoinverse of $\mM$.

Recall that many iterative methods for convex optimization can be seen as decomposing a complex problem into a series of simpler subproblems \citep{nw06}. Our algorithms for distributionally robust linear regression follow this pattern, where the simple subproblem resembles
\begin{align}
    \cO(\vq) \coloneqq \min_{\norm{\vx-\vq}_{\mM} \le r_{\vq}}\quad f(\vx)\enspace,\label{eq:gp_prox_intro}
\end{align}
for some positive semidefinite $\mM$ and for some ball radius $r_{\vq}$ which may depend on the query $\vq$. Sub-routines like \eqref{eq:gp_prox_intro} are central to many trust-region methods~\citep{conn2000trust, nw06}, and, importantly when $f$ is the sum of a linear function and a self-concordant barrier, interior point methods derived from the self-concordant barrier framework \footnote{In this case, the matrix $\mM$ is given by the Hessian of the barrier function evaluated at the subproblem's solution.} \citep{nn94}.

With such a subproblem structure in hand, three questions arise. \begin{enumerate*}[label={(\arabic*)}] \item How do we solve the subproblems efficiently? \item How do we combine our subproblem solutions to arrive at our final answer?  \item How do we choose the ``local geometry'' $\mM$ to optimize the iteration complexity we get from the previous two parts? \end{enumerate*} We address these concerns in order in the following discussion.

\subsection{Solving Proximal Subproblems}

For this discussion, let $\mM$ be any positive semidefinite matrix, as the arguments apply for any geometry $\mM$. It will be helpful to assume that $\norm{\cdot}_{\mM}$ is a good approximation to our objective function in the sense that for some \textit{distortion} $\triangle$ that is as close to $1$ as possible, we have
\begin{align*}
    \text{for all } \vx \in \R^d:\quad\quad \norm{\vx - \vb}_{\mM} \le \inparen{\sum_{i=1}^m \norm{\mA_{S_i}\vx-\vb_{S_i}}_2^p}^{\frac{1}{p}} \le \triangle\norm{\vx -\vb}_{\mM}\enspace.
\end{align*}
Here, we discuss how to solve problems of the form \eqref{eq:gp_prox_intro} for a fixed query $\vq$. Our strategy follows two general steps. First, we establish some form of local stability for $\nabla^2 f(\vx)$ within the ball we are solving in, i.e., we want $\nabla^2 f(\vx)$ to not change too much inside the ball $\inbraces{\vx\in\R^d \suchthat \norm{\vx-\vq}_{\mM} \le r_{\vq}}$. Second, we use this to demonstrate that an appropriate second-order algorithm exhibits a favorable convergence rate to an approximate solution for our subproblem. We handle the $p=\infty$ and $2 \le p < \infty$ cases separately below.

\subsubsection{The Robust Case \texorpdfstring{($p=\infty$)}{(p=infinity)}.}
\label{sec:gp_reg_overview_robust_hessian}

Unfortunately, since $f$ is not even differentiable (it is the pointwise maximum of Euclidean norms, each of which is also not differentiable), we cannot directly argue about the stability of $\nabla^2 f(\vx)$. We therefore first need to find some surrogate objective $\ftilde$ so that:
\begin{enumerate}
    \item The approximation error $\maxnorm{\ftilde -f}$ is small;
    \item The surrogate objective $\ftilde$ is smooth in $\norm{\cdot}_{\mM}$ in such a way that we can solve the proximal subproblems fast.
\end{enumerate}
To smoothen $f(\vx)$, we use the family of objectives parameterized by $\beta,\delta$
\begin{align}
    \ftilde_{\beta,\delta}(\vx) \coloneqq \beta\logv{\sum_{i=1}^m \expv{\frac{\sqrt{\delta^2 + \norm{\mA_{S_i}\vx-\vb_{S_i}}_2^2}-\delta}{\beta}}}\enspace.\label{eq:intro_smooth_fair_objective}
\end{align}
This can be seen as composing the softmax function with temperature $\beta$ with ``inner functions'' $\sqrt{\delta^2 + \norm{\mA_{S_i}\vx-\vb_{S_i}}_2^2} - \delta$. It is straightforward to show that for all $\vx\in\R^d$, $\abs{\ftilde_{\beta,\delta}(\vx) - f(\vx)} \le \beta\log m + \delta$. So, setting $\beta = \eps/4\log m$ and $\delta = \eps/4$, it is sufficient to optimize $\ftilde_{\beta,\delta}$ up to $\eps/2$ additive error to get an $\eps$-additive suboptimal solution to our original objective. Furthermore, we prove that $\ftilde_{\beta,\delta}$ is $O(1/\beta+1/\delta)$-smooth in the norm $\gnorm{\mA\vx}{\infty} \coloneqq \max_{1 \le i \le m} \norm{\mA\vx}_2$. Thus, if $\norm{\cdot}_{\mM}$ is a good approximation to $\gnorm{\mA\vx}{\infty}$, we will get that $\ftilde_{\beta,\delta}$ is also smooth in the norm $\norm{\vx}_{\mM}$.

Next, \citet{msbacon} show that if $\ftilde_{\beta,\delta}$ satisfies a higher-order smoothness condition called \textit{quasi-self-concordance} with respect to the norm $\norm{\cdot}_{\mM}$, then we can get the required Hessian stability for a \textit{fixed} $r_{\vq} = \Theta(1/\eps)$ (in particular, $r_{\vq}$ does not depend on $\vq$ here). To clarify, we define quasi-self-concordance as follows.
\begin{definition}[Quasi-self-concordance, adapted from {\cite[Appendix A]{ksj18}}]
\label{defn:qsc}
Let $\myfunc{f}{\R^k}{\R}$. We say that $f$ is $\nu$-quasi-self-concordant in the norm $\norm{\cdot}$ if for all vectors $\vy\in\R^k$, directions $\vd\in\R^k$, and $t\in\R$, we have
\begin{align*}
    \abs{\inparen{\frac{d}{dt}}^3 f(\vy+t\vd)} \le \nu\norm{\vd}\inparen{\frac{d}{dt}}^2 f(\vy+t\vd)\enspace.
\end{align*}
\end{definition}
Then, \cite{msbacon} shows how to leverage this Hessian stability to implement \eqref{eq:gp_prox_intro} with low linear-system-solve iteration complexity. However, previously, it was only shown that the composition of the softmax function with linear functions is quasi-self-concordant. So, it was unknown whether composing softmax with other functions could also be quasi-self-concordant. 

To resolve this, we prove a much more general composition result, which to the best of our knowledge was not known prior to this work and may be of independent interest. It essentially states that if we compose the softmax function with any combination of ``inner'' functions that are quasi-self-concordant, the resulting function is also quasi-self-concordant. For a more formal statement, see \Cref{lemma:composed_qsc}. 

\begin{restatable*}[Composing softmax with quasi-self-concordant functions]{lemma}{composedqsc}
\label{lemma:composed_qsc}
Let $\norm{\cdot}$ be an arbitrary norm and $h_1,\dots,h_m$ be such that $\myfunc{h_i}{\R^d}{\R}$. Let $h$ be the vector formed by concatenating the results of $h_1, \dots, h_m$. Additionally, let $h_1, \dots, h_m$ be such that for all $1 \le i \le m$ and for all $\vy,\vd\in\R^m$ and $t\in\R$,
\begin{align*}
    &\inparen{\frac{d}{dt}} h_i(\vy + t\vd) \le \norm{\vd} & \text{(Lipschitzness)}\\
    &\abs{\inparen{\frac{d}{dt}}^3 h_i(\vy + t\vd)} \le \nu\norm{\vd}\inparen{\frac{d}{dt}}^2 h_i(\vy + t\vd) & \text{(quasi-self-concordance)}.
\end{align*}
Then, for all $\vy,\vd\in\R^m$ and all $t\in\R$, we have
\begin{align*}
    \abs{\inparen{\frac{d}{dt}}^3 \beta\logv{\sum_{i=1}^m \expv{\frac{h_i(\vy+t\vd)}{\beta}}}} \le \inparen{\frac{16}{\beta}+\nu}\norm{\vd}\inparen{\frac{d}{dt}}^2 \lse_{\beta}(h(\vy + t\vd)).
\end{align*}
\end{restatable*}

Hence, to show the requisite Hessian stability, we use the following steps. We show that the ``inner'' functions for \eqref{eq:intro_smooth_fair_objective}, $\sqrt{\delta^2+\norm{\mA_{S_i}\vx-\vb_{S_i}}_2^2}-\delta$, are each $O(1/\delta)$-quasi-self-concordant in the norm $\norm{\mA_{S_i}\vx}_2$. So, we can apply our composition result \Cref{lemma:composed_qsc} to prove that $\ftilde_{\beta,\delta}$ is $O(1/\beta + 1/\delta)$-quasi-self-concordant in the norm $\max_{i\in[m]}\norm{\mA_{S_i}\vx}_2$. Again, assuming that $\norm{\cdot}_{\mM}$ is a good approximation to $\gnorm{\cdot}{\infty}$, we will get that $\ftilde_{\beta,\delta}$ is quasi-self-concordant in $\norm{\vx}_{\mM}$ as well. 

With these analytic inequalities in hand, we can finally apply the recipe given in \cite{msbacon} and get our subproblem solver for the $p=\infty$ case.

\subsubsection{The Interpolating Case (\texorpdfstring{$2 \le p < \infty$}{2 <= p < infinity}).}

Instead of explicitly constraining $r_{\vq}$ like in the $p=\infty$ case, we regularize our movement from $\vq$ in the norm $\norm{\cdot}_{\mM}$. Specifically, the subproblem we solve for any query $\vq$ is
\begin{align}
    \argmin{\vx\in\R^d} f(\vx) + ep^p\norm{\vx-\vq}_{\mM}^p\enspace.\label{eq:gp_prox_intro_reg}
\end{align}
This is the natural generalization of the proximal problem that \cite{jls21} use to get their results for $\ell_p$ regression, and the outline of our solver for these subproblems is similar to what \cite{jls21} use for this special case (see their Section 4).

However, we go a step further and show how to obtain approximate stationary points to \eqref{eq:gp_prox_intro_reg} instead of just getting a small objective value. This is because the acceleration scheme we use to iterate subproblem solutions to get our final answer $\xhat$ requires us to obtain an approximate stationary point for \eqref{eq:gp_prox_intro_reg}. The main new technical tool we develop for this purpose is a form of strong convexity for functions of the form $\norm{\vy}_2^p$ for $\vy\in\R^k$ for any $k\ge 1$. See \Cref{lemma:strong_convexity_component}.

\begin{restatable*}[Strong convexity of $\norm{\vy}_2^p$]{lemma}{lemmastrongconvexitycomponent}
\label{lemma:strong_convexity_component}
Let $\vv \in \R^k$ for $k\ge 1$. For any $\triangle \in \R^k$, we have
\begin{align*}
    \norm{\vv+\triangle}_2^p \ge \norm{\vv}_2^p + p\norm{\vv}_2^{p-2}\ip{\vv,\triangle}+\frac{4}{2^p}\norm{\triangle}_2^p\enspace.
\end{align*}
\end{restatable*}

With \Cref{lemma:strong_convexity_component}, we can argue about the strong convexity of $\norm{\vx-\vq}_{\mM}^p$, which means that we can convert an approximately optimal solution to \eqref{eq:gp_prox_intro_reg} in function value to one that is approximately optimal in parameter space as well. We combine this with a local gradient Lipschitzness property of the objective \eqref{eq:gp_prox_intro_reg} to get our approximate stationary point, which is enough for our purposes. The local gradient Lipschitzness property itself follows from a form of Hessian stability that we show for the objective \eqref{eq:gp_prox_intro_reg}. See \Cref{lemma:gp_hessian_stable}.

Finally, to obtain an approximately optimal solution to \eqref{eq:gp_prox_intro_reg} in function value, we again apply the Hessian stability property to conclude that \eqref{eq:gp_prox_intro_reg} is relatively smooth and relatively strongly convex in a simpler reference function. We show how to solve optimization problems in this reference function up to an approximate optimality that is sufficient for the rest of our applications -- this requires a mild modification of the standard mirror descent analysis, and we do this in \Cref{sec:mirror_descent}. Combining all of these building blocks gives us our subproblem solver for the $2 \le p < \infty$ case.

\subsection{Iterating Proximal Calls}\label{sec:iter_prox_calls}

We now discuss the second item. Recall that we think of $\cO(\vq)$ as answering a proximal problem for the query $\vq$. It is not hard to show that under reasonable conditions on $f$ and on the structure of the subproblems, we can iterate calls to $\cO(\vq)$ to optimize $f$ (see, e.g., \cite[Appendix A]{msbacon}). This conceptually simple approach will already give us guarantees of the form ${\norm{\vx_0-\xstar}_{\mM}}/{\eps}$ for the problems we study.

But we can do better. An acceleration framework originally due to \citet{ms13} and generalized/refined in subsequent works \citep{bjlls19,msbacon,chjjs22} gives a recipe to iterate calls of $\cO(\vq)$ to optimize the original function $f$. From these, the iteration complexity we need for an $\eps$-additive solution with an initialization $\vx_0$ and optimum $\xstar$ is roughly $\inparen{{\norm{\vx_0-\xstar}_{\mM}}/{\eps}}^{2/3}$ (see \Cref{thm:optimal_ms_acceleration} for a more formal statement). This cosmetically resembles the rate we get in \Cref{mainthm:fair_regression_iteration_complexity}. To get something that looks like our rate for \Cref{mainthm:gp_regression_iteration_complexity}, we use our new strong convexity lemma  (\Cref{lemma:strong_convexity_component}). With this, we can demonstrate that after a sufficient number of iterations, we have $\norm{\vx_t-\xstar}_{\mM} \le 0.5\norm{\vx_0-\xstar}_{\mM}$. Therefore, repeating this argument yields a high-accuracy solution, as required. 

Interestingly, our algorithm for the $2 \le p < \infty$ case employs a form of the accelerated scheme developed in \cite{chjjs22}, which does not require solving an implicit equation for the query point, thereby improving upon the results from \cite{jls21} for $\ell_p$ regression. It would be practically relevant to obtain this for the $p=\infty$ case (in \Cref{sec:optimal_ms_acceleration}, we discuss a technical challenge in obtaining this).

\subsection{The Geometry of the Proximal Subproblems and Block Lewis Weights}
\label{sec:gp_reg_overview_blw}

At this point, we have the tools we need to get rates of the form $O\inparen{\inparen{{\norm{\vx_0-\xstar}_{\mM}}/{\eps}}^{2/3}}$ for the robust objective (\Cref{mainthm:fair_regression_iteration_complexity}) and of the form $O\inparen{{\norm{\vx_0-\xstar}_{\mM}}^{(p-2)/(3p-2)}}$ for the interpolating objective (\Cref{mainthm:gp_regression_iteration_complexity}). From this, we see that the rates depend on the geometry $\mM$ that we impose on our problem. Our goal in this section is to choose this geometry $\mM$.

Observe that when we solve \eqref{eq:gp_prox_intro}, we are solving an optimization problem over the sublevel sets $\inbraces{\vx\suchthat \norm{\vx}_{\mM}\le r_{\vq}}$ -- these are ellipsoids. Now, consider choosing the $\ell_2$ geometry that best approximates our loss function. Specifically, recall that earlier in the section, we stated that for some distortion $\triangle \ge 1$ that is as close to $1$ as possible, we want
\begin{align*}
    \text{for all } \vx \in \R^d:\quad\quad \norm{\vx - \vb}_{\mM} \le \inparen{\sum_{i=1}^m \norm{\mA_{S_i}\vx-\vb_{S_i}}_2^p}^{\frac{1}{p}} \le \triangle\norm{\vx -\vb}_{\mM}\enspace.
\end{align*}
To see what kinds of distortion guarantees we can hope for, let us see what happens when we choose the most ``obvious'' geometry. By relating $\ell_2^m$ to $\ell_p^m$, we get
\begin{align*}
    \inparen{\sum_{i=1}^m \norm{\mA_{S_i}\vx-\vb_{S_i}}_2^2}^{\frac{1}{2}} \le \inparen{\sum_{i=1}^m \norm{\mA_{S_i}\vx-\vb_{S_i}}_2^p}^{\frac{1}{p}} \le m^{\frac{1}{2}-\frac{1}{p}}\inparen{\sum_{i=1}^m \norm{\mA_{S_i}\vx-\vb_{S_i}}_2^2}^{\frac{1}{2}},
\end{align*}
and notice that $\inparen{\sum_{i=1}^m \norm{\mA_{S_i}\vx-\vb_{S_i}}_2^2}^2 = \norm{\mA\vx-\vb}_2$. Thus, setting $\mM = \mA^{\top}\mA$ (which is what we call the na\"ive geometry in \Cref{table:compare}) gives us our basic rate of $m^{1/3}\eps^{-2/3}$ in the setting of \Cref{mainthm:fair_regression_iteration_complexity} and $m^{(p-2)/(3p-2)}$ in the setting of \Cref{mainthm:gp_regression_iteration_complexity}.

But, there exists an improvement over above na\"ive geometry. Note our loss function is a norm on $\R^d$ -- in particular, we can check that for $\vy\in\R^n$, the functions $\gnorm{\vy}{p}=\inparen{\sum_{i=1}^m \norm{\vy_{S_i}}_2^p}^{1/p}$ for $1 \le p \le \infty$ are norms. Now, recall John's theorem, a fundamental result in high-dimensional convex geometry.
\begin{theorem}[John's theorem, \cite{john1948}]
For any symmetric convex body $K \subset \R^d$, let $\cE(K)$ denote the ellipsoid of maximum volume contained within $K$. Then, we have
\begin{align*}
    \cE(K) \subseteq K \subseteq \sqrt{d}\cdot \cE(K)\enspace.
\end{align*}
Moreover, the $\sqrt{d}$ is worst-case optimal (e.g. let $K$ be the unit $\ell_{\infty}$ ball).
\end{theorem}
It is easy to see that sublevel sets of norms, i.e., sets of the form $\inbraces{\vx\in\R^d\suchthat \norm{\vx}\le 1}$, are symmetric convex bodies. Hence, using John's theorem, we see that for our normed losses, there exists $\mM$ that achieves distortion $\triangle\le\sqrt{d}$. From this, it is easy to see that there exists $\mM$ for which we can guarantee $\norm{\vx_0-\xstar}_{\mM} \lesssim \sqrt{d}$. Plugging this into the guarantees from the previous subsections, we get that if we choose the $\mM$ from John's theorem, and then switch based on whether $m \le d$, we get exactly the rates quoted in \Cref{mainthm:fair_regression_iteration_complexity} and  \Cref{mainthm:gp_regression_iteration_complexity}.

However, as written, this is only an existence result. To make this useful for us and actually find $\mM$, we need an algorithm to calculate John's ellipsoid for the level sets of our losses (or some other ellipsoid that gets an even better approximation factor). To this end, a result of \citet{mo23} gives us an efficient algorithm to find this $\ell_2$ geometry for the loss families we consider.

\begin{theorem}[Combining Lemmas 5.6, 5.8, Equation (1.8) from {\cite{mo23}}]
\label{thm:gp_regression_blw_intro}
Let $p\ge 2$. There exists an algorithm that finds a positive diagonal matrix $\mW \in \R^{n \times n}$ such that for all $\vx\in\R^d$ and all $c\in\R$, we have
\begin{align*}
    \frac{\norm{\mW^{\frac{1}{2}-\frac{1}{p}}\inparen{\mA\vx-c\vb}}_2}{(2(\rank{\mA}+1))^{\frac{1}{2}-\frac{1}{p}}} \le \inparen{\sum_{i=1}^m \norm{\mA_{S_i}\vx-c\vb_{S_i}}_2^p}^{\frac{1}{p}} \le \norm{\mW^{\frac{1}{2}-\frac{1}{p}}\inparen{\mA\vx-c\vb}}_2\enspace.
\end{align*}
The algorithm runs in $O(\log m)$ linear-system-solves in matrices of the form $\mA^{\top}\mD\mA$ for positive diagonal matrices $\mD$.
\end{theorem}

The diagonal entries of matrix $\mW$ are called \textit{block Lewis weights}. This is a generalization of Lewis weights, and both objects have been used previously for various matrix approximation problems \citep{blm89,mmwy21,jls22,jlls23,mo23}. Furthermore, Lewis weights are central to improvements in the iteration complexities for linear programming and vanilla $\ell_p$ regression \citep{ls19,jls21}. We go into greater detail about block Lewis weights in \Cref{sec:blw_to_l2}. 

Additionally, notice that the distortion of $O(\rank{\mA}^{1/2-1/p})$ guaranteed by \Cref{thm:gp_regression_blw_intro} is optimal. To see this, let $\mA \in \R^{n \times d}$ be such that for $i\in[d]$, row $\va_i=\ve_i$, where $\ve_i$ is the $i$th standard basis vector. Then, for all $d+1 \le i \le n$, let $\va_i=0$. In words, $\mA$ is the $d$-dimensional identity matrix atop a large matrix of all $0$s. It is easy to see that for any $p$, we have $\norm{\mA\vx}_p = \norm{\vx}_p$, and the best distortion we can get for relating $\norm{\vx}_p$ to any $d$-dimensional $\ell_2$ norm is $d^{\abs{1/2-1/p}}$.

With \Cref{thm:gp_regression_blw_intro} and its near optimality in hand, we choose $\mM = \mA^{\top}\mW^{1-\frac{2}{p}}\mA$ if $\rank{\mA} \le m$ and $\mM = \mA^{\top}\mA$ if $\rank{\mA} \ge m$ (recall that in the latter case, we get a $\sqrt{m}$ distortion for free from relating $\ell_2^m$ to $\ell_{\infty}^m$). Combining this with the results from the previous two subsections gives us \Cref{mainthm:fair_regression_iteration_complexity} and \Cref{mainthm:gp_regression_iteration_complexity}.

\subsection{Algorithm for Distributionally Robust Regression}

In \Cref{alg:fair_regression}, we present pseudocode for the algorithm that yields the guarantee in \Cref{mainthm:fair_regression_iteration_complexity}. We compare the empirical performance of this algorithm against other baselines mentioned in \Cref{sec:gp_reg_intro} and examine the effect of Lewis weights in \Cref{app:exps}.

\begin{algorithm}[H]
\caption{\textsf{MinMaxRegression}: optimizes \eqref{eq:main_objective} to $(1+\eps)$-multiplicative error}
\label{alg:fair_regression}
\begin{algorithmic}[1]
\Require Regression problems $(\mA_{S_1},\vb_{S_1}),\dots,(\mA_{S_m},\vb_{S_m})$, accuracy $\eps > 0$
\State Using \cite[Algorithm 2]{mo23} with input $\insquare{\mA \vert \vb}$, find nonnegative diagonal $\mW$ and weights $w_1,\dots,w_m$ such that for all $j \in S_i$, $\mW[j][j] = w_i$ and for all $\vx\in\R^{d}$ and $c \in \R$,
\begin{align*}
    \gnorm{\mA\vx-c\vb}{\infty} \le \norm{\mW^{1/2}\mA\vx-c\mW^{1/2}\vb}_2 \le \sqrt{2(\rank{\mA    }+1)}\gnorm{\mA\vx-c\vb}{\infty}.
\end{align*}\label{line:blw}
\If{$\sum_{i=1}^m w_i \ge m$} \Comment{$\rank{\mA}+1\le\sum_{i=1}^m w_i \le 2(\rank{\mA}+1)$}
    \State Reset $\mW = \mI_n$.
\EndIf
\State Let $\vx_0 = \inparen{\mA^{\top}\mW\mA}^{-1}\mA^{\top}\mW\vb$. \Comment{$\vx_0 \coloneqq \argmin{\vx\in\R^d} \norm{\mW^{1/2}\mA\vx-\mW^{1/2}\vb}_2$.}
\State Let
\begin{align*}
    \ftilde_{\beta,\delta}(\vx) \coloneqq \beta\logv{\sum_{i=1}^m \expv{\frac{\sqrt{\delta^2 + \norm{\mA_{S_i}\vx-\vb_{S_i}}_2^2}-\delta}{\beta}}}
\end{align*}
where $\beta = \frac{\eps}{4\log m}$ and $\delta = \frac{\eps}{4}$. \Comment{A family of smoothenings of the objective.}
\State Let $\fhat(\vx) \coloneqq \ftilde_{\eps/4\log m, \eps/4}(\vx) + \frac{\eps}{1000\min\inbraces{\rank{\mA},m}}\norm{\mW^{1/2}\mA(\vx-\vx_0)}_2^2$.
\State Using \cite[Algorithm 3]{msbacon}, implement a $\inparen{\frac{C}{\min\inbraces{\rank{\mA},m}},\frac{C}{\eps}}$-ball optimization oracle for $\fhat$, where $C$ is a universal constant. \Comment{Iteration complexity guaranteed by \Cref{lemma:obj_smooth_qsc}}
\State Using \cite[Algorithm 2]{msbacon}, implement a $\frac{1}{2}$-MS oracle for $\fhat$.
\State Run \cite[Algorithm 1]{msbacon} for $\widetilde{O}\inparen{\frac{\min\inbraces{\rank{\mA},m}^{1/3}\logv{\frac{d}{\eps}}}{\eps^{2/3}}}$ iterations using the MS oracle from the previous line and with initial point $\vx_0$ and final point $\xhat$.
\State \Return $\xhat$
\end{algorithmic}
\end{algorithm}




\input{other_proofs}

\input{mirror_descent}

\input{improved_ms}

\section{Minimizing the Distributionally Robust Loss}\label{sec:gp_robust_proof}

The goal of this section is to prove \Cref{mainthm:fair_regression_iteration_complexity}. We break up the proof into parts as described in \Cref{sec:gp_reg_overview}. We structure the section as follows. In the remainder of this subsection, we present \Cref{alg:fair_regression}, our algorithm for minimizing the distributionally robust loss. In \Cref{sec:gp_robust_approx}, we introduce our smooth approximation for the objective \eqref{eq:main_objective} and show that it is a good additive approximation (this is a standard argument, but we include it as it provides crucial intuition). 

As the main difficulty of the proof in \Cref{mainthm:fair_regression_iteration_complexity} is to establish a Hessian stability for our surrogate loss, we devote the bulk of this section to proving this. Recall that in \Cref{sec:gp_reg_overview_robust_hessian}, we claimed that a higher-order smoothness condition called \textit{quasi-self-concordance} gives us the needed Hessian stability -- in fact, this follows from \cite[Lemma 11]{msbacon}. In light of this, it suffices to demonstrate that our surrogate loss is quasi-self-concordant. 

In \Cref{sec:lse_calculus}, we work out some calculus facts related to the softmax function. In particular, it is in \Cref{sec:lse_calculus} that we prove the general composition result \Cref{lemma:composed_qsc} that states that if we take the softmax of several quasi-self-concordant functions, then the resulting function is also quasi-self-concordant. In \Cref{sec:lse_calculus_special}, we apply this composition fact to prove that our surrogate objective is quasi-self-concordant. Finally, in \Cref{sec:gp_robust_analysis_combining}, we combine these building blocks with the acceleration framework in \cite{msbacon} and complete the proof of \Cref{mainthm:fair_regression_iteration_complexity}.

\subsection{Smoothly Approximating the Objective}
\label{sec:gp_robust_approx}

Recall that for $\vy\in\R^n$, let $\gnorm{\vy}{\infty} \coloneqq \max_{1 \le i \le m} \norm{\vy_{S_i}}_2$, where for $\vy\in\R^n$ we let $\vy_{S_i}$ refer to the vector in $\R^{n_i}$ indexed by the indices in $S_i$. Also, for $\vy\in\R^m$, let $\lse_{\beta}(\vy)$ refer to the function
\begin{align*}
    \lse_{\beta}(\vy) \coloneqq \beta\logv{\sum_{i=1}^m \expv{\frac{y_i}{\beta}}}.
\end{align*}
At a high level, our algorithm will minimize the function
\begin{align*}
    \ftilde_{\beta,\delta}(\vx) \coloneqq \beta\logv{\sum_{i=1}^m \expv{\frac{\sqrt{\delta^2 + \norm{\mA_{S_i}\vx-\vb_{S_i}}_2^2}-\delta}{\beta}}}
\end{align*}
for appropriate choices of the parameters $\beta$ and $\delta$. This choice of smoothening is natural because of the following approximation statement -- see \Cref{lemma:smooth_approx}.
\begin{lemma}
\label{lemma:smooth_approx}
For all $\vx \in\R^d$, we have
\begin{align*}
    \abs{\ftilde_{\beta,\delta}(\vx)-\gnorm{\mA\vx-\vb}{\infty}} \le \beta\log m + \delta.
\end{align*}
\end{lemma}
\begin{proof}[Proof of \Cref{lemma:smooth_approx}]
These guarantees are well-known, but we prove them anyway for the sake of self-containment. We first prove that for any $\vv\in\R^m$, we have 
\begin{align*}
    \max_{1 \le i \le m} v_i \le \lse_{\beta}(\vv) \le \max_{1 \le i \le m} v_i +\beta\log m.
\end{align*}
In one direction, we have
\begin{align*}
    \lse_{\beta}(\vv) \le \beta\logv{\sum_{i=1}^m \expv{\frac{\max_{1\le i \le m} v_i}{\beta}}} = \beta\log m + \max_{1 \le i \le m} v_i,
\end{align*}
and in the other, we have
\begin{align*}
    \lse_{\beta}(\vv) \ge \beta\logv{\expv{\frac{\max_{1 \le i \le m} v_i}{\beta}}} = \max_{1 \le i \le m} v_i.
\end{align*}
Next, for $\vv \in \R^m$, we will show that
\begin{align*}
    \norm{\vv}_2 - \delta \le \sqrt{\delta^2+\norm{\vv}_2^2} - \delta \le \norm{\vv}_2.
\end{align*}
Indeed, we have
\begin{align*}
    \sqrt{\delta^2+\norm{\vv}_2^2} - \delta \le \sqrt{\delta^2} + \sqrt{\norm{\vv}_2^2} - \delta = \norm{\vv}_2,
\end{align*}
and
\begin{align*}
    \sqrt{\delta^2+\norm{\vv}_2^2} - \delta \ge \sqrt{\norm{\vv}_2^2} - \delta = \norm{\vv}_2 - \delta.
\end{align*}
From this, we get
\begin{align*}
    \ftilde_{\beta,\delta}(\vx) \le \max_{1 \le i \le m} \inparen{\sqrt{\delta^2+\norm{\mA_{S_i}\vx-\vb_{S_i}}_2^2}-\delta} + \beta\log m \le \gnorm{\mA\vx-\vb}{\infty} + \beta\log m
\end{align*}
and
\begin{align*}
    \ftilde_{\beta,\delta}(\vx) \ge \beta\logv{\sum_{i=1}^m \expv{\frac{\norm{\mA_{S_i}\vx-\vb_{S_i}}_2-\delta}{\beta}}} \ge \gnorm{\mA\vx-\vb}{\infty}-\delta.
\end{align*}
Putting these together gives
\begin{align*}
    \abs{\ftilde_{\beta,\delta}(\vx)-\gnorm{\mA\vx-\vb}{\infty}} \le \max\inparen{\beta\log m, \delta} \le \beta\log m + \delta,
\end{align*}
completing the proof of \Cref{lemma:smooth_approx}.
\end{proof}
Eventually, we will choose $\beta = \eps/(4\log m)$ and $\delta = \eps/4$ and then minimize $\ftilde_{\beta,\delta}$ to $\eps/2$ additive error. In light of \Cref{lemma:smooth_approx}, this will be enough to get an $\eps$-additive approximation to the optimum for $\gnorm{\mA\vx-\vb}{\infty}$.


\subsection{Calculus for \lsetext}
\label{sec:lse_calculus}

We investigate certain properties of $\lse_{\beta}(\vy)$ when each entry $[\vy]_i$ is a function $h_i(t)$ for $t \in \R$ for all $i\in[m]$. Let $h(t) \in \R^m$ denote the vector where its $i$th entry is given by $h_i(t)$. We treat each $h_i$ as a one-dimensional restriction of a function $\myfunc{g_i}{\R^m}{\R}$, so $h_i(t) = g_i(\vy+t\vd)$ for center $\vy$ and direction $\vd$ (we omit the parameters $\vy,\vd$ in the notation $h_i$ as it will be clear from context). Finally, recall the definition of quasi-self-concordance (\Cref{defn:qsc}).

We begin with calculating the first two derivatives of $\lse_{\beta}(h(t))$ with respect to $t$ in \Cref{lemma:lse_first_two}.
\begin{lemma}
\label{lemma:lse_first_two}
Let $\lambda_i(t) \coloneqq \expv{h_i(t)/\beta}$. Then, we have
\begin{align*}
    \inparen{\frac{d}{dt}}\lse_{\beta}(h(t)) &= \frac{\sum_{i=1}^m \inparen{\lambda_i(t) \cdot h_i'(t)}}{\sum_{i=1}^m \lambda_i(t)} \\
    \inparen{\frac{d}{dt}}^2\lse_{\beta}(h(t)) &= \frac{1}{\beta}\inparen{\frac{\sum_{i=1}^m \lambda_i(t)h_i'(t)^2}{\sum_{i=1}^m \lambda_i(t)} - \inparen{\frac{\sum_{i=1}^m \lambda_i(t)h_i'(t)}{\sum_{i=1}^m \lambda_i(t)}}^2} + \frac{\sum_{i=1}^m \lambda_i(t)h_i''(t)}{\sum_{i=1}^m \lambda_i(t)}.
\end{align*}
\end{lemma}
\begin{proof}[Proof of \Cref{lemma:lse_first_two}]
The first derivative follows from the chain rule. Indeed, we have
\begin{align*}
    \lse_{\beta}'(h(t)) &= \beta \cdot \frac{\sum_{i=1}^m \lambda_i'(t)}{\sum_{i=1}^m \lambda_i(t)} = \beta \cdot \frac{\sum_{i=1}^m \inparen{\lambda_i(t) \cdot \frac{h_i'(t)}{\beta}}}{\sum_{i=1}^m \lambda_i(t)} = \frac{\sum_{i=1}^m \inparen{\lambda_i(t) \cdot h_i'(t)}}{\sum_{i=1}^m \lambda_i(t)} \le \max_{i} h_i'(t).
\end{align*}
For the second derivative, we use the differentiation rule for multiplication and division and the chain rule, giving
\begin{align*}
    \lse_{\beta}''(h(t)) &= \frac{\left[\inparen{\sum_{i=1}^m \lambda_i'(t)h_i'(t)+\lambda_i(t)h_i''(t)}\inparen{\sum_{i=1}^m \lambda_i(t)}\right] - \frac{1}{\beta}\inparen{\sum_{i=1}^m \lambda_i(t)h_i'(t)}^2}{\inparen{\sum_{i=1}^m \lambda_i(t)}^2} \\
    &= \frac{\left[\frac{1}{\beta}\inparen{\sum_{i=1}^m \lambda_i(t)h_i'(t)^2+\beta\lambda_i(t)h_i''(t)}\inparen{\sum_{i=1}^m \lambda_i(t)}\right] - \frac{1}{\beta}\inparen{\sum_{i=1}^m \lambda_i(t)h_i'(t)}^2}{\inparen{\sum_{i=1}^m \lambda_i(t)}^2} \\
    &= \frac{1}{\beta}\inparen{\frac{\sum_{i=1}^m \lambda_i(t)h_i'(t)^2}{\sum_{i=1}^m \lambda_i(t)} - \frac{\inparen{\sum_{i=1}^m \lambda_i(t)h_i'(t)}^2}{\inparen{\sum_{i=1}^m \lambda_i(t)}^2}} + \frac{\sum_{i=1}^m \lambda_i(t)h_i''(t)}{\sum_{i=1}^m \lambda_i(t)}\enspace.
\end{align*}
This completes the proof of \Cref{lemma:lse_first_two}.
\end{proof}
Next, we prove a general fact regarding composing $\lse$ with a vector formed by functions that are themselves quasi self concordant. See \Cref{lemma:composed_qsc}.
\composedqsc
As far as we are aware, this type of composition result was not previously known and may be of independent interest.

To prove \Cref{lemma:composed_qsc}, we need \Cref{lemma:variance_rv_bound}.
\begin{lemma}
\label{lemma:variance_rv_bound}
For any two random variables $X, Y$, we have
\begin{align*}
    \var{XY} \le 2\maxnorm{Y}^2\var{X}+2\maxnorm{X}^2\var{Y}.
\end{align*}
\end{lemma}
\begin{proof}[Proof of \Cref{lemma:variance_rv_bound}]
The proof follows that of \cite{var_stackexchange}, but we reproduce it here for completeness. First, notice that for random variables $U, V$, we have
\begin{align*}
    2\var{U}+2\var{V} - \var{U+V} = \var{U}+\var{V}-2\cov{U,V} = \var{U-V} \ge 0.
\end{align*}
Let $U=(X-\exv{X})Y$ and $V=\exv{X}Y$. Then, $U+V=XY$, and we have
\begin{align*}
    \var{XY} \le 2\var{(X-\exv{X})Y}+2\var{\exv{X}Y}=2\var{(X-\exv{X})Y}+2\exv{X}^2\var{Y}.
\end{align*}
It remains to bound $\var{(X-\exv{X})Y}$. By H\"older's inequality, we have
\begin{align*}
    \var{(X-\exv{X})Y} \le \exv{((X-\exv{X})Y)^2} \le \exv{(X-\exv{X})^2}\maxnorm{Y}^2=\var{X}\maxnorm{Y}^2.
\end{align*}
Combining everything gives us the conclusion of \Cref{lemma:variance_rv_bound}.
\end{proof}
We are now ready to prove \Cref{lemma:composed_qsc}.
\begin{proof}[Proof of \Cref{lemma:composed_qsc}]
Let $\lambda_i(t) \coloneqq \expv{h_i(t)/\beta}$.

In this proof, we will encounter many weighted averages of vectors $\vz \in \R^m$ of the form
\begin{align*}
    \frac{\sum_{i=1}^m \lambda_i(t)z_i}{\sum_{i=1}^m \lambda_i(t)}.
\end{align*}
Let $\cD$ be the distribution over $[m]$ whose entries are given by $\cD_j = \lambda_j(t)/\sum_{i=1}^m \lambda_i(t)$. In the rest of this proof, all expected values, variances, and covariances will be taken with respect to this distribution. In an abuse of notation, let $h(t)$ denote the ``random'' variable that is $h_i(t)$ with probability $\cD_i$. Define $h'(t), h''(t), h'''(t)$ analogously.

To find the third derivative of $\lse_{\beta}(h(t))$, we start with its second derivative. By \Cref{lemma:lse_first_two}, it is given by
\begin{align*}
    \lse_{\beta}''(h(t)) &= \underbrace{\frac{1}{\beta}\inparen{\frac{\sum_{i=1}^m \lambda_i(t)h_i'(t)^2}{\sum_{i=1}^m \lambda_i(t)} - \inparen{\frac{\sum_{i=1}^m \lambda_i(t)h_i'(t)}{\sum_{i=1}^m \lambda_i(t)}}^2}}_{T_1} + \underbrace{\frac{\sum_{i=1}^m \lambda_i(t)h_i''(t)}{\sum_{i=1}^m \lambda_i(t)}}_{T_2} \\
    &= \frac{1}{\beta}\var{h'(t)}+\exv{h''(t)}.
\end{align*}
We now differentiate the above term by term. First, we have
\begin{align*}
    T_2'(t) &= \frac{\sum_{i=1}^m \lambda_i(t)\inparen{\inparen{\frac{h_i'(t)h_i''(t)}{\beta}}+h_i'''(t)}}{\sum_{i=1}^m \lambda_i(t)} - \frac{1}{\beta}\cdot\frac{\inparen{\sum_{i=1}^m \lambda_i(t)h_i'(t)}\inparen{\sum_{i=1}^m \lambda_i(t)h_i''(t)}}{\inparen{\sum_{i=1}^m \lambda_i(t)}^2} \\
    &= \frac{1}{\beta}\inparen{\frac{\sum_{i=1}^m \lambda_i(t)h_i'(t)h_i''(t)}{\sum_{i=1}^m \lambda_i(t)} - \frac{\inparen{\sum_{i=1}^m \lambda_i(t)h_i'(t)}\inparen{\sum_{i=1}^m \lambda_i(t)h_i''(t)}}{\inparen{\sum_{i=1}^m \lambda_i(t)}^2}} + \frac{\sum_{i=1}^m \lambda_i(t)h_i'''(t)}{\sum_{i=1}^m \lambda_i(t)} \\
    &= \frac{1}{\beta}\cov{h'(t),h''(t)} + \exv{h'''(t)}.
\end{align*}
Next, we have
\begin{align*}
    \frac{d}{dt}\exv{h'(t)}^2 = 2\exv{h'(t)} \cdot \frac{d}{dt}\exv{h'(t)} = 2\exv{h'(t)}\inparen{\frac{1}{\beta}\var{h'(t)} + \exv{h''(t)}}
\end{align*}
and
\begin{align*}
    &\quad \frac{d}{dt}\exv{h'(t)^2}\\
    &= \frac{\inparen{\sum_{i=1}^m \lambda_i'(t)h_i'(t)^2 + 2h_i'(t)h_i''(t)\lambda_i(t)}\inparen{\sum_{i=1}^m \lambda_i(t)} - \frac{1}{\beta}\inparen{\sum_{i=1}^m \lambda_i(t)h_i'(t)}\inparen{\sum_{i=1}^m \lambda_i(t)h_i'(t)^2}}{\inparen{\sum_{i=1}^m \lambda_i(t)}^2} \\
    &= \frac{\inparen{\sum_{i=1}^m \lambda_i'(t)h_i'(t)^2 + 2h_i'(t)h_i''(t)\lambda_i(t)}}{\sum_{i=1}^m \lambda_i(t)} - \frac{1}{\beta}\cdot\frac{\inparen{\sum_{i=1}^m \lambda_i(t)h_i'(t)}\inparen{\sum_{i=1}^m \lambda_i(t)h_i'(t)^2}}{\inparen{\sum_{i=1}^m \lambda_i(t)}^2} \\
    &= \frac{\sum_{i=1}^m \lambda_i(t)\inparen{\frac{h_i'(t)^3}{\beta} + 2h_i'(t)h_i''(t)}}{\sum_{i=1}^m \lambda_i(t)} - \frac{1}{\beta}\cdot\frac{\inparen{\sum_{i=1}^m \lambda_i(t)h_i'(t)}\inparen{\sum_{i=1}^m \lambda_i(t)h_i'(t)^2}}{\inparen{\sum_{i=1}^m \lambda_i(t)}^2} \\
    &= \frac{1}{\beta}\cov{h'(t), h'(t)^2} + 2\exv{h'(t)h''(t)}.
\end{align*}
Combining everything gives us
\begin{align*}
    &\quad \lse_{\beta}'''(h(t)) \\
    &= \frac{1}{\beta}\inparen{\frac{1}{\beta}\cov{h'(t), h'(t)^2} + 2\exv{h'(t)h''(t)} - 2\exv{h'(t)}\inparen{\frac{1}{\beta}\var{h'(t)} + \exv{h''(t)}}}\\
    &\quad + \frac{1}{\beta}\cov{h'(t),h''(t)} + \exv{h'''(t)} \\
    &= \frac{1}{\beta^2}\cov{h'(t),h'(t)^2} - \frac{2}{\beta^2}\exv{h'(t)}\var{h'(t)} + \frac{3}{\beta}\cov{h'(t),h''(t)} + \exv{h'''(t)}.
\end{align*}
We first analyze the terms that only depend on $h'(t)$. To do so, we use \Cref{lemma:variance_rv_bound} to write
\begin{align*}
    \abs{\cov{h'(t),h'(t)^2}} \le \sqrt{\var{h'(t)}}\sqrt{\var{h'(t)^2}} \le 2\norm{\vd}\var{h'(t)}.
\end{align*}
Now, we have
\begin{align*}
    &\quad \frac{1}{\beta^2}\abs{\cov{h'(t),h'(t)^2} - 2\exv{h'(t)}\var{h'(t)}} \\
    &\le \frac{1}{\beta^2}\abs{\cov{h'(t),h'(t)^2}} + \frac{2}{\beta^2}\abs{\exv{h'(t)}\var{h'(t)}} \\
    &\le \frac{4}{\beta^2}\norm{\vd}\var{h'(t)} \le \frac{4}{\beta}\norm{\vd}\lse_{\beta}''(h(t)).
\end{align*}
Next, we take care of the remaining terms. We have
\begin{align*}
    \frac{3}{\beta}\abs{\cov{h'(t),h''(t)}} + \abs{\exv{h'''(t)}} &\le \frac{6}{\beta}\inparen{\max_{i} h_i'(t)}\exv{\abs{h''(t)-\exv{h''(t)}}} + \abs{\exv{h'''(t)}} \\
    &\le \frac{12}{\beta}\norm{\vd}\lse_{\beta}''(h(t)) + \exv{\abs{h'''(t)}} \\
    &\le \frac{12}{\beta}\norm{\vd}\lse_{\beta}''(h(t)) + \nu\norm{\vd}\exv{h''(t)} \\
    &\le \inparen{\frac{12}{\beta}+\nu}\norm{\vd}\lse_{\beta}''(h(t)),
\end{align*}
where the penultimate line follows from \Cref{lemma:hi_qsc}. Combining these conclusions yields
\begin{align*}
    \abs{\lse_{\beta}'''(h(t))} \le \inparen{\frac{16}{\beta}+\nu}\norm{\vd}\lse_{\beta}''(h(t)),
\end{align*}
completing the proof of \Cref{lemma:composed_qsc}.
\end{proof}

\subsection{Smoothness and Quasi-self-concordance of the Modified Objective}
\label{sec:lse_calculus_special}

The main result of this subsection is \Cref{lemma:obj_smooth_qsc}.

\begin{lemma}
\label{lemma:obj_smooth_qsc}
Let $\mW$ be such that for all $\vz \in \R^d$, we have $\gnorm{\mA\vz}{\infty} \le \norm{\mW^{1/2}\mA\vz}_2$. For all $\vx, \vz \in \R^d$ and $t\in\R$, we have
\begin{align*}
    \inparen{\frac{d}{dt}}^2 \ftilde_{\beta,\delta}(\vx+t\vz) &\le \inparen{\frac{1}{\delta}+\frac{1}{\beta}}\norm{\mW^{1/2}\mA\vz}_2^2 & \text{(smoothness)} \\
    \abs{\inparen{\frac{d}{dt}}^3 \ftilde_{\beta,\delta}(\vx+t\vz)} &\le \inparen{\frac{16}{\delta}+\frac{3}{\beta}}\norm{\mW^{1/2}\mA\vz}_2\inparen{\frac{d}{dt}}^2 \ftilde_{\beta,\delta}(\vx+t\vz) & \text{(quasi-self-concordance)}.
\end{align*}
\end{lemma}

Our goal in the rest of this section is to prove \Cref{lemma:obj_smooth_qsc}.

We begin with defining $h_i(t)$ as (absorb the $\delta,\vy,\vd$ parameters into the definition of $h_i$)
\begin{align*}
    h_i(t) \coloneqq \sqrt{\delta^2+\norm{\vy_{S_i}+t\vd_{S_i}}_2^2}.
\end{align*}
Let $h(t)$ denote the vector whose $i$th entry is $h_i(t)$. Then, observe that
\begin{align*}
    \lse_{\beta}(h(t)) = \beta\logv{\sum_{i=1}^m \expv{\frac{h_i(t)}{\beta}}} = \beta\logv{\sum_{i=1}^m \expv{\frac{\sqrt{\delta^2+\norm{\vy_{S_i}+t\vd_{S_i}}_2^2}}{\beta}}}.
\end{align*}
It is easy to see that every one-dimensional restriction of $\ftilde_{\beta,\delta}$ can be obtained by an affine transformation of $\lse_{\beta}(h(t))$ after appropriate choices of $\vy,\vd\in\R^m$. Hence, we first analyze $\lse_{\beta}(h(t))$ for all $\vy,\vd\in\R^m$. 

We begin with proving the smoothness of $\lse_{\beta}(h(t))$ with respect to $\gnorm{\cdot}{\infty}$.
\begin{lemma}
\label{lemma:base_smooth}
For all $\vy,\vd\in\R^m$ and all $t \in \R$, we have
\begin{align*}
    \inparen{\frac{d}{dt}}^2 \lse_{\beta}(h(t)) \le \inparen{\frac{1}{\delta}+\frac{1}{\beta}}\gnorm{\vd}{\infty}^2.
\end{align*}
\end{lemma}
\begin{proof}[Proof of \Cref{lemma:base_smooth}]
By direct calculation, it is easy to see that
\begin{equation}\label{eq:hi_derivatives}
    \begin{aligned}
         h_i'(t) &= \frac{\ip{\vy_{S_i}+t\vd_{S_i},\vd_{S_i}}}{h_i(t)} \\
         h_i''(t) &= \frac{\norm{\vd_{S_i}}_2^2 h_i(t) - h_i'(t)^2h_i(t)}{h_i(t)^2} = \frac{\norm{\vd_{S_i}}_2^2-h_i'(t)^2}{h_i(t)}.
    \end{aligned}
\end{equation}
We plug this into the result of \Cref{lemma:lse_first_two} and get
\begin{align*}
    \lse_{\beta}''(h(t)) &\le \frac{1}{\beta}\max_i h_i'(t)^2 + \max_i h_i''(t) \\
    &= \frac{1}{\beta} \max_{i}\inparen{\frac{\ip{\vy_{S_i}+t\vd_{S_i},\vd_{S_i}}}{\sqrt{\delta^2 + \norm{\vy_{S_i}+t\vd_{S_i}}_2^2}}}^2 + \max_{i} \frac{\norm{\vd_{S_i}}_2^2 - h_i'(t)^2}{\sqrt{\delta^2+\norm{\vy_{S_i}+t\vd_{S_i}}_2^2}} \\
    &\le \frac{1}{\beta}\max_{i} \norm{\vd_{S_i}}_2^2 + \frac{1}{\delta}\max_{i}\norm{\vd_{S_i}}_2^2 = \inparen{\frac{1}{\beta}+\frac{1}{\delta}}\gnorm{\vd}{\infty}^2,
\end{align*}
completing the proof of \Cref{lemma:base_smooth}.
\end{proof}
Our next task is to show that $\lse_{\beta}(h(t))$ is $O(1/\beta + 1/\delta)$-quasi-self-concordant in $\gnorm{\cdot}{\infty}$. To do so, we will appeal to \Cref{lemma:composed_qsc}. To be able to do this, we first have to prove the quasi-self-concordance of each component function in $\lse_{\beta}(h(t))$.
\begin{lemma}
\label{lemma:hi_qsc}
For all $\vy,\vd\in\R^m$ and all $t\in\R$, we have
\begin{align*}
    \abs{\inparen{\frac{d}{dt}}^3 \sqrt{\delta^2+\norm{\vy_{S_i}+t\vd_{S_i}}_2^2}} \le \frac{3}{\delta}\norm{\vd_{S_i}}_2\inparen{\inparen{\frac{d}{dt}}^2\sqrt{\delta^2+\norm{\vy_{S_i}+t\vd_{S_i}}_2^2}}.
\end{align*}
\end{lemma}
\begin{proof}[Proof of \Cref{lemma:hi_qsc}]
Although a similar fact appears in \cite[Section 2.1.2]{ob18}, it is not in the exact form we need. So, we prove the required statement here.

Recycling the computation from \eqref{eq:hi_derivatives}, recall
\begin{align*}
    h_i''(t) = \frac{\norm{\vd_{S_i}}_2^2 - h_i'(t)^2}{h_i(t)},
\end{align*}
which gives
\begin{align*}
    h_i'''(t) = \frac{-2h_i'(t)h_i''(t)h_i(t) - h_i'(t)(h_i(t)h_i''(t))}{h_i(t)^2} = -\frac{3h_i'(t)h_i''(t)}{h_i(t)}.
\end{align*}
Finally, again recalling \eqref{eq:hi_derivatives}, notice that
\begin{align*}
    \abs{\frac{h_i'(t)}{h_i(t)}} = \abs{\frac{\ip{\vy_{S_i}+t\vd_{S_i},\vd_{S_i}}}{h_i(t)^2}} = \abs{\ip{\frac{\vy_{S_i}+t\vd_{S_i}}{\sqrt{\delta^2+\norm{\vy_{S_i}+t\vd_{S_i}}_2^2}}, \frac{\vd_{S_i}}{\sqrt{\delta^2+\norm{\vy_{S_i}+t\vd_{S_i}}_2^2}}}} \le \frac{\norm{\vd_{S_i}}_2}{\delta}.
\end{align*}
Combining everything completes the proof of \Cref{lemma:hi_qsc}.
\end{proof}
We are now ready to prove the quasi-self-concordance of $\lse_{\beta}(h(t))$ in $\gnorm{\cdot}{\infty}$.
\begin{lemma}
\label{lemma:base_qsc}
For all $\vy, \vd \in \R^m$ and $t\in\R$, we have
\begin{align*}
    \abs{\inparen{\frac{d}{dt}}^3 \lse_{\beta}(h(t))} \le \inparen{\frac{16}{\beta}+\frac{3}{\delta}}\gnorm{\vd}{\infty}\inparen{\frac{d}{dt}}^2 \lse_{\beta}(h(t)).
\end{align*}
\end{lemma}
\begin{proof}[Proof of \Cref{lemma:base_qsc}]
In the statement of \Cref{lemma:composed_qsc}, let $\norm{\cdot} = \gnorm{\cdot}{\infty}$. By the definition of $\gnorm{\cdot}{\infty}$ and $h_i$, we have for all $i$ and $t$ that $h_i'(t) \le \gnorm{\vd}{\infty}$. Additionally, from \Cref{lemma:hi_qsc}, we have that the $h_i(t)$ are $3/\delta$-quasi-self-concordant in the norm $\gnorm{\vd}{\infty}$ for all $i$. \Cref{lemma:base_qsc} now follows immediately from \Cref{lemma:composed_qsc}.
\end{proof}
Finally, we can prove \Cref{lemma:obj_smooth_qsc}.
\begin{proof}[Proof of \Cref{lemma:obj_smooth_qsc}]
By the conclusion of \Cref{lemma:base_smooth}, we know that for all $\vy, \vd \in \R^m$ and $t\in\R$ that
\begin{align*}
    \inparen{\frac{d}{dt}}^2 \lse_{\beta}(h(t)) \le \inparen{\frac{1}{\delta}+\frac{1}{\beta}}\gnorm{\vz}{\infty}^2.
\end{align*}
Let $\vy = \mA\vx - \vb$ for some $\vx$ and $\vd = \mA\vz$ for some $\vz$. Let
\begin{align*}
    g(\vy) \coloneqq \beta\logv{\sum_{i=1}^m \expv{\frac{\sqrt{\delta^2+\norm{\vy_{S_i}}_2^2}-\delta}{\beta}}}.
\end{align*}
Then,
\begin{align*}
    \inparen{\frac{d}{dt}}^2 \ftilde_{\beta,\delta}(\vx+t\vz) = \inparen{\frac{d}{dt}}^2 g(\mA\vx - \vb + t\mA\vz) \le \inparen{\frac{1}{\delta}+\frac{1}{\beta}}\gnorm{\mA\vz}{\infty}^2.
\end{align*}
With the exact same reasoning applied to the conclusion of \Cref{lemma:base_qsc}, we also see that
\begin{align*}
    \abs{\inparen{\frac{d}{dt}}^3 \ftilde_{\beta,\delta}(\vx+t\vz)} \le \inparen{\frac{16}{\delta}+\frac{3}{\beta}}\gnorm{\mA\vz}{\infty}\inparen{\frac{d}{dt}}^2 \ftilde_{\beta,\delta}(\vx+t\vz).
\end{align*}
The conclusion of \Cref{lemma:obj_smooth_qsc} then follows from remembering that we have $\mW$ such that for all $\vz\in\R^d$, $\gnorm{\mA\vz}{\infty} \le \norm{\mW^{1/2}\mA\vz}_2$ (following from \Cref{thm:gp_regression_blw_intro}).
\end{proof}

\subsection{Analysis of \texorpdfstring{\Cref{alg:fair_regression}}{Algorithm 1}}
\label{sec:gp_robust_analysis_combining}

In this subsection, we use the calculus facts from the previous two subsections to analyze \Cref{alg:fair_regression}. The outline of this proof follows that of \cite[Theorem 2]{jls21}, which in turn builds up to using the proof used in \cite[Corollary 12]{msbacon}. The main idea is to define the algorithm based on the norm given by a good choice of positive semidefinite $\mM$, given by \Cref{thm:gp_regression_blw_intro}.

In the rest of this section, let $\mW$ be factor-$2$ block Lewis weight overestimates for $\insquare{\mA \vert \vb}$. As in Line \ref{line:blw} of \Cref{alg:fair_regression} and from the corresponding guarantee given in \cite[Lemmas 5.6, 5.8]{mo23}, this means that within $2\log m$ linear-system-solves in $\mA^{\top}\mD\mA$ for diagonal $\mD$, we can find $\mW$ such that for all $\vx\in\R^d$ and $c\in\R$ we have
\begin{align*}
    \gnorm{\mA\vx-c\vb}{\infty} \le \norm{\mW^{1/2}\mA\vx-c\mW^{1/2}\vb}_2 \le \sqrt{2(\rank{\mA}+1)}\gnorm{\mA\vx-c\vb}{\infty}.
\end{align*}
Note that choosing $c=1$ yields our original objective on either side of the above inequality. Motivated by the above, it is natural to use the norm given by $\mM \coloneqq \mA^{\top}\mW\mA$ to give the geometry for the ball optimization oracle and for the analysis. Additionally, without loss of generality and for the sake of the analysis, let us rescale the problem so that
\begin{align*}
    1 = \opt \coloneqq \gnorm{\mA\xstar-\vb}{\infty}.
\end{align*}
Also, as mentioned earlier, assume without loss of generality that $\rank{\mA} = d$.

We begin with \Cref{lemma:initialization}, which bounds our initial suboptimality in $\ftilde$ and in $\norm{\cdot}_{\mM}$.
\begin{lemma}
\label{lemma:initialization}
Let $\xtilde_{\beta,\delta} \coloneqq \argmin{\vx\in\R^d} \ftilde_{\beta,\delta}(\vx)$. Then,
\begin{equation*}
    \begin{aligned}
        \norm{\xtilde_{\beta,\delta} - \vx_0}_{\mM} &\le (2+2(\beta\log m + \delta))\sqrt{2(d+1)} \\
        \ftilde_{\beta,\delta}(\vx_0) - \ftilde_{\beta,\delta}(\xtilde_{\beta,\delta}) &\le \sqrt{2(d+1)} - 1+2(\beta\log m + \delta)
    \end{aligned}.
\end{equation*}
\end{lemma}
\begin{proof}[Proof of \Cref{lemma:initialization}]
It is easy to check that
\begin{align*}
    \vx_0 \coloneqq \inparen{\mA^{\top}\mW\mA}^{-1}\mA^{\top}\mW\vb = \argmin{\vx \in \R^d} \norm{\mW^{1/2}\mA\vx-\mW^{1/2}\vb}_2.
\end{align*}
By \Cref{lemma:smooth_approx}, for all $\vx\in\R^d$,
\begin{align*}
    \abs{\ftilde_{\beta,\delta}(\vx)-\gnorm{\mA\vx-\vb}{\infty}} \le \beta\log m+\delta,
\end{align*}
implying
\begin{align*}
    \abs{\gnorm{\mA\xstar-\vb}{\infty} - \ftilde_{\beta,\delta}(\xtilde_{\beta,\delta})} \le \beta\log m+\delta.
\end{align*}
Combining this with \Cref{thm:blw_to_l2}, we get
\begin{align*}
    1 \le \gnorm{\mA\xstar-\vb}{\infty} \le \gnorm{\mA\vx_0-\vb}{\infty} \le \norm{\mW^{1/2}\mA\vx_0 - \mW^{1/2}\vb}_2
\end{align*}
and
\begin{align*}
    \frac{\norm{\mW^{1/2}\mA\vx_0 - \mW^{1/2}\vb}_2}{\sqrt{2(d+1)}} \le \frac{\norm{\mW^{1/2}\mA\xstar-\mW^{1/2}\vb}_2}{\sqrt{2(d+1)}} \le \gnorm{\mA\xstar-\vb}{\infty} = 1.
\end{align*}
Combining these gives
\begin{align*}
    1 \le \norm{\mW^{1/2}\mA\vx_0 - \mW^{1/2}\vb}_2 \le \sqrt{2(d+1)}.
\end{align*}
Additionally,
\begin{align*}
     \norm{\mW^{1/2}\mA\xtilde_{\beta,\delta}-\mW^{1/2}\vb}_2 &\le \sqrt{2(d+1)}\gnorm{\mA\xtilde_{\beta,\delta}-\vb}{\infty} \\
     &\le \sqrt{2(d+1)}\inparen{\ftilde_{\beta,\delta}(\xtilde_{\beta,\delta})+\beta\log m+\delta} \\
     &\le \sqrt{2(d+1)}\inparen{\gnorm{\mA\xstar-\vb}{\infty}+2(\beta\log m + \delta)} \\
     &= \sqrt{2(d+1)}(1+2(\beta\log m + \delta)).
\end{align*}
Then,
\begin{align*}
    \norm{\xtilde-\vx_0}_{\mM} &= \norm{\inparen{\mW^{1/2}\mA\xtilde_{\beta,\delta}-\mW^{1/2}\vb} - \inparen{\mW^{1/2}\mA\vx_0-\mW^{1/2}\vb}}_2 \\
    &\le\norm{\mW^{1/2}\mA\xtilde_{\beta,\delta}-\mW^{1/2}\vb}_2 +\norm{\mW^{1/2}\mA\vx_0-\mW^{1/2}\vb}_2 \\
    &\le (2+2(\beta\log m + \delta))\sqrt{2(d+1)},
\end{align*}
and
\begin{align*}
    \ftilde_{\beta,\delta}(\vx_0) - \ftilde_{\beta,\delta}(\xtilde_{\beta,\delta}) &\le \gnorm{\mA\vx_0-\vb}{\infty} - \gnorm{\mA\xstar-\vb}{\infty}+2(\beta\log m + \delta) \\
    &\le \norm{\mW^{1/2}\mA\vx_0 - \mW^{1/2}\vb}_2 - \opt+2(\beta\log m + \delta) \\
    &\le \sqrt{2(d+1)}-1+2(\beta\log m + \delta).
\end{align*}
This completes the proof of \Cref{lemma:initialization}.
\end{proof}

We are now ready to prove \Cref{mainthm:fair_regression_iteration_complexity}.

\begin{proof}[Proof of \Cref{mainthm:fair_regression_iteration_complexity}]
\Cref{alg:fair_regression} optimizes the regularization of $\ftilde$ given by
\begin{align*}
    \fhat(\vx) \coloneqq \ftilde_{\beta,\delta}(\vx) + \frac{\eps}{110R^2}\norm{\mW^{1/2}\mA(\vx-\vx_0)}_2^2,
\end{align*}
where $R$ is such that $\norm{\vx_0 - \xtilde_{\beta,\delta}}_{\mM} \le R$. Let $ \xhat \coloneqq \argmin{\vx\in\R^d} \fhat(\vx)$. Using \cite[Proof of Corollary 12]{msbacon}, we know that for every iterate $\vx$ of \Cref{alg:fair_regression},
\begin{align*}
    \abs{\fhat(\vx)-\ftilde_{\beta,\delta}(\vx)} \le \frac{\eps}{4}.
\end{align*}
We now choose $\beta = \eps/(4\log m)$ and $\delta = \eps/4$, so that $\ftilde_{\beta,\delta}$ approximates $f$ up to error $\eps/2$ on every point. Using \Cref{lemma:initialization}, this gives $R = (2+\eps)\sqrt{2(d+1)}$. It is therefore sufficient to optimize $\fhat$ up to $\eps/4$ additive error.

Next, using \Cref{lemma:obj_smooth_qsc} and \cite[Lemmas 11, 43]{msbacon}, we have that $\fhat$ is $(1/\nu, e)$-Hessian stable in $\norm{\cdot}_{\mM}$ for $\nu = \Omega(1/(\eps\log m))$. We now invoke \cite[Theorem 9]{msbacon}, which tells us that we can implement a $(C/\sqrt{d}, C/\eps)$-ball optimization oracle for $f$ with $O\inparen{\logv{\frac{d}{\eps}}^2}$ linear-system-solves. 

The next step is to turn the ball optimization oracle into a $\frac{1}{2}$-MS oracle (\Cref{defn:ms_oracle}). Using \cite[Proposition 5]{msbacon}, we get a ball oracle complexity of $O\inparen{\logv{\frac{d}{\eps}}}$ to implement the MS oracle. In total, our linear-system-solve complexity for implementing the MS oracle for iteration $t$ is $O\inparen{\logv{\frac{d}{\eps}}^3}$.

Finally, using \cite[Theorem 6]{msbacon}, we get that \Cref{alg:fair_regression} has a Newton iteration complexity of
\begin{align*}
    &\quad O\inparen{\inparen{\frac{(1+\eps)\sqrt{d}\log m}{\eps}}^{2/3}\logv{\frac{\sqrt{d}+\eps}{\eps}}\inparen{\logv{\frac{(\log m/\eps)d(1+(1+\eps)\sqrt{d}\log m/\eps)}{\eps}}}^3} \\
    &= O\inparen{\frac{d^{1/3}}{\eps^{2/3}}\logv{\frac{d\log m}{\eps}}^{14/3}},
\end{align*}
as promised.

Next, we analyze what happens if we fall in the case where $\mW=\mI_m$. Here, by using the $\sqrt{m}$ distortion from approximating $\ell_{\infty}^m$ with $\ell_{2}^m$, we have for all $\vx\in\R^d$,
\begin{align*}
    \frac{\norm{\mA\vx-\vb}_2}{\sqrt{m}} \le \gnorm{\mA\vx-\vb}{\infty} \le \norm{\mA\vx-\vb}_2.
\end{align*}
Using this and repeating the previous analysis with this choice of $\mM$ gives us a rate of
\begin{align*}
    O\inparen{\frac{m^{1/3}}{\eps^{2/3}}\logv{\frac{m\log m}{\eps}}^{14/3}},
\end{align*}
as required.

It remains to determine the form of the Newton steps. For this, it is sufficient to understand the Hessian of $\fhat$. A straightforward calculation shows that it is of the form $\mA^{\top}\mB\mA$ where $\mB$ is a block-diagonal matrix where each block has size $\abs{S_i} \times \abs{S_i}$. Thus, each Newton step solves a linear system of the form $\mA^{\top}\mB\mA\vz = \vv$.

Combining this with the iteration complexity guarantee to find $\mW$ (see \Cref{thm:gp_regression_blw_intro}) completes the proof of \Cref{mainthm:fair_regression_iteration_complexity}.
\end{proof}

\input{interpolation}

\input{experiments}

\subsubsection*{Acknowledgments}
We thank Aaron Sidford for useful comments during the early stage of the project. We also thank the anonymous reviewers at ICLR'26 and NeurIPS'25 for their significant contributions to improving the paper. The authors used ChatGPT 5.3 and Claude Opus 4.6 to implement the algorithms in this paper. Most of this work was done when NSM and KKP were graduate students at Toyota Technological Institute, Chicago (TTIC). KKP and NSM were supported through the NSF TRIPOD Institute on Data, Economics, Algorithms and Learning (IDEAL) and other awards from DARPA and NSF. NSM thanks Citadel LLC for sponsoring the conference travel and attendance for the presentation of this work.

%% file: intro.tex
Machine learning algorithms and their training datasets have grown substantially in both size and complexity over the past decade. This increased model complexity has made it challenging to interpret and predict their behavior in unobserved scenarios. Hence, many applications that involve societal decisions still rely on simple, interpretable models like linear regression, often after feature engineering. Examples of such applications include predicting national housing prices, estimating wages across industries, forecasting loan amounts across banks, predicting life insurance premiums across groups, and projecting energy consumption across communities \citep{cohen2024fairness}. 

A shared safety and sometimes legal concern across the above applications is the potential for wildly different model qualities for different distributions, i.e., outputting a notably worse model for some source data distributions~\citep{data2014seizing, barocas2016big,hardt2016equality,veale2018fairness,selbst2019fairness,berk2021fairness, corbett2023measure,chouldechova2016fair,kleinberg2018algorithmic, agarwal2019fair,cohen2024fairness,svwz24}. Specifically, consider fitting a linear model $\vx\in\R^d$ to make real predictions on some task over $m$ groups where group $i$'s dataset consists of $n_i$ entries and is denoted by $S_i=\{(\va_i^j, b_i^j)\}_{j\in [n_i]}$. The \textit{utilitarian} or the total-cost-minimizing objective minimizes the average squared prediction error across groups, i.e.,        
\begin{align}
    \min_{\vx\in\R^d}\frac{1}{m}\sum_{i\in[m]}\frac{1}{n_i}\norm{\mA_{S_i}\vx-\vb_{S_i}}_2^2\enspace,\label{eq:utilitarian}
\end{align}
where $\mA_{S_i} \coloneqq [\va_i^1 \dots \va_i^{n_i}]^{\top}\in \R^{n_i \times d}$ is the feature matrix and $\vb_{S_i} \coloneqq [b_i^1 \dots b_i^{n_i}]^{\top} \in \R^{n_i}$ is the label vector for group $i\in[m]$. 

Due to the inherent heterogeneity of the datasets, the model derived by optimizing the objective \eqref{eq:utilitarian} may be particularly detrimental to some groups, as the prediction error may be disproportionately higher for these groups. To overcome these limitations, the following \textit{egalitarian} or group Distributionally Robust Optimization (DRO) objective  has been considered in several recent works \citep{ben2013robust, duchi2016statistics, sagawa2019distributionally, levy2020large, soma2022optimal,aakmrz22,svwz24},
\begin{align}
    \min_{\vx\in\R^d} \max_{i\in[m]}\frac{1}{n_i}\norm{\mA_{S_i}\vx-\vb_{S_i}}_2^2\enspace.\label{eq:main_objective}
\end{align}
Objective \ref{eq:main_objective} is the ``fairest'' objective among all objectives that balance utility and distributional robustness~\citep{kleinberg2018algorithmic,chouldechova2018frontiers,asadpour2022sequential,chen2022fair,rahmattalabi2019exploring,golrezaei2024online}. Since objective \ref{eq:main_objective} is a convex problem, it is natural to apply standard black-box optimization techniques to solve it. However, we identify several challenges in applying existing methods:
    \paragraph{Efficient first-order algorithms have geometry-dependent rates.} To our knowledge, using an efficient first-order method (such as sub-gradient descent) will incur a geometry-dependent runtime. In particular, if the matrices $\mA_{S_i}$ or if the stacked matrix $\mA\coloneqq [\mA_{S_1}^{\top} \dots \mA_{S_M}^{\top}]^{\top}$ are poorly conditioned, then this will be reflected accordingly in the convergence rates. This is a drawback of the existing results by \citet{aakmrz22} and \citet{svwz24}.
    \paragraph{Objective \eqref{eq:main_objective} is not smooth.} Since the objective is the pointwise maximum of several continuous functions, the derivative is not well-defined at the points at which the maximizing function changes. Thus, applying subgradient descent to this objective without a tailored analysis will yield a rather unimpressive $1/\eps^{2}$ dependence in the iteration complexity. 
    \paragraph{Min-max optimization/regret minimization approaches have a $1/\eps^{2}$ dependence on iteration complexity.} Since problem \ref{eq:main_objective} is a min-max optimization objective, it is also natural to try to use game theory-inspired approaches that use some oracle (such as gradients) for each group as a black box. For instance, we can cast objective \ref{eq:main_objective} as a repeated game between a min player (equipped with a no-regret algorithm) and a max player (equipped with the best response oracle). The main shortcoming of this approach is that even though the function for each group is smooth, the iteration complexity (to get $\eps$ average regret) for smooth online convex optimization still has an unimpressive $1/\eps^2$ dependence (as opposed to $1/\eps$ for smooth convex optimization)~\citep{soma2022optimal, zhang2024stochastic}. Thus, this approach is no better than optimizing \eqref{eq:main_objective} using sub-gradient descent.
    \paragraph{Interior point methods have a poor iteration complexity for large $m$.} Another natural approach (that can partially address the previous two issues), following the discussion by \citet[Section 6.4]{boyd2004convex}, is to rewrite problem \ref{eq:main_objective} in its epigraph form and use an interior point method (IPM) to solve the resulting problem (which, in this case, is a quadratically constrained linear program). Unfortunately, this will give an algorithm whose analysis is only known to yield an iteration complexity of $O(\sqrt{m})$, where each iteration solves a linear system in matrices of the form $\mA^{\top}\mB\mA$ for a block-diagonal $\mB$ (see \Cref{rem:linear_systems}).
    A na\"ive implementation of this algorithm will thus have a superlinear runtime in the number of groups, which is undesirable when the number of groups is large. Alternately, consider an example in which we copy each group $k$ times in the objective. The new objective value does not change from the original objective value, but the iteration complexity from the IPM now blows up to $\sqrt{mk}$. This also signals to us that we should search for an algorithm whose iteration complexity is mostly independent from $m$. 

Hence, designing an algorithm without these shortcomings requires novel ideas. 

\subsection{Our Results}
\label{sec:gp_our_results}

In this paper, we present a new algorithm (\Cref{alg:fair_regression}) to approximately optimize objective \ref{eq:main_objective}, which addresses the aforementioned difficulties. We state the iteration complexity of our algorithm in the following theorem. 

\begin{mainthm}[Robust regression]\label{mainthm:fair_regression_iteration_complexity}
Let $\mA_{S_i} \in \R^{n_i \times d}$ and $\vb_{S_i} \in \R^{n_i}$ for all $i\in[m]$. Denote their concatenations by $\mA\coloneqq [\mA_{S_1}^{\top} \dots \mA_{S_M}^{\top}]^{\top}\in\R^{n\times d}$ and $\vb \coloneqq [\vb_{S_1}^{\top}\dots\vb_{S_M}^{\top}]^{\top}\in \R^n$ where $n:= \sum_{i\in[m]}n_i$. Let $\eps > 0$. Then \Cref{alg:fair_regression} returns $\xhat$ such that,
\begin{align}
    \max_{i\in[m]} \frac{1}{\sqrt{n_i}}\norm{\mA_{S_i}\xhat-\vb_{S_i}}_2 \le (1+\eps)\cdot\min_{\vx\in\R^d}\max_{i\in[m]} \frac{1}{\sqrt{n_i}}\norm{\mA_{S_i}\vx-\vb_{S_i}}_2\enspace,\label{eq:objective_approx}
\end{align}
and it runs in
\begin{align*}
    O\inparen{\frac{\min\inbraces{\mathsf{rank}(\mA),m}^{1/3}\inparen{\logv{\frac{n\log m}{\eps}}^{14/3}+\logv{m}}}{\eps^{2/3}}}
\end{align*}
linear-system-solves in matrices of the form $\mA^{\top}\mB\mA$, where $\mB$ is a block-diagonal matrix for which block $i$ has size $n_i \times n_i$.
\end{mainthm}
We prove \Cref{mainthm:fair_regression_iteration_complexity} in \Cref{sec:gp_robust_proof}. We compare the guarantee of \Cref{mainthm:fair_regression_iteration_complexity} against the other baselines in \Cref{table:compare}. 

\def\arraystretch{1.5}
\begin{table}[h]
\centering
\begin{tabular}{l c c} 
\toprule\toprule
\textbf{Algorithm}                                                                                                     & \textbf{Iteration Complexity}                                                                              & \textbf{Each Iteration}                     \\ \midrule\midrule
\makecell[l]{Subgradient descent}                                                                                           & \makecell[c]{$\frac{\norm{\xstar}_2 \max_{1 \le i \le m} \tfrac{1}{\sqrt{n_i}}\opnorm{\mA_{S_i}}}{\eps^{2}}$}                                   & \makecell[c]{Evaluate $\nabla f(\vx)$}                    \\ \midrule
\makecell[l]{Nesterov acceleration\\ on smoothened objective}                                                                 & \makecell[c]{$\frac{\norm{\xstar}_2 \left(\max_{1 \le i \le m} \tfrac{1}{\sqrt{n_i}}\opnorm{\mA_{S_i}}\right)^{1/2}}{\eps}$}                      & \makecell[c]{Evaluate $\nabla \ftilde_{\beta,\delta}(\vx)$} \\ \midrule
\makecell[l]{\cite{aakmrz22}}                                                                                               & \makecell[c]{$\frac{\norm{\xstar}_2 \max_{1 \le i \le m} \tfrac{1}{\sqrt{n_i}}\opnorm{\mA_{S_i}}}{\eps}$}                                       & \makecell[c]{Evaluate $\nabla \ftilde_{\beta,\delta}(\vx)$} \\ \midrule
\makecell[l]{Interior point with log barrier\\ \citep{boyd2004convex}}                                                         & \makecell[c]{$m^{1/2} \log\left(\frac{1}{\eps}\right)$}                                                                    & \makecell[c]{Linear-system-solve\\ in $\mA^{\top}\mB\mA$}    \\ \midrule
\makecell[l]{\textbf{This paper} \\ (na\"ive geometry)}                                                                                          & \makecell[c]{$\frac{m^{1/3}}{\eps^{2/3}}$}                                                                                & \makecell[c]{Linear-system-solve\\ in $\mA^{\top}\mB\mA$}    \\ \midrule
\makecell[l]{$\ell_{\infty}$ regression with Lewis \\weights \citep{jls21}}                                                    & \makecell[c]{$\frac{\rank{\mA}^{1/3}}{\eps^{2/3}}$}                                                                       & \makecell[c]{Linear-system-solve\\ in $\mA^{\top}\mD\mA$}    \\ \midrule
\makecell[l]{$\ell_{\infty}$ regression with IPM\\ \citep{ls19}}                                                               & \makecell[c]{$\rank{\mA}^{1/2} \log\left(\frac{1}{\eps}\right)$}                                                           & \makecell[c]{Linear-system-solve\\ in $\mA^{\top}\mD\mA$}    \\ \midrule
\makecell[l]{\textbf{This paper (\Cref{mainthm:fair_regression_iteration_complexity})}}                                                                                                    & \makecell[c]{$\frac{\min\inbraces{\rank{\mA},m}^{1/3}}{\eps^{2/3}}$}                                                               & \makecell[c]{Linear-system-solve\\ in $\mA^{\top}\mB\mA$}    \\ \bottomrule\bottomrule
\end{tabular}
\caption{The complexities of algorithms for optimizing \eqref{eq:main_objective} or for the special case of $\ell_{\infty}$ regression, assuming $\opt=1$ (the first three guarantees are additive approximations) and ignoring $\mathrm{polylog}(n,m)$ terms. We write $\mD$ to be a diagonal matrix and $\mB$ to be a block-diagonal matrix where each block has size $(n_i + o(1)) \times (n_i + o(1))$. We remark that in the special case where $n_i=1$, our algorithm exactly recovers guarantees of \cite{jls21}. We stress that we include the references to $\ell_{\infty}$ regression only to show that our algorithm is no worse than that of \cite{jls21} in this special case of $n_i=1$ for all $i$, and none of their algorithms apply to our general setting.\label{table:compare}}
\end{table}

Unlike the aforementioned first-order methods, our algorithm has no geometry-dependent terms. Additionally, our algorithm improves over the standard log-barrier IPM when the desired accuracy $\eps \ge m^{-1/4}$ --- this improvement is more pronounced when $m \gg \rank{\mA}$, i.e. when the number of data sources is much larger than the dimension of the parameter vector $\vx$. Additionally, for $\eps \ge \rank{\mA}^{-1/4}$, our guarantee matches the best known guarantee for $\ell_{\infty}$ regression \citep{ls19,jls21}.

\begin{remark}[Why use linear-system-solve complexity?]\label{rem:linear_systems}
    We benchmark our algorithms using the number of linear-system-solves for a few reasons. First, this is typically how second-order algorithms are compared, such as interior point methods for linear programming \citep{ls19}. Second, the particular structure of the linear-system-solves presents the possibility of a faster amortized runtime for the systems over the algorithm's run. This observation, combined with an understanding of how the linear systems changed between iterations, was used recently used to achieve fast runtimes for linear programming \citep{ls19} and $\ell_{\infty}$ regression \citep{ajk24}.
\end{remark}

\paragraph{Interpolating between robust and nonrobust optimization.} We also study the following family of objectives that interpolate between objectives \ref{eq:utilitarian} and \ref{eq:main_objective} for different values of $p\geq 2$,
\begin{align}
    \min_{\vx\in\R^d}\frac{1}{m}\sum_{i\in[m]}\left(\frac{1}{n_i}\norm{\mA_{S_i}\vx-\vb_{S_i}}_2^2\right)^{p/2}\enspace.\label{eq:interpolation}
\end{align}
In particular, note that choosing $p=2$ in the above objective gives us the average least-squares problem in objective \ref{eq:utilitarian}, while $p\to\infty$ recovers objective \ref{eq:main_objective}. Varying $p$ from $2$ to $\infty$ and minimizing gives solutions that interpolate between utilitarian and egalitarian approaches, allowing for a smooth trade-off between utility and robustness. To this end, we give \Cref{alg:gp_regression_alg} to approximately optimize objective \ref{eq:interpolation} and prove the following guarantee about its iteration complexity.

\begin{mainthm}[Trading off utility and robustness]\label{mainthm:gp_regression_iteration_complexity}
Let $\mA_{S_i} \in \R^{n_i \times d}$ and $\vb_{S_i} \in \R^{n_i}$ for all $i\in[m]$. Denote their concatenations by $\mA\coloneqq [\mA_{S_1}^{\top} \dots \mA_{S_M}^{\top}]^{\top}\in\R^{n\times d}$ and $\vb \coloneqq [\vb_{S_1}^{\top}\dots\vb_{S_M}^{\top}]^{\top}\in \R^n$ where $n \coloneqq \sum_{i\in[m]}n_i$. Let $p \ge 2$ and $\eps > 0$. Then \Cref{alg:gp_regression_alg} returns $\xhat$ such that,
\begin{align}\label{eq:interpolation_approx}
    \inparen{\sum_{i=1}^m \inparen{\frac{1}{\sqrt{n_i}}\norm{\mA_{S_i}\xhat-\vb_{S_i}}_2}^p}^{1/p} \le (1+\eps)\cdot\min_{\vx\in\R^d}\inparen{\sum_{i=1}^m \inparen{\frac{1}{\sqrt{n_i}}\norm{\mA_{S_i}\vx-\vb_{S_i}}_2}^p}^{1/p}
\end{align}
and runs in
\begin{align*}          
    O\inparen{p^{O(1)}\min\inbraces{\rank{\mA},m}^{\frac{p-2}{3p-2}}\logv{\frac{pd}{\eps}}^3}
\end{align*}
linear-system-solves in matrices of the form $\mA^{\top}\mB\mA$, where $\mB$ is a block-diagonal matrix for which block $i$ has size $n_i \times n_i$.
\end{mainthm}
We prove \Cref{mainthm:gp_regression_iteration_complexity} in \Cref{sec:gp_interpolating_proof}.

In the special case where $n_i=1$ for all $i$ (and therefore the problem is $\ell_p$ regression for $p \ge 2$), the complexity promised by \Cref{mainthm:gp_regression_iteration_complexity} is comparable to that promised by \citet{jls21} for $\ell_p$ regression. The main difference is that our iteration complexity is unconditionally polynomial in $p$. In contrast, the comparable result from \cite{jls21} seems to require mild assumptions on the problem parameters (see the ``Discussion on numerical stability'' by \citet[Section 4]{jls21}).

\begin{remark}[Large values of $p$] 
    Note that for values of $p$ larger than $\log(m)$, solving \eqref{eq:main_objective} is almost equivalent to solving \eqref{eq:interpolation}. To intuitively see this, first recall that for any vector $\vx\in\R^d$ and $p=\log_2(m)$ we have, $\|\vx\|_\infty \leq \|\vx\|_{p} \leq 2\cdot\|\vx\|_\infty$. This implies that for all $i\in[m]$ we have the following for objective \eqref{eq:interpolation} (for $p=\log_2(m)$) for any $\vx\in\R^d$,
    \begin{align*}
        \max_{i\in[m]}\norm{\mA_{S_i}\vx-\vb_{S_i}}_2 \leq \left(\sum_{i\in[m]}\norm{\mA_{S_i}\vx-\vb_{S_i}}_2^{p}\right)^{1/p}\leq 2\cdot\max_{i\in[m]}\norm{\mA_{S_i}\vx-\vb_{S_i}}_2\enspace. 
    \end{align*}
    In particular, this means that minimizing the interpolating objective \eqref{eq:interpolation} also minimizes the robust objective \eqref{eq:main_objective} (up to numerical constants) and vice versa. Thus, for $p=\Omega\left(\log_2(m)\right)$, for our intended applications, it makes sense to minimize the robust objective instead. This is why, in \Cref{mainthm:gp_regression_iteration_complexity}, we do not care too much about the exponent on $p$ in the iteration complexity. Our main goal is to show that we can get a $O(\mathsf{poly}(p,\logv{\frac{1}{\eps}})\min\inbraces{\rank{\mA},m}^{1/3})$ iteration complexity.
\end{remark}

\subsection{Prior Results, Connections, and Open Problems}
\input{related}

\subsection{Paper Outline}

In the remainder of this paper, we will outline the key details of our approach and provide a proof outline for our theoretical results. In \Cref{sec:gp_reg_overview}, we give proof sketches of our main results. In \Cref{sec:blw_to_l2}, we prove some background results that appear in the main body, particularly about block Lewis weights. In \Cref{sec:mirror_descent}, we give an analysis of mirror descent under inexact subproblem solves -- we will need this in the proof of \Cref{mainthm:gp_regression_iteration_complexity}. In \Cref{sec:optimal_ms_acceleration}, we modify an acceleration scheme due to \cite{chjjs22}, which we will use to iterate calls to the proximal subproblem solver \eqref{eq:gp_prox_intro_reg} for the proof of \Cref{mainthm:gp_regression_iteration_complexity}. In \Cref{sec:gp_robust_proof}, we prove \Cref{mainthm:fair_regression_iteration_complexity}. In \Cref{sec:gp_interpolating_proof}, we prove \Cref{mainthm:gp_regression_iteration_complexity}. Finally, in \Cref{app:exps} we include an empirical comparison of our proposed algorithms against the aforementioned baselines.

%% file: related.tex
Here, we discuss prior work that conceptually and technically relates to ours. We then suggest natural directions for future work.

\paragraph{Multi-distribution learning.}
Many learning problems involve multiple data sources, for instance, when multiple agents generate their data independently. One can formulate these multi-distribution problems as standard learning/optimization problems by considering a mixture of their distributions, as in objective \ref{eq:utilitarian}. However, this approach often biases solutions toward dominant data sources, leading to poor performance on outliers—an issue stemming from statistical heterogeneity. This limitation motivates the study of multi-objective optimization problems~\citep{miettinen1999nonlinear, ehrgott2005multicriteria}, where each agent \( m \) has a distribution \( \mathcal{D}_m \) that defines its objective as \( \mathbb{E}_{z\sim \mathcal{D}_m}[f(\vx_m;z)] \), and where models \( \vx_m \) can vary across agents---a framework known as personalization. 

One of the earliest algorithms for such problems was introduced by \citet{blum2017collaborative}, where each agent's objective must be minimized to a pre-specified threshold \( \epsilon \) with high probability, framed within a PAC learning framework~\citep{valiant1984theory,vapnik2013nature}. Subsequent research has refined these algorithms, achieving optimal sample complexity guarantees for learning from multiple distributions~\citep{chen2018tight,nguyen2018improved,hanneke2019value,haghtalab2022demand,zhang2024optimal}. Our objectives \( \ref{eq:main_objective} \) and \( \ref{eq:interpolation} \) offer different approaches to multi-distribution learning, where data distributions correspond to empirical agent distributions. In particular, \citet{mohri2019agnostic} analyzed objective \( \ref{eq:main_objective} \) to establish generalization bounds for unknown mixtures of agents' distributions. 

Beyond sample efficiency, researchers have also examined other challenges, such as communication costs in large-scale distributed optimization~\citep{mcmahan2016s}. A particularly relevant study is that of \citet{bullins2021stochastic}, which employs an efficient distributed quadratic sub-solver~\citep{woodworth2020local,patel2024limits} to implement an inexact Newton method for optimizing quasi-self-concordant functions (see \Cref{defn:qsc}).

\paragraph{Group fairness.} 
Recently, interest in algorithmic fairness has intensified~\citep{barocas2016big, abebe2020roles, kasy2021fairness} with researchers exploring fairness across various domains, including supervised learning~\citep{calders2009building, dwork2012fairness, hardt2016equality, kusner2017counterfactual, goel2018non, ustun2019fairness}, resource allocation~\citep{bertsimas2011price, bertsimas2012efficiency, hooker2012combining, donahue2020fairness, manshadi2021fair}, scheduling~\citep{mulvany2021fair}, online matching~\citep{chierichetti2019matroids, ma2023fairness}, assortment planning~\citep{singh2018fairness, biega2018equity, singh2019policy, chen2022fair}, and facility location~\citep{gupta2022socially}. The extensive literature on algorithmic fairness falls into three main categories: (1) individual fairness, which ensures that similar individuals receive comparable predictions~\citep{dwork2012fairness, loi2019philosophical, chen2022fair}, (2) group fairness, which aims for equal treatment of different demographic groups, often in terms of resource allocation or performance parity~\citep{singh2018fairness, balseiro2021regularized}, and (3) subgroup fairness, which blends aspects of both individual and group fairness~\citep{kearns2018preventing, kearns2019empirical}. 

This paper focuses on a well-studied group fairness notion in machine learning literature: the group DRO problem~\citep{ben2013robust, duchi2016statistics, sagawa2019distributionally}. The idea of interpolating between robustness and utility is also common~\citep{golrezaei2024online} and closely related to multi-objective optimization, where scalarization~\citep{miettinen1999nonlinear, ehrgott2005multicriteria} helps recover desired solutions along the Pareto frontier. 

\paragraph{Linear programming and $\ell_p$ regression.} 
In the last several years, there has been a surge of work in obtaining second-order, condition-free algorithms for linear programming and $\ell_p$ regression \citep{bcll18,ls19,akps19,jls21}. Observe that $\ell_p$ regression is a special case of the problem we study in objective \eqref{eq:interpolation}, which is recovered when all $n_i=1$, and $\ell_{\infty}$ regression is captured by linear programming. Note that neither of these problem families is expressive enough to capture the objectives we study. In general, to achieve iteration complexities in the smaller of the two dimensions for these problems, it appears that a geometric understanding of the solution space is required --- these ideas were central to the improvements obtained by \citet{ls19,jls21} as well as our work.

\paragraph{Open problems.} 
Our work raises several open questions. One limitation of \Cref{mainthm:fair_regression_iteration_complexity} is that its iteration complexity is not high-accuracy, meaning its dependence on \( \eps \) is not \(\mathrm{polylog}(1/\eps)\). Designing a high-accuracy solver under the same conditions as \Cref{mainthm:gp_regression_iteration_complexity} with iteration complexity \(\widetilde{O}\inparen{\mathsf{poly}(\min\inbraces{\rank{\mA}, m}, \logv{\frac{1}{\eps}})}\) remains an open problem.

A more ambitious general goal is to design algorithms for convex quadratic programs with the aforementioned iteration complexity. This would generalize analogous results for linear programming \citep{ls19}. We view the current work as a first step towards this goal, as the objective \eqref{eq:main_objective} is a structured convex quadratic program for which we get an iteration complexity independent of $m$. It would also be interesting to consider other complexity measures beyond $\rank{\mA}$, for instance, assumptions about the ground-truth labeling vector $\vx_i^\star$ for each group's data $S_i$.

Finally, our results suggest that optimizing for ``$\ell_p$-interpolants'' between non-robust and robust objectives may be computationally easier than optimizing for the robust objective alone. A more precise statistical characterization of how robustness and utility trade-off as \( p \) varies in collaborative, fair, or multi-distributional learning settings would be valuable. Additionally, exploring interpolations or solution concepts along the Pareto frontier of the \( m \)-dimensional multi-objective optimization problem or other DRO notions (eg Wasserstein DRO \citep{bkm19, cpo20}) could yield further insights.

%% file: other_proofs.tex
\section{Block Lewis Weights and their Properties}
\label{sec:blw_to_l2}

In this section, we introduce \textit{block Lewis weights} and explore some of their properties. Several of these statements can be found in \cite{jlls23,mo23}, but we include definitions and proofs here for self-completion.

We first need to define \textit{leverage scores}.

\begin{definition}[Leverage scores]
\label{defn:leverage_score}
For a matrix $\mA\in\R^{n \times d}$ with rows $\va_1,\dots,\va_n$, let $\tau_j$ denote the $j$th \textit{leverage score} of $\mA$, which we define to be
\begin{align*}
    \tau_j(\mA) \coloneqq \max_{\vx\in\R^d\setminus \inbraces{0}} \frac{\ip{\va_j,\vx}^2}{\norm{\mA\vx}_2^2} = \va_j^{\top}\inparen{\mA^{\top}\mA}^{-1}\va_j\enspace.
\end{align*}
\end{definition}

We now introduce the main object of interest in this section, \Cref{defn:blw}. Our version of the definition is adapted from \cite[Definition 1.2]{mo23} (there, we set $p_1=\dots=p_m=2$, let their $\mW = \mI$, replace $\vlambda$ with $\vw/\norm{\vw}_1$, and rescale $\Fstar$ appropriately).

\begin{definition}[Adapted from {\cite[Definition 1.2]{mo23}}]
\label{defn:blw}
Let $\vw \in \R^m_{\ge 0}$ and $\mW \in \R^{n \times n}_{\ge 0}$ be a diagonal matrix for which for all $j \in S_i$, we have $\mW_{jj} = w_i$. Let $p > 0$. We say that $\vw$ is a block Lewis overestimate if for all $i \in [m]$, we have
\begin{align*}
    \frac{\sum_{j \in S_i} \tau_j\inparen{\mW^{\frac{1}{2}-\frac{1}{p}}\mA}}{w_i} \le 1\enspace.
\end{align*}
\end{definition}

The main reason that \Cref{defn:blw} is interesting is that it gives us a formula with which we can relate the level sets of the group norm $\gnorm{\cdot}{p}$ to $\ell_2$. See \Cref{thm:blw_to_l2}. 

\begin{theorem}[Block Lewis weights give us ellipsoidal approximations to $\gnorm{\cdot}{p}$]
\label{thm:blw_to_l2}
Let $p \ge 2$. If $\vw$ is a block Lewis overestimate, then for all $\vx\in\R^d$, we have
\begin{align*}
    \frac{\norm{\mW^{\frac{1}{2}-\frac{1}{p}}\mA\vx}_2}{\norm{\vw}_1^{\frac{1}{2}-\frac{1}{p}}} \le \gnorm{\mA\vx}{p} \le \norm{\mW^{\frac{1}{2}-\frac{1}{p}}\mA\vx}_2\enspace.
\end{align*}
\end{theorem}

We prove \Cref{thm:blw_to_l2} in \Cref{sec:blw_to_l2}. An analogous statement can also be shown for $p\le2$, but since we do not use it in this paper, we do not write it here.

Observe that if we can get $\vw$ that satisfies \Cref{defn:blw} and for which $\norm{\vw}_1 = \rank{\mA}$, then \Cref{thm:blw_to_l2} gives us the optimal relationship between $\ell_2$ and $\gnorm{\cdot}{p}$ whenever $\rank{\mA} \le m$. Furthermore, for intuition, suppose $p=\infty$. By John's theorem, we know that for any symmetric convex body, there exists an ellipsoid such that the ellipsoid approximates the convex body up to a $\sqrt{d}$ distortion. Moreover, this is worst-case tight (e.g. the best distortion we can get when we approximate $\ell_1^d$ with $\ell_2$ is $\sqrt{d}$). Thus, assuming we can find $\norm{\vw}_1 \approx \rank{\mA}$, in this case, we get a guarantee that is similar to what John's theorem tells us.

Now, assuming we can find a low-distortion ellipsoidal approximation to the level sets of our loss, we get that the ``effective'' diameter of our problem is $\sim \sqrt{d}$. Combining this and the discussion in \Cref{sec:gp_reg_overview_blw} (or, more formally, \Cref{thm:optimal_ms_acceleration}), we can see why we should expect an iteration complexity of $\sim d^{1/3}$ (or better, if we can find a better ellipsoid).

What is left is whether weights $\vw$ satisfying \Cref{defn:blw} with small sum can be found. To this end, we invoke \cite[Algorithm 2]{mo23}.

\begin{theorem}[{\cite[Algorithm 2 and Lemma 5.6]{mo23}}]
\label{thm:gp_regression_blw_alg}
There exists an algorithm that returns a block Lewis overestimate $\vw$ for which $\norm{\vw}_1 \le 2\rank{\mA}$. The algorithm runs in $O(\log m)$ linear system solves with matrices of the form $\mA^{\top}\mD\mA$ for nonnegative diagonal $\mD$.
\end{theorem}

Thus, by applying \Cref{thm:gp_regression_blw_alg} as a preprocessing step, we get an $\ell_2$ geometry under which we can run the accelerated proximal algorithms. As an example of the power of this, observe the following.

\begin{lemma}
\label{lemma:gp_regression_initialization}
Consider the matrix $\widehat{\mA} \coloneqq \mA \vert \vb \in \R^{n \times (d+1)}$ that is formed by appending the column vector $\vb$ to the right of the matrix $\mA$. If we have a vector $\vw$ of block Lewis overestimates for the matrix $\widehat{\mA}$, then there exists an algorithm that finds an initialization $\vx_0$ for which
\begin{equation*}
    \begin{aligned}
        \norm{\vx_0-\xstar}_{\mA^{\top}\mW^{\frac{1}{2}-\frac{1}{p}}\mA} &\le 2\inparen{2\rank{\mA}}^{\frac{1}{2}-\frac{1}{p}}\gnorm{\mA\xstar-\vb}{p} \\
        \gnorm{\mA\vx_0-\vb}{p} &\le \inparen{2\rank{\mA}}^{\frac{1}{2}-\frac{1}{p}}\gnorm{\mA\xstar-\vb}{p}
    \end{aligned}\enspace .
\end{equation*}
The algorithm runs in $1$ linear system solve in $\widehat{\mA}^{\top}\mD\widehat{\mA}$.
\end{lemma}

\begin{proof}[Proof of \Cref{lemma:gp_regression_initialization}]
 By \Cref{thm:blw_to_l2}, our weights $\vw$ are such that for all $\vx\in\R^n$ and reals $c \in \R$,
\begin{align*}
    \frac{\norm{\mW^{\frac{1}{2}-\frac{1}{p}}\mA\vx-c\mW^{\frac{1}{2}-\frac{1}{p}}\vb}_2}{(2(d+1))^{\frac{1}{2}-\frac{1}{p}}} \le \gnorm{\mA\vx-c\vb}{p} \le \norm{\mW^{\frac{1}{2}-\frac{1}{p}}\mA\vx-c\mW^{\frac{1}{2}-\frac{1}{p}}\vb}_2.
\end{align*}
Let $\vx_0$ be the solution to the least squares regression problem
\begin{align*}
    \vx_0 &\coloneqq \argmin{\vx\in\R^d} \norm{\mW^{\frac{1}{2}-\frac{1}{p}}\mA\vx-\mW^{\frac{1}{2}-\frac{1}{p}}\vb}_2 = \inparen{\mA^{\top}\mW^{1-\frac{2}{p}}\mA}^{-1}\mA^{\top}\mW^{\frac{1}{2}-\frac{1}{p}}\vb.
\end{align*}
It is easy to see that computing $\vx_0$ amounts to $1$ linear system solve in $\mA^{\top}\mD\mA$.

Next, let $\mM \coloneqq \mA^{\top}\mW^{1-\frac{2}{p}}\mA$ and observe that
\begin{align*}
    \norm{\vx_0-\xstar}_{\mM} &= \norm{\inparen{\mW^{\frac{1}{2}-\frac{1}{p}}\mA\vx_0-\mW^{\frac{1}{2}-\frac{1}{p}}\vb} - \inparen{\mW^{\frac{1}{2}-\frac{1}{p}}\mA\xstar-\mW^{\frac{1}{2}-\frac{1}{p}}\vb}}_2 \\
    &\le 2\norm{\mW^{\frac{1}{2}-\frac{1}{p}}\mA\xstar-\mW^{\frac{1}{2}-\frac{1}{p}}\vb}_2 \le 2\inparen{2d}^{\frac{1}{2}-\frac{1}{p}}\gnorm{\mA\xstar-\vb}{p}.
\end{align*}
Finally, write
\begin{align*}
    \gnorm{\mA\vx_0-\vb}{p} &\le \norm{\mW^{\frac{1}{2}-\frac{1}{p}}\mA\vx_0-\mW^{\frac{1}{2}-\frac{1}{p}}\vb}_2 \\
    &\le \norm{\mW^{\frac{1}{2}-\frac{1}{p}}\mA\xstar-\mW^{\frac{1}{2}-\frac{1}{p}}\vb}_2 \le \inparen{2d}^{\frac{1}{2}-\frac{1}{p}}\gnorm{\mA\xstar-\vb}{p},
\end{align*}
giving us the conclusion of \Cref{lemma:gp_regression_initialization}.
\end{proof}

\begin{proof}[Proof of \Cref{thm:blw_to_l2}]
Let $\vlambda \coloneqq \vw/\norm{\vw}_1$ and $\mLambda \coloneqq \mW/\norm{\vw}_1$. It is easy to check that $\vlambda$ is a probability measure on $[m]$. When $p \ge 2$, using monotonicity of $L_p$ norms taken under probability measures, we get
\begin{align*}
    \inparen{\sum_{i=1}^m \norm{\mA_{S_i}\vx}_2^p}^{\frac{1}{p}} = \inparen{\sum_{i=1}^m \lambda_i\norm{\lambda_i^{-\frac{1}{p}}\mA_{S_i}\vx}_2^p}^{\frac{1}{p}} \ge \inparen{\sum_{i=1}^m \lambda_i\norm{\lambda_i^{-\frac{1}{p}}\mA_{S_i}\vx}_2^2}^{1/2}.
\end{align*}
Expanding the RHS and substituting $\lambda_i = w_i/\norm{\vw}_1$ gives
\begin{align*}
    \gnorm{\mA\vx}{p} \ge \frac{\norm{\mW^{\frac{1}{2}-\frac{1}{p}}\mA\vx}_2}{\norm{\vw}_1^{\frac{1}{2}-\frac{1}{p}}}.
\end{align*}
For the ``hard'' direction, we will use \Cref{defn:blw} in a nontrivial way. Notice that
\begin{align*}
    &\inparen{\sum_{i=1}^m w_i\norm{w_i^{-\frac{1}{p}}\mA_{S_i}\vx}_2^p}^{\frac{1}{p}} = \inparen{\sum_{i=1}^m w_i\norm{w_i^{-\frac{1}{p}}\mA_{S_i}\vx}_2^2\norm{w_i^{-\frac{1}{p}}\mA_{S_i}\vx}_2^{p-2}}^{\frac{1}{p}} \\
    &\le \inparen{\sum_{i=1}^m w_i\norm{w_i^{-\frac{1}{p}}\mA_{S_i}\vx}_2^2 \cdot \max_{\vx\in\R^d\setminus\inbraces{0}}\frac{\norm{w_i^{-\frac{1}{p}}\mA_{S_i}\vx}_2^{p-2}}{\norm{\mW^{\frac{1}{2}-\frac{1}{p}}\mA\vx}_2^{p-2}}}^{\frac{1}{p}} \\
    &= \inparen{\sum_{i=1}^m w_i\norm{w_i^{-\frac{1}{p}}\mA_{S_i}\vx}_2^2 \cdot \inparen{\max_{\vx\in\R^d\setminus\inbraces{0}}\frac{\norm{w_i^{-\frac{1}{p}}\mA_{S_i}\vx}_2^{2}}{\norm{\mW^{\frac{1}{2}-\frac{1}{p}}\mA\vx}_2^{2}}}^{\frac{p}{2}-1}\cdot\norm{\mW^{\frac{1}{2}-\frac{1}{p}}\mA\vx}_{2}^{p-2}}^{\frac{1}{p}} \\
    &\le \inparen{\sum_{i=1}^m w_i\norm{w_i^{-\frac{1}{p}}\mA_{S_i}\vx}_2^2 \cdot \inparen{\frac{\sum_{j\in S_i}\tau_j\inparen{\mW^{\frac{1}{2}-\frac{1}{p}}\mA}}{w_i}}^{\frac{p}{2}-1}\norm{\mW^{\frac{1}{2}-\frac{1}{p}}\mA\vx}_{2}^{p-2}}^{\frac{1}{p}} \\
    &\overset{\text{\Cref{defn:blw}}}{\le} \inparen{\sum_{i=1}^m w_i\norm{w_i^{-\frac{1}{p}}\mA_{S_i}\vx}_2^2\norm{\mW^{\frac{1}{2}-\frac{1}{p}}\mA\vx}_{2}^{p-2}}^{\frac{1}{p}} = \norm{\mW^{\frac{1}{2}-\frac{1}{p}}\mA\vx}_2,
\end{align*}
so combining our upper and lower bounds gives the conclusion of \Cref{thm:blw_to_l2}.
\end{proof}

%% file: mirror_descent.tex
\section{Mirror Descent with Inexact Updates}
\label{sec:mirror_descent}

\paragraph{Notation warning.} This section is meant to be a self-contained, standalone analysis of mirror descent under inexact updates. The notation is chosen to be consistent with most material we could find on mirror descent and therefore conflicts with the notation used in the rest of the paper.

In this section, we give an analysis of unconstrained mirror descent when each Bregman proximal problem is solved only approximately (\Cref{alg:mirror_descent}). Although we expect that this is a standard fact about mirror descent, we could not find an appropriate reference. Hence, we produce it here.

\begin{algorithm}[H]
\caption{\textsf{ApproximateMirrorDescent}: Implements mirror descent to optimize convex and differentiable $f$ given $L$-relative smoothness and $\mu$-relative strong convexity in the reference $h$ when we may not be able to solve each proximal problem exactly.}
\label{alg:mirror_descent}
\begin{algorithmic}[1]
\Require Initial point $\vx_0$, iteration count $T$.
\State Define \begin{equation*}
    \begin{aligned}
    D_{h}(\vx,\vy) &\coloneqq h(\vx)-h(\vy)-\ip{\nabla h(\vy),\vx-\vy} \\
    \xstar &\coloneqq \argmin{\vx\in\R^d} f(\vx)
    \end{aligned}.
\end{equation*}
\For{$i = 1, \dots, T$}
    \State $\xstar_i = \argmin{\xtilde \in \R^d} f(\vx_{i-1}) + \ip{\nabla f(\vx_{i-1}), \xtilde-\vx_{i-1}} + LD_h(\xtilde,\vx_{i-1})$ \label{line:mirror_descent_exact} \Comment{We may only be able to approximate $\xstar_i$ -- see the next line.}
    \State Let $\vx_i$ be an approximate stationary point for the above objective.\label{line:mirror_descent_approx}
\EndFor
\Return $\argmin{0 \le i \le T} f(\vx_i)$
\end{algorithmic}
\end{algorithm}

In \Cref{alg:mirror_descent}, we assume that the function $f$ is $\mu$-relatively strongly convex and $L$-smooth in a \textit{reference function} $h$. This means that for all $\vx,\vy\in\R^d$, we have
\begin{align*}
    \mu D_h(\vx,\vy) \le f(\vx) - f(\vy) - \ip{\nabla f(\vy), \vx-\vy} \le  LD_h(\vx,\vy).
\end{align*}
Using \cite[Proposition 1.1]{lfn18}, when $f$ is twice-differentiable, this condition is equivalent to asking for all $\vx\in\R^d$, 
\begin{align*}
    \mu \nabla^2 h(\vx) \preceq \nabla^2 f(\vx) \preceq L\nabla^2 h(\vx).
\end{align*}
We are now ready to state the performance guarantee of \Cref{alg:mirror_descent}. See \Cref{thm:gp_mirror_descent}.

\begin{theorem}
\label{thm:gp_mirror_descent}
Let index $j$ be the index output by \Cref{alg:mirror_descent}. Let $\triangle_i$ be defined such that
\begin{align*}
    \triangle_i \coloneqq \nabla f(\vx_{i-1}) + L\inparen{\nabla h(\vx_i) - \nabla h(\vx_{i-1})}.
\end{align*}
Then, we have
\begin{align*}
    f(\vx_j)-f(\xstar) \le L\inparen{1-\frac{\mu}{L}}^TD_h(\xstar,\vx_0)+\max_{1\le i \le n}\ip{\triangle_i,\vx_i-\xstar}.
\end{align*}
\end{theorem}

To prove \Cref{thm:gp_mirror_descent}, we begin with a few standard facts about the mirror descent iterations.

\begin{lemma}
\label{lemma:gp_md_threepoint}
Let $\vy\in\R^d$ be arbitrary. We have
\begin{align*}
    \ip{\nabla f(\vx_{i-1}), \vx_i-\vy} = L\inparen{D_h(\vy,\vx_{i-1})-D_h(\vy,\vx_i)-D_h(\vx_i,\vx_{i-1})} + \ip{\triangle_i, \vx_i-\vy}.
\end{align*}
\end{lemma}
\begin{proof}[Proof of \Cref{lemma:gp_md_threepoint}]
By the three point identity (see, e.g., \cite[Equation (A.9)]{syls16}), we have
\begin{align*}
    D_h(\vy,\vx_{i-1})-D_h(\vy,\vx_i)-D_h(\vx_i,\vx_{i-1})&=-\ip{\nabla h(\vx_i)-\nabla h(\vx_{i-1}),\vx_i-\vy} \\
    &= \frac{1}{L}\ip{\nabla f(\vx_{i-1})-\triangle_i,\vx_i-\vy},
\end{align*}
completing the proof of \Cref{lemma:gp_md_threepoint}.
\end{proof}
\begin{lemma}[Mirror descent lemma under approximate stationary point updates]
\label{lemma:gp_md_singlestep}
Let $\vy\in\R^d$ be arbitrary. For every iteration $i$, we have
\begin{align*}
    f(\vx_i)-f(\vy) \le (L-\mu)D_h(\vy,\vx_{i-1}) - LD_h(\vy,\vx_i) + \ip{\triangle_i,\vx_i-\vy}.
\end{align*}
\end{lemma}
\begin{proof}[Proof of \Cref{lemma:gp_md_singlestep}]
The definition of $\mu$-relative strong convexity tells us that
\begin{align*}
    f(\vx_{i-1})-f(\vy) \le \ip{\nabla f(\vx_{i-1}), \vx_{i-1}-\vy} - \mu D_h(\vy,\vx_{i-1}).
\end{align*}
We now write
\begin{align*}
    f(\vx_i)-f(\vy) &\le f(\vx_{i-1})-f(\vy) + \ip{\nabla f(\vx_{i-1}), \vx_i-\vx_{i-1}} + LD_h(\vx_i,\vx_{i-1}) & \text{($L$-RS)} \\
    &\le \ip{\nabla f(\vx_{i-1}), \vx_i-\vy}-\mu D_h(\vy,\vx_{i-1}) + LD_h(\vx_i,\vx_{i-1}) & \text{($\mu$-RSC)} \\
    &\le (L-\mu)D_h(\vy,\vx_{i-1}) - LD_h(\vy,\vx_i) + \ip{\triangle_i,\vx_i-\vy}, & \text{(\Cref{lemma:gp_md_threepoint})}
\end{align*}
completing the proof of \Cref{lemma:gp_md_singlestep}.
\end{proof}

We now have the tools to complete the proof of \Cref{thm:gp_mirror_descent}.

\begin{proof}[Proof of \Cref{thm:gp_mirror_descent}]
Let $E_i \coloneqq f(\vx_i)-f(\xstar)-\ip{\triangle_i,\vx_i-\xstar}$. Substituting $\vy=\xstar$ and rearranging the conclusion of \Cref{lemma:gp_md_singlestep} gives
\begin{align*}
    E_i \le (L-\mu)D_h(\xstar,\vx_{i-1})-LD_h(\xstar,\vx_i).
\end{align*}
We multiply both sides by $\inparen{\frac{L}{L-\mu}}^i$ and write
\begin{align*}
    \inparen{\frac{L}{L-\mu}}^iE_i \le \frac{L^i}{(L-\mu)^{i-1}}D_h(\xstar,\vx_{i-1})-\frac{L^{i+1}}{(L-\mu)^{i}}D_h(\xstar,\vx_i).
\end{align*}
Adding over all $T$ iterations yields
\begin{align*}
    \sum_{i=1}^T \inparen{\frac{L}{L-\mu}}^iE_i \le LD_h(\xstar,\vx_0)-\inparen{\frac{L}{L-\mu}}^TLD_h(\xstar,\vx_T) \le LD_h(\xstar,\vx_0).
\end{align*}
Expanding out the definition of $E_i$ and rearranging gives
\begin{align*}
    \sum_{i=1}^T \inparen{\frac{L}{L-\mu}}^i(f(\vx_i)-f(\xstar)) \le LD_h(\xstar,\vx_0) + \sum_{i=1}^T \inparen{\frac{L}{L-\mu}}^i\ip{\triangle_i,\vx_i-\xstar}.
\end{align*}
By the geometric series summation formula, we define and have
\begin{align*}
    C_T \coloneqq \sum_{i=1}^T \inparen{\frac{L}{L-\mu}}^i = \frac{L}{\mu}\inparen{\inparen{1+\frac{\mu}{L-\mu}}^T-1}.
\end{align*}
Let $j$ be the index that \Cref{alg:mirror_descent} returns. It is easy to check that
\begin{align*}
    \sum_{i=1}^T \inparen{\frac{L}{L-\mu}}^i(f(\vx_i)-f(\xstar)) \ge C_T\inparen{f(\vx_j)-f(\xstar)}
\end{align*}
and
\begin{align*}
    \sum_{i=1}^T \inparen{\frac{L}{L-\mu}}^i\ip{\triangle_i,\vx_i-\xstar} \le C_T\max_{1 \le i \le n}\ip{\triangle_i,\vx_i-\xstar}.
\end{align*}
This gives us
\begin{align*}
    f(\vx_j)-f(\xstar) \le \frac{L}{C_T}D_h(\xstar,\vx_0)+\max_{1 \le i \le n}\ip{\triangle_i,\vx_i-\xstar}.
\end{align*}
Finally, notice that
\begin{align*}
    \frac{L}{C_T} = \frac{\mu}{\inparen{1+\frac{\mu}{L-\mu}}^T-1} \le L\inparen{1-\frac{\mu}{L}}^T.
\end{align*}
Combining everything completes the proof of \Cref{thm:gp_mirror_descent}.
\end{proof}

Finally, we add another useful lemma that quantifies the descent, if any, in the objective value between iterations.

\begin{lemma}
\label{lemma:gp_md_descent}
For every iteration $i$, we have
\begin{align*}
    f(\vx_i)-f(\vx_{i-1}) \le -LD_h(\vx_{i-1},\vx_i) + \ip{\triangle_i,\vx_i-\vx_{i-1}}.
\end{align*}
In particular, if $\ip{\triangle_i,\vx_{i}-\vx_{i-1}} \le LD_h(\vx_{i-1},\vx_i)$, then iteration $i$ is a descent step.
\end{lemma}
\begin{proof}[Proof of \Cref{lemma:gp_md_descent}]
We substitute $\vy=\vx_{i-1}$ in the conclusion of \Cref{lemma:gp_md_singlestep}. This gives
\begin{align*}
    f(\vx_i)-f(\vx_{i-1}) \le -LD_h(\vx_{i-1},\vx_i) + \ip{\triangle_i,\vx_i-\vx_{i-1}},
\end{align*}
completing the proof of \Cref{lemma:gp_md_descent}.
\end{proof}

%% file: improved_ms.tex
\section{Optimal MS Acceleration under Custom Euclidean Geometry}
\label{sec:optimal_ms_acceleration}

In this section, we adapt the bisection-free Monteiro-Svaiter acceleration framework developed in \cite{chjjs22} to handle custom Euclidean geometries. The object of interest here is \Cref{alg:optimal_ms_acceleration}, which we will call with different choices of the oracle $\oms$ for our algorithms.

\begin{algorithm}[H]
\caption{\textsf{OptimalMSAcceleration}: optimizes function $f$ given MS oracle $\oms$.}
\label{alg:optimal_ms_acceleration}
\begin{algorithmic}[1]
\Require Initial $\vx_0$, function $f$, oracle $\oms$, initial $\lambda_0'$, multiplicative adjustment factor $\alpha >1$, iteration count $T$
\State Set $\vv_0 = \vx_0$, $A_0 = 0$, $A'_0 = 0$.
\State Set $\xtilde_1, \lambda_1 = \cO(\vx_0; \lambda_0')$ and $\lambda_1' = \lambda_1$.
\For{$t=0,\dots,T$}
    \State $a_{t+1}' = \frac{1}{2\lambda_{t+1}'}\inparen{1+\sqrt{1+4\lambda'_{t+1}A_t}}$
    \State $A_{t+1}' = A_t+a_{t+1}'$
    \State $\vq_t = \frac{A_t}{A_{t+1}'}\vx_t + \frac{a_{t+1}'}{A_{t+1}'}\vv_t$\label{line:optimal_ms_query}
    \If{$t > 0$} $\xtilde_{t+1},\lambda_{t+1}=\oms(\vq_t;\lambda_{t+1}')$ \label{line:optimal_ms_acceleration_query}\EndIf
    \State $\gamma_{t+1} = \min\inbraces{1, \frac{\lambda_{t+1}'}{\lambda_{t+1}}}$
    \State $a_{t+1} = \gamma_{t+1}a'_{t+1}$ and $A_{t+1} = A_{t} + a_{t+1}$ \Comment{$A_{t+1} = A_{t+1}'-(1-\gamma_{t+1})a'_{t+1}$}
    \State $\vx_{t+1} = \frac{(1-\gamma_{t+1})A_{t}}{A_{t+1}}\vx_t + \frac{\gamma_{t+1}A_{t+1}'}{A_{t+1}}\xtilde_{t+1}$
    \If{$\gamma_{t+1} = 1$}
        \State $\lambda_{t+2}' = \frac{1}{\alpha}\lambda'_{t+1}$
    \Else
        \State $\lambda_{t+1}' = \alpha\lambda_{t+1}'$
    \EndIf
    \State $\vv_{t+1} = \vv_{t}-a_{t+1}\mM^{-1}\nabla f(\xtilde_{t+1})$
\EndFor
\end{algorithmic}
\end{algorithm}

In order to state the performance guarantee of \Cref{alg:optimal_ms_acceleration}, we require the notions of an \textit{MS oracle} and a \textit{movement bound}. See \Cref{defn:ms_oracle} and \Cref{defn:movement_bound}.

\begin{definition}[MS oracle, generalization of {\cite[Definition 1]{chjjs22}}]
\label{defn:ms_oracle}
Let $\mM \in \R^{d\times d}$ be a positive semidefinite matrix. An oracle $\myfunc{\cO}{\R^d \times \R_{\ge 0}}{\R^d \times \R_{\ge 0}}$ is a $\sigma$-MS oracle for function $\myfunc{f}{\R^d}{\R}$ if for every $\vq\in\R^d$ and $\lambda' > 0$, the points $(\vx,\lambda)=\cO(\vq;\lambda')$ satisfy
\begin{align*}
    \norm{\vx-\vq+\frac{1}{\lambda}\mM^{-1}\nabla f(\vx)}_{\mM} \le \sigma\norm{\vx-\vq}_{\mM}.
\end{align*}
\end{definition}

\begin{definition}[Movement bound {\cite[Definition 2]{chjjs22}}]
\label{defn:movement_bound}
For a norm $\norm{\cdot}_{\mM}$ induced by positive semidefinite $\mM\in\R^{d\times d}$, numbers $s \ge 1, c, \lambda > 0$, and $\vx,\vy\in\R^d$, we say that $(\vx,\vy,\lambda)$ satisfies a $(s,c)$-movement bound if
\begin{align*}
    \norm{\vx-\vy}_{\mM} \ge \begin{cases} \inparen{\frac{\lambda}{c^s}}^{\frac{1}{s-1}} & \text{ if } s < \infty \\ \frac{1}{c} & \text{ if } s = \infty \end{cases}.
\end{align*}
\end{definition}

With these in hand, we are ready to state the convergence guarantee we get with \Cref{alg:optimal_ms_acceleration}. See \Cref{thm:optimal_ms_acceleration}.

\begin{theorem}[Modification of {\cite[Theorem 1]{chjjs22}}]
\label{thm:optimal_ms_acceleration}
Let $\myfunc{f}{\R^d}{\R}$ be convex and differentiable. Consider running \Cref{alg:optimal_ms_acceleration} with parameters $\alpha = \expv{3-\frac{2}{s+1}}$ and a $\sigma$-MS oracle with $0 \le \sigma < 0.99$ (\Cref{defn:ms_oracle}). Let $s\ge1$ and $c>0$ and suppose that for all $t$ such that $\lambda_t>\lambda'_t$ or $t=1$, the iterates $(\xtilde_t,\vq_{t-1},\lambda_t)$ satisfy an $(s,c)$-movement bound (\Cref{defn:movement_bound}). Let $C$ be a universal constant. For any iteration count $T$ satisfying
\begin{align*}
    T \ge C\begin{cases} s\inparen{\frac{c^s\norm{\vx_0-\xstar}_{\mM}^{s+1}}{\eps}}^{\frac{2}{3s+1}} & \text{ if } s < \infty \\ \inparen{c\norm{\vx_0-\xstar}_{\mM}}^{2/3}\logv{\frac{\lambda_1\norm{\vx_0-\xstar}_{\mM}^2}{\eps}} & \text{ if } s = \infty \end{cases},
\end{align*}
we have
\begin{align*}
    f(\vx_T)-f(\xstar) \le \eps.
\end{align*}
\end{theorem}

The proof of \Cref{thm:optimal_ms_acceleration} follows the same recipe as the proof of \cite[Theorem 1]{chjjs22}. The only modification needed is that stated in \Cref{lemma:ms_acceleration_potential}.

\begin{lemma}[Replaces {\cite[Proposition 1]{chjjs22}}]
\label{lemma:ms_acceleration_potential}
In the context of \Cref{thm:optimal_ms_acceleration}, let $E_t \coloneqq f(\vx_t) - f(\xstar), D_t \coloneqq \frac{1}{2}\norm{\vv_t-\xstar}_{\mM}^2, N_{t+1} \coloneqq \frac{1}{2}\norm{\xtilde_{t+1}-\vq_t}_{\mM}^2$. Then, for all $t \ge 0$, we have
\begin{align*}
    A_{t+1}E_{t+1}+D_{t+1}+(1-\sigma^2)A_{t+1}'\min\inbraces{\lambda_{t+1},\lambda_{t+1}'}N_{t+1} \le A_tE_t+D_t.
\end{align*}
Consequently, for all $T\ge1, \sqrt{A_T}\ge\frac{1}{2}\sum_{t \in \cS_{T}^{\le}} \frac{1}{\sqrt{\lambda_t'}}$,
\begin{align*}
    E_{T} \le \frac{D_0}{A_T} \quad \text{ and } \quad (1-\sigma^2)\sum_{t \in \cS_{T}^{\ge}} A_t\lambda_t'N_t \le D_0 - A_TE_T.
\end{align*}
\end{lemma}
\begin{proof}[Proof of \Cref{lemma:ms_acceleration_potential}]
This proof is a straightforward modification of \cite[Proposition 1]{chjjs22}. We have
\begin{align*}
    D_{t+1} &= \frac{1}{2}\norm{\vv_{t+1}-\xstar}_{\mM}^2 = \frac{1}{2}\norm{\vv_t-a_{t+1}\mM^{-1}\nabla f(\xtilde_{t+1})-\xstar}_{\mM}^2 \\
    &= D_t + a_{t+1}\ip{\mM^{-1}\nabla f(\xtilde_{t+1}), \xstar-\vv_t}_{\mM} + \frac{a_{t+1}^2}{2}\norm{\mM^{-1}\nabla f(\xtilde_{t+1})}_{\mM}^2.
\end{align*}
By definition of $\vq_t$ and $A_{t+1}' = A_t + a'_{t+1}$, we have
\begin{align*}
    a_{t+1}'\vv_t = A_{t+1}'\vq_t - A_t\vx_t = a_{t+1}'\xtilde_{t+1} + A'_{t+1}\inparen{\vq_t-\xtilde_{t+1}} - A_t\inparen{\vx_t-\xtilde_{t+1}}.
\end{align*}
Subtracting $a_{t+1}'\xstar$ and taking the inner product with $\mM^{-1}\nabla f(\xtilde_{t+1})$ gives
\begin{align*}
    &\quad a_{t+1}'\ip{\mM^{-1}\nabla f(\xtilde_{t+1}), \xstar-\vv_t}_{\mM} \\
    &= \ip{\mM^{-1}\nabla f(\xtilde_{t+1}),  a_{t+1}'(\xstar-\xtilde_{t+1}) + A'_{t+1}\inparen{\xtilde_{t+1}-\vq_t} + A_t\inparen{\vx_t-\xtilde_{t+1}}}_{\mM} \\
    &\le a_{t+1}'\inparen{f(\xstar)-f(\xtilde_{t+1})}+A_{t+1}'\ip{\mM^{-1}\nabla f(\xtilde_{t+1}),\xtilde_{t+1}-\vq_t}_{\mM} + A_t\inparen{f(\vx_t)-f(\xtilde_{t+1})} \\
    &\le A_tE_t-A_{t+1}'\inparen{f(\xtilde_{t+1})-f(\xstar)}+A_{t+1}'\ip{\mM^{-1}\nabla f(\xtilde_{t+1}),\xtilde_{t+1}-\vq_t}_{\mM}.
\end{align*}
Rearranging gives
\begin{align*}
    A_{t+1}'\inparen{f(\xtilde_{t+1})-f(\xstar)} &\le A_tE_t+a_{t+1}'\ip{\mM^{-1}\nabla f(\xtilde_{t+1}), \vv_t-\xstar}_{\mM}\\
    &\quad +A_{t+1}'\ip{\mM^{-1}\nabla f(\xtilde_{t+1}),\xtilde_{t+1}-\vq_t}_{\mM}.
\end{align*}
Next, recall that by \Cref{defn:ms_oracle}, we have
\begin{align*}
    \norm{\mM^{-1}\nabla f(\xtilde_{t+1})+\lambda_{t+1}\inparen{\xtilde_{t+1}-\vq_t}}_{\mM}^2 \le \lambda_{t+1}^2\sigma^2\norm{\xtilde_{t+1}-\vq_t}_{\mM}^2.
\end{align*}
We use this to write
\begin{align*}
    &\quad \lambda_{t+1}\ip{\mM^{-1}\nabla f(\xtilde_{t+1}),\xtilde_{t+1}-\vq_t}_{\mM} \\
    &= \frac{1}{2}\norm{\mM^{-1}\nabla f(\xtilde_{t+1})+\lambda_{t+1}(\xtilde_{t+1}-\vq_t)}_{\mM}^2 - \frac{1}{2}\norm{\mM^{-1}\nabla f(\xtilde_{t+1})}_{\mM}^2-\frac{\lambda_{t+1}^2}{2}\norm{\xtilde_{t+1}-\vq_t}_{\mM}^2 \\
    &\le -\lambda_{t+1}^2(1-\sigma^2)N_{t+1}-\frac{1}{2}\norm{\mM^{-1}\nabla f(\xtilde_{t+1})}_{\mM}^2,
\end{align*}
from which we conclude
\begin{align*}
    \ip{\mM^{-1}\nabla f(\xtilde_{t+1}),\xtilde_{t+1}-\vq_t}_{\mM} &\le -\lambda_{t+1}(1-\sigma^2)N_{t+1}-\frac{1}{2\lambda_{t+1}}\norm{\mM^{-1}\nabla f(\xtilde_{t+1})}_{\mM}^2.
\end{align*}
Substituting back gives
\begin{align*}
    A_{t+1}'\inparen{f(\xtilde_{t+1})-f(\xstar)} &\le A_tE_t+a_{t+1}'\ip{\mM^{-1}\nabla f(\xtilde_{t+1}), \vv_t-\xstar}_{\mM}\\
    &\quad +A_{t+1}'\ip{\mM^{-1}\nabla f(\xtilde_{t+1}),\xtilde_{t+1}-\vq_t}_{\mM} \\
    &\le A_tE_t+a_{t+1}'\ip{\mM^{-1}\nabla f(\xtilde_{t+1}), \vv_t-\xstar}_{\mM}\\
    &\quad -A_{t+1}'\lambda_{t+1}(1-\sigma^2)N_{t+1}-\frac{A_{t+1}'}{2\lambda_{t+1}}\norm{\mM^{-1}\nabla f(\xtilde_{t+1})}_{\mM}^2.
\end{align*}
Next, recall that $\gamma_{t+1}a'_{t+1}=a_{t+1}$ and $\gamma_{t+1}\lambda_{t+1}=\min\inbraces{\lambda_{t+1},\lambda_{t+1}'}$, by construction. Let $\widehat{\lambda}_{t+1} \coloneqq \min\inbraces{\lambda_{t+1},\lambda_{t+1}'}$ We multiply both sides by $\gamma_{t+1}$ and conclude
\begin{align*}
    \gamma_{t+1}A_{t+1}'\inparen{f(\xtilde_{t+1})-f(\xstar)}&\le \gamma_{t+1}A_tE_t+a_{t+1}\ip{\mM^{-1}\nabla f(\xtilde_{t+1}), \vv_t-\xstar}_{\mM}\\
    &\quad -A_{t+1}'\widehat{\lambda}_{t+1}(1-\sigma^2)N_{t+1}-\frac{\gamma_{t+1}A_{t+1}'}{2\lambda_{t+1}}\norm{\mM^{-1}\nabla f(\xtilde_{t+1})}_{\mM}^2. 
\end{align*}
Now, by convexity of $f$ and from the definition of $\vx_{t+1}$, we have
\begin{align*}
    f(\vx_{t+1})-f(\xstar) \le \frac{(1-\gamma_{t+1})A_t}{A_{t+1}}\inparen{f(\vx_{t})-f(\xstar)}+\frac{\gamma_{t+1}A_{t+1}'}{A_{t+1}}\inparen{f(\xtilde_{t+1})-f(\xstar)}.
\end{align*}
Recall the definition of $E_t$, multiply both sides by $A_{t+1}$, apply our bound on $\gamma_{t+1}A_{t+1}'\inparen{f(\xtilde_{t+1})-f(\xstar)}$, and we get
\begin{align*}
    A_{t+1}E_{t+1} &\le (1-\gamma_{t+1})A_tE_t + \gamma_{t+1}A_{t+1}'\inparen{f(\xtilde_{t+1})-f(\xstar)} \\
    &\le A_tE_t + a_{t+1}\ip{\mM^{-1}\nabla f(\xtilde_{t+1}), \vv_t-\xstar}_{\mM}\\
    &\quad -A_{t+1}'\widehat{\lambda}_{t+1}(1-\sigma^2)N_{t+1}-\frac{\gamma_{t+1}A_{t+1}'}{2\lambda_{t+1}}\norm{\mM^{-1}\nabla f(\xtilde_{t+1})}_{\mM}^2
\end{align*}
After shifting terms around, we see that it remains to show
\begin{align*}
    a_{t+1}\ip{\mM^{-1}\nabla f(\xtilde_{t+1}), \vv_t-\xstar}_{\mM}-\frac{\gamma_{t+1}A_{t+1}'}{2\lambda_{t+1}}\norm{\mM^{-1}\nabla f(\xtilde_{t+1})}_{\mM}^2 \overset{?}{\le} D_t-D_{t+1}.
\end{align*}
In fact, by the choice of $a'_{t+1}$ and the definition of $A_{t+1}'$, we have
\begin{align*}
    \lambda'_{t+1}(a'_{t+1})^2=a'_{t+1}+A_{t}=A_{t+1}'.
\end{align*}
Multiply both sides by $\gamma_{t+1}^2/(2\lambda'_{t+1})$ and we get
\begin{align*}
    \frac{a_{t+1}^2}{2}=\frac{\gamma_{t+1}^2A_{t+1}'}{2\lambda'_{t+1}} = \frac{\min\inbraces{1,\frac{\lambda_{t+1}'}{\lambda_{t+1}}}\gamma_{t+1}A_{t+1}'}{2\lambda'_{t+1}} \le \frac{\gamma_{t+1}A_{t+1}'}{2\lambda_{t+1}}.
\end{align*}
We recycle an earlier computation and know that
\begin{align*}
    D_t-D_{t+1} &= a_{t+1}\ip{\mM^{-1}\nabla f(\xtilde_{t+1}), \vv_t-\xstar}_{\mM} - \frac{a_{t+1}^2}{2}\norm{\mM^{-1}\nabla f(\xtilde_{t+1})}_{\mM}^2 \\
    &\ge a_{t+1}\ip{\mM^{-1}\nabla f(\xtilde_{t+1}), \vv_t-\xstar}_{\mM}-\frac{\gamma_{t+1}A_{t+1}'}{2\lambda_{t+1}}\norm{\mM^{-1}\nabla f(\xtilde_{t+1})}_{\mM}^2,
\end{align*}
which completes the proof of the potential decrease.

The remaining statements follow as written in \cite[Proof of Proposition 1]{chjjs22}, and we conclude the proof of \Cref{lemma:ms_acceleration_potential}.
\end{proof}

Now that we have shown \Cref{lemma:ms_acceleration_potential}, we refer the reader to \cite[Appendix A]{chjjs22} for the proof of \Cref{thm:optimal_ms_acceleration}, as it now follows exactly as written there.

We also give additional bounds on the movement of the iterates in $\norm{\cdot}_{\mM}$, which is a straightforward adaptation of \cite[Lemma 31]{msbacon} to the improved framework from \cite{chjjs22}.

\begin{lemma}
\label{lemma:ms_acceleration_iterate_diameter}
For all $t\ge 1$, we have both
\begin{equation*}
    \begin{aligned}
        \norm{\vv_t-\xstar}_{\mM} &\le \sqrt{2}\norm{\vx_0-\xstar}_{\mM} \\
        \norm{\vx_t-\xstar}_{\mM} &\le \inparen{\sqrt{2}+\max_{1 \le i \le t} \frac{\lambda_i'}{\lambda_i} \cdot \sqrt{\frac{2}{1-\sigma^2}}}\norm{\vx_0-\xstar}_{\mM}
    \end{aligned}.
\end{equation*}
\end{lemma}

In the statement of \Cref{lemma:ms_acceleration_iterate_diameter}, the cost of overshooting the guess $\lambda_i'$ becomes evident -- without an additional strong convexity guarantee, it is challenging to ensure that each iterate remains in a small ball around $\xstar$. This is the main reason we are unable to apply the framework of \cite{chjjs22} to the $p=\infty$ case.

\begin{proof}[Proof of \Cref{lemma:ms_acceleration_iterate_diameter}]
Using the same notation as in \Cref{lemma:ms_acceleration_potential} and in that proof, we define
\begin{align*}
    P_t &\coloneqq A_tE_t+D_t \\
    \widehat{\lambda}_t &\coloneqq \min\inbraces{\lambda_t,\lambda_t'}.
\end{align*}
By induction on the conclusion of \Cref{lemma:ms_acceleration_potential}, for $t\ge 1$ we have
\begin{align*}
    \frac{1}{2}\norm{\vv_t-\xstar}_{\mM}^2 = D_t \le P_t + (1-\sigma^2)\sum_{k=1}^t A_k'\widehat{\lambda}_kN_k \le P_0 = \norm{\vx_0-\xstar}_{\mM}^2.
\end{align*}
Thus,
\begin{align*}
    \norm{\vv_t-\xstar}_{\mM} \le \sqrt{2}\norm{\vx_0-\xstar}_{\mM}.
\end{align*}
For the second conclusion, we introduce the following notation.
\begin{align*}
    \alpha_{t+1} &\coloneqq \frac{(1-\gamma_{t+1})A_t}{A_{t+1}} \\
    \beta_{t+1} &\coloneqq \frac{A_t}{A_{t+1}'} \\
    \delta_{t+1} &\coloneqq 1-(1-\alpha_{t+1})(1-\beta_{t+1}) = 1 - \frac{\gamma_{t+1}A'_{t+1}}{A_{t+1}} \cdot \frac{a'_{t+1}}{A'_{t+1}} = \frac{A_t}{A_{t+1}}
\end{align*}
We also establish for any $i$,
\begin{align*}
    \frac{\gamma_i A_i'}{\lambda_i a_i^2} = \frac{A_i'}{\lambda_i \gamma_i (a_i')^2} = \frac{1}{\gamma_i} \cdot \frac{\lambda_i'}{\lambda_i} = \max\inbraces{\frac{\lambda_i'}{\lambda_i}, 1},
\end{align*}
which implies
\begin{align*}
    \frac{\gamma_i A_i'}{\lambda_i} = a_i^2\max\inbraces{\frac{\lambda_i'}{\lambda_i}, 1}.
\end{align*}
Notice that
\begin{align*}
    \norm{\vx_{t+1}-\xstar}_{\mM} &\le \alpha_{t+1}\norm{\vx_t-\xstar}_{\mM} + (1-\alpha_{t+1})\norm{\xtilde_{t+1}-\xstar}_{\mM} \\
    &\le \alpha_{t+1}\norm{\vx_t-\xstar}_{\mM} + (1-\alpha_{t+1})\inparen{\norm{\vq_t-\xstar}_{\mM}+\norm{\xtilde_{t+1}-\vq_t}_{\mM}} \\
    &\le \alpha_{t+1}\norm{\vx_t-\xstar}_{\mM} \\
    &\quad + (1-\alpha_{t+1})\inparen{\beta_{t+1}\norm{\vx_t-\xstar}_{\mM}+(1-\beta_{t+1})\norm{\vv_t-\xstar}_{\mM}+\norm{\xtilde_{t+1}-\vq_t}_{\mM}} \\
    &= \inparen{\beta_{t+1}+\alpha_{t+1}-\alpha_{t+1}\beta_{t+1}}\norm{\vx_t-\xstar}_{\mM} \\
    &\quad + \inparen{1-\alpha_{t+1}}\inparen{1-\beta_{t+1}}\norm{\vv_t-\xstar}_{\mM} + \inparen{1-\alpha_{t+1}}\norm{\xtilde_{t+1}-\vq_t}_{\mM} \\
    &= \delta_{t+1}\norm{\vx_t-\xstar}_{\mM} + (1-\delta_{t+1})\norm{\vv_t-\xstar}_{\mM} + \inparen{1-\alpha_{t+1}}\norm{\xtilde_{t+1}-\vq_t}_{\mM} \\
    &\le \prod_{i=0}^t \delta_{i+1}\norm{\vx_0-\xstar}_{\mM} + \inparen{1-\prod_{i=0}^t \delta_{i+1}}\norm{\vv_t-\xstar}_{\mM} \\
    &\quad + \sum_{i=1}^{t+1} \prod_{j=i+1}^{t+1}\delta_{j}(1-\alpha_{i})\norm{\xtilde_{i}-\vq_{i-1}}_{\mM} \\
    &\le \sqrt{2}\norm{\vx_0-\xstar}_{\mM} + \sum_{i=1}^{t+1} \prod_{j=i+1}^{t+1}\delta_{j}(1-\alpha_{i})\norm{\xtilde_{i}-\vq_{i-1}}_{\mM} \\
    &= \sqrt{2}\norm{\vx_0-\xstar}_{\mM} + \sum_{i=1}^{t+1} \frac{A_i}{A_{t+1}}(1-\alpha_{i})\norm{\xtilde_{i}-\vq_{i-1}}_{\mM} \\
    &= \sqrt{2}\norm{\vx_0-\xstar}_{\mM} + \sum_{i=1}^{t+1} \frac{A_i}{A_{t+1}}\cdot\frac{\gamma_iA_i'}{A_i}\norm{\xtilde_{i}-\vq_{i-1}}_{\mM} \\
    &= \sqrt{2}\norm{\vx_0-\xstar}_{\mM} + \frac{1}{A_{t+1}}\sum_{i=1}^{t+1} \sqrt{\frac{\gamma_iA_i'}{\lambda_i}} \cdot \sqrt{\lambda_i\gamma_iA_i'}\norm{\xtilde_{i}-\vq_{i-1}}_{\mM} \\
    &\le \sqrt{2}\norm{\vx_0-\xstar}_{\mM} + \frac{\inparen{\sum_{i=1}^{t+1} \frac{\gamma_iA_i'}{\lambda_i}}^{1/2}}{A_{t+1}} \cdot \inparen{\sum_{i=1}^{t+1} \lambda_i\gamma_iA_i'\norm{\xtilde_{i}-\vq_{i-1}}_{\mM}^2}^{1/2} \\
    &\le \sqrt{2}\norm{\vx_0-\xstar}_{\mM} + \frac{\inparen{\sum_{i=1}^{t+1} \frac{\gamma_iA_i'}{\lambda_i}}^{1/2}}{A_{t+1}} \cdot \sqrt{\frac{2}{1-\sigma^2}}\norm{\vx_0-\xstar}_{\mM} \\
    &\le \sqrt{2}\norm{\vx_0-\xstar}_{\mM} + \frac{\sum_{i=1}^{t+1} a_i\max\inbraces{1,\frac{\lambda_i'}{\lambda_i}}}{A_{t+1}} \cdot \sqrt{\frac{2}{1-\sigma^2}}\norm{\vx_0-\xstar}_{\mM} \\
    &\le \sqrt{2}\norm{\vx_0-\xstar}_{\mM} + \max_{1 \le i \le t+1} \frac{\lambda_i'}{\lambda_i} \cdot \sqrt{\frac{2}{1-\sigma^2}}\norm{\vx_0-\xstar}_{\mM} \\
    &= \inparen{\sqrt{2}+\max_{1 \le i \le t+1} \frac{\lambda_i'}{\lambda_i} \cdot \sqrt{\frac{2}{1-\sigma^2}}}\norm{\vx_0-\xstar}_{\mM},
\end{align*}
completing the proof of \Cref{lemma:ms_acceleration_iterate_diameter}.
\end{proof}

%% file: interpolation.tex
\section{Interpolating Between Average and Robust Losses}
\label{sec:gp_interpolating_proof}

In this section, we prove \Cref{mainthm:gp_regression_iteration_complexity}. As before, our proof follows the outline in \Cref{sec:gp_reg_overview}. The main technical challenges are to establish a form of strong convexity for our objective $f$ and then to build a solver for the proximal problem \eqref{eq:gp_prox_intro_reg}. 

The rest of this section is organized as follows. In \Cref{sec:gp_objective_calculus}, we derive calculus facts about our objective $f$, including bounds on its Hessian and the promised strong convexity (particularly \Cref{lemma:strong_convexity_gp} and the more general result it builds on, \Cref{lemma:strong_convexity_component}). In \Cref{sec:gp_iterate_facts}, we prove some facts about the iterates of \Cref{alg:optimal_ms_acceleration} when applied to our setting. In \Cref{sec:gp_prox_solver}, we more precisely define and analyze our solver for proximal sub-problems. This section is fairly technical and we give a more detailed outline there. Finally, in \Cref{sec:gp_interpolating_alg}, we assemble all these components and analyze \Cref{alg:gp_regression_alg}, thereby proving \Cref{mainthm:gp_regression_iteration_complexity}.

Throughout this analysis, we rescale the problem so that $f(\xstar) = 1$. It is now sufficient to solve for an $\eps$-additive error solution.

\subsection{Calculus for the objective}
\label{sec:gp_objective_calculus}

In this section, we work out some calculus facts related to our objective $\gnorm{\mA\vx-\vb}{p}^p$. Throughout this discussion, let $f(\vx) \coloneqq \gnorm{\mA\vx-\vb}{p}^p$.

\begin{lemma}
\label{lemma:gp_regression_hessian_upperbound}
For any $\vz\in\R^d$, we have
\begin{align*}
    p\sum_{i=1}^m \norm{\mA_{S_i}\vx-\vb_{S_i}}_2^{p-2}\norm{\mA_{S_i}\vz}_2^2 \le \vz^{\top}\inparen{\nabla^2 f(\vx)}\vz \le p(p-1)\sum_{i=1}^m\norm{\mA_{S_i}\vx - \vb_{S_i}}_2^{p-2}\norm{\mA_{S_i}\vz}_2^2.
\end{align*}
\end{lemma}
\begin{proof}[Proof of \Cref{lemma:gp_regression_hessian_upperbound}]
Let us first calculate the derivative and hessian for $f(\cdot)$ using the chain rule and usual matrix differentiation rules:
\begin{align}
    f(\vx) &= \sum_{i=1}^m\norm{\mA_{S_i}\vx - \vb_{S_i}}_2^p\enspace,\nonumber\\
    \nabla f(\vx) &= p\sum_{i=1}^m \norm{\mA_{S_i}\vx - \vb_{S_i}}_2^{p-2}\mA_{S_i}^{\top}(\mA_{S_i}\vx - \vb_{S_i})\enspace,\label{eq:f_grad}\\
    \nabla^2 f(\vx) &= p\sum_{i=1}^m\norm{\mA_{S_i}\vx - \vb_{S_i}}_2^{p-2}\mA_{S_i}^{\top}\mA_{S_i} \nonumber\\
    &\quad + p(p-2)\sum_{i=1}^m\norm{\mA_{S_i}\vx - \vb_{S_i}}_2^{p-4}\left(\mA_{S_i}^{\top}(\mA_{S_i}\vx - \vb_{S_i})(\mA_{S_i}\vx - \vb_{S_i})^{\top}\mA_{S_i}\right)\enspace.\label{eq:f_hess}
\end{align}
Using this formula, we take the quadratic form with respect to a vector $\vz$. By Cauchy-Schwarz, notice that
\begin{align*}
    &\quad \vz^{\top}\norm{\mA_{S_i}\vx - \vb_{S_i}}_2^{p-4}\left(\mA_{S_i}^{\top}(\mA_{S_i}\vx - \vb_{S_i})(\mA_{S_i}\vx - \vb_{S_i})^{\top}\mA_{S_i}\right)\vz \\
    &= \norm{\mA_{S_i}\vx - \vb_{S_i}}_2^{p-4}\ip{\mA_{S_i}\vz, \mA_{S_i}\vx-\vb_{S_i}}^2 \le \norm{\mA_{S_i}\vx-\vb_{S_i}}_2^{p-2}\norm{\mA_{S_i}\vz}_2^2.
\end{align*}
With that, we have
\begin{align}
    \vz^{\top}\inparen{\nabla^2 f(\vx)}\vz &\le p\sum_{i=1}^m \norm{\mA_{S_i}\vx-\vb_{S_i}}^{p-2}\norm{\mA_{S_i}\vz}_2^2 + (p-2)\norm{\mA_{S_i}\vx-\vb_{S_i}}^{p-2}\norm{\mA_{S_i}\vz}_2^2 \enspace,\nonumber\\
    &= p(p-1)\sum_{i=1}^m\norm{\mA_{S_i}\vx - \vb_{S_i}}_2^{p-2}\norm{\mA_{S_i}\vz}_2^2\enspace.\label{eq:f_quad_form}
\end{align}
For the lower bound, we use our calculation for $\nabla^2 f(\vx)$ to write
\begin{align*}
    \vz^{\top}\inparen{\nabla^2 f(\vx)}\vz \ge p\sum_{i=1}^m \norm{\mA_{S_i}\vx-\vb_{S_i}}_2^{p-2}\norm{\mA_{S_i}\vz}_2^2,
\end{align*}
completing the proof of \Cref{lemma:gp_regression_hessian_upperbound}.
\end{proof}

\subsubsection{Strong Convexity of the Objective}

The main pair of results of this section are \Cref{lemma:strong_convexity_gp} and \Cref{lemma:strong_convexity_component}. We can think of \Cref{lemma:strong_convexity_gp} as a form of strong convexity for our objective.

\begin{lemma}[Strong convexity of $f$]
\label{lemma:strong_convexity_gp}
Let $f(\vx) \coloneqq \gnorm{\mA\vx-\vb}{p}^p$. For all $\vd\in\R^d$, we have
\begin{align*}
    f(\vx+\vd) \ge f(\vx) + \ip{\nabla f(\vx),\vd} + \frac{4}{2^p}\gnorm{\mA\vd}{p}^p,
\end{align*}
and therefore
\begin{align*}
    \norm{\vx-\xstar}_{\mM} \le 2^{3/2 - 3/p}d^{1/2-1/p}(f(\vx)-f(\xstar))^{1/p}\enspace.
\end{align*}
\end{lemma}

\lemmastrongconvexitycomponent

To motivate \Cref{lemma:strong_convexity_component}, let us see how \Cref{lemma:strong_convexity_component} implies \Cref{lemma:strong_convexity_gp}.

\begin{proof}[Proof of \Cref{lemma:strong_convexity_gp}]
Note that
\begin{align*}
    \nabla f(\vx) = \sum_{i=1}^m p\norm{\mA_{S_i}\vx-\vb_{S_i}}_2^{p-2}\mA_{S_i}^{\top}(\mA_{S_i}\vx-\vb_{S_i})\enspace.
\end{align*}
This implies
\begin{align*}
    \sum_{i=1}^m p\norm{\mA_{S_i}\vx-\vb_{S_i}}_2^{p-2}\ip{\mA_{S_i}\vx-\vb_{S_i},\mA_{S_i}\vd} = \ip{\nabla f(\vx),\vd}\enspace.
\end{align*}
Combining this and applying \Cref{lemma:strong_convexity_component} (which is a strong convexity lemma for $\|\cdot\|_2^p$ that we prove subsequently in this section), we get
\begin{align*}
    f(\vx+\vd) &= \gnorm{\mA(\vx+\vd)-\vb}{p}^p = \gnorm{\mA\vd+(\mA\vx-\vb)}{p}^p\enspace,\\
    &= \sum_{i=1}^m \norm{\mA_{S_i}\vd+(\mA_{S_i}\vx-\vb_{S_i})}_2^p\enspace,\\
    &\ge^{\text{(\Cref{lemma:strong_convexity_component})}} \sum_{i=1}^m \norm{\mA_{S_i}\vx-\vb_{S_i}}_2^p + p\|\mA_{S_i}\vx-\vb_{S_i}\|_2^{p-2}\ip{(\mA_{S_i}\vx-\vb_{S_i}),\mA_{S_i}\vd}+ \frac{4}{2^p}\norm{\mA_{S_i}\vd}_2^p\enspace,\\
    &= \sum_{i=1}^m \norm{\mA_{S_i}\vx-\vb_{S_i}}_2^p + \ip{p\|\mA_{S_i}\vx-\vb_{S_i}\|_2^{p-2}\mA_{S_i}^{\top}(\mA_{S_i}\vx-\vb_{S_i}),\vd}+ \frac{4}{2^p}\norm{\mA_{S_i}\vd}_2^p\enspace,\\
    &=^{\text{\eqref{eq:f_grad}}} \gnorm{\mA\vx-\vb}{p}^p + \ip{\nabla f(\vx),\vd} + \frac{4}{2^p}\gnorm{\mA\vd}{p}^p = f(\vx) + \ip{\nabla f(\vx),\vd} + \frac{4}{2^p}\gnorm{\mA\vd}{p}^p\enspace.
\end{align*}
We now take care of the second statement. Observe that at optimality, we have $\nabla f(\xstar) = 0$. Plugging this in (replace $\vx$ by $\vx^\star$ and $\vd$ by $\vx-\vx^\star$ above), rearranging, and taking $p$th roots gives
\begin{align*}
    \gnorm{\mA(\vx-\xstar)}{p} \le \left(\frac{4}{2^p}\right)^{-1/p}(f(\vx)-f(\xstar))^{1/p} = \frac{2}{4^{1/p}}(f(\vx)-f(\xstar))^{1/p}\enspace. 
\end{align*}
Next, recall that by \Cref{thm:gp_regression_blw_intro},
\begin{align*}
    \norm{\vx-\xstar}_{\mM} = \norm{\mW^{1/2-1/p}\mA(\vx-\xstar)}_2 \le (2d)^{1/2-1/p}\gnorm{\mA(\vx-\xstar)}{p}\enspace.
\end{align*}
Stitching the inequalities together completes the proof of \Cref{lemma:strong_convexity_gp}.
\end{proof}

In the rest of this subsection, we prove \Cref{lemma:strong_convexity_component}. We begin with a few numerical inequalities. 

\begin{lemma}
\label{lemma:gp_sc_galpha}
For $\alpha \le -1/2$ and $p\ge 2$, $g(\alpha) := \frac{1+p\alpha}{(-(2\alpha+1))^{p/2}}$ is nonincreasing in $\alpha$.
\end{lemma}
\begin{proof}[Proof of \Cref{lemma:gp_sc_galpha}]
We first take the derivative of $g$ with respect to $\alpha$,
\begin{align*}
    g'(\alpha) &= \frac{p(-(2\alpha+1))^{p/2} - \inparen{(-2)\frac{p}{2}\inparen{-(2\alpha+1)}^{p/2-1}}(1+p\alpha)}{(-(2\alpha+1))^p}\enspace,\\
    &= \frac{p(-(2\alpha+1)^{p/2}) + p\inparen{-(2\alpha+1)}^{p/2-1}(1+p\alpha)}{(-(2\alpha+1))^p}\enspace,\\
    &= p\cdot\frac{(-(2\alpha+1)) + (1+p\alpha)}{(-(2\alpha+1))^{p/2+1}}\enspace,\\
    &= p\cdot\frac{(p-2)\alpha}{(-(2\alpha+1))^{p/2+1}} \le 0\enspace,
\end{align*}
where in the final inequality we used that $p\geq 2$ and $\alpha\leq -1/2$. This completes the proof of the lemma.
\end{proof}
We also need the following lemma, which is similar to a result due to \citeauthor{akps19} \cite[Lemma 4.5]{akps19}. It amounts to proving \Cref{lemma:strong_convexity_component} when the dimension $k=1$.
\begin{lemma}[Case A. of \Cref{lemma:gp_sc_alphaall}]\label{lemma:adil_sc_lemma}
For any $\alpha\in\R$ and $p \ge 2$,
\begin{align*}
    \abs{1+\alpha}^p \ge 1+p\alpha+\frac{4}{2^p}\abs{\alpha}^p\enspace.
\end{align*}
\end{lemma}
\begin{proof}[Proof of \Cref{lemma:adil_sc_lemma}]
Note that the inequality is true when $p=2$ and becomes an equality. We consider the case when $p>2$ and use $h(\alpha)$ to denote the error function,
\begin{align*}
    h(\alpha) \coloneqq \abs{1+\alpha}^p - \inparen{1+p\alpha+\frac{4}{2^p}\abs{\alpha}^p}\enspace.
\end{align*}
We aim to show $h(\alpha) \ge 0$ for all $\alpha \in \R$. Let us first write the derivatives of $h$.
\begin{align*}
    h'(\alpha) &= p\inparen{\abs{1+\alpha}^{p-2}(1+\alpha) - \inparen{1 + \frac{4}{2^p}\abs{\alpha}^{p-2}\alpha}}\enspace,\\
    h''(\alpha) &= p(p-1)\inparen{\abs{1+\alpha}^{p-2}-\frac{4}{2^p}\abs{\alpha}^{p-2}} = p(p-1)\inparen{\abs{1+\alpha}^{p-2}-\abs{\frac{\alpha}{2}}^{p-2}}\enspace.
\end{align*}
It is now easy to verify the following statements about $h$,
\begin{itemize}
    \item[I.] $h'(-2) = h''(-2) = 0$ and $h''(\alpha) > 0$ for $\alpha < -2$, $\Rightarrow$ within the range $(-\infty,-2]$ the function $h$ is minimized at $-2$;
    \item[II.] $h'(-2)=0$ and $h''(\alpha) \leq 0$ for $\alpha\in (-2, -2/3]$ $\Rightarrow$ $h'(\alpha) < 0$ in the range $(-2, -2/3]$, i.e., in that range the function $h$ is minimized at $-2/3$;
    \item[III.] $h'(-2/3) < 0 = h'(0)$ and $h''(\alpha) > 0$ for $\alpha > -2/3$ $\Rightarrow$ the function $h$ is decreasing in $(-2/3,0)$ and increasing in $[0, \infty)$, i.e., within the range $(-2/3, \infty)$ the function $h$ is minimized at $0$.
\end{itemize}
As a result of the above observations, it is enough to check the inequality at the inputs $\alpha \in \inbraces{-2, -2/3, 0}$. We have for $p>2$,
\begin{align*}
    h(-2) &= 1 - \inparen{1-2p+4} = 2p -4 > 0\enspace,\\
    h\inparen{-\frac{2}{3}} &= \frac{1}{3^p} - \inparen{1-\frac{2p}{3}+\frac{4}{2^p}\abs{\frac{2}{3}}^p} = \frac{1}{3^p}-1+\frac{2p}{3}-\frac{4}{3^p} = -1 + \frac{2p}{3}-\frac{3}{3^p} > 0\\
    h(0) &= 1 - 1 = 0\enspace.
\end{align*}
This implies that $h(\alpha)\geq 0$ for all values of $\alpha$, concluding the proof of \Cref{lemma:adil_sc_lemma}.
\end{proof}

Next, we prove a special case of \Cref{lemma:strong_convexity_component}.

\begin{lemma}\label{lemma:gp_sc_alphaall}
    For any $\alpha\in \R$, $\beta\ge 0$, and $p \ge 2$, we have
\begin{align*}
    \inparen{(1+\alpha)^2+\beta^2}^{p/2} \ge 1+p\alpha + \frac{4}{2^p}\inparen{\alpha^2+\beta^2}^{p/2}\enspace.
\end{align*}
\end{lemma}
\begin{proof}[Proof of \Cref{lemma:gp_sc_alphaall}]
    Let us study the difference of both sides of the inequality using the following function,
    \begin{align*}
        h(\alpha, \beta) \coloneqq \inparen{(1+\alpha)^2+\beta^2}^{p/2} - \left(1+p\alpha + \frac{4}{2^p}\inparen{\alpha^2+\beta^2}^{p/2}\right)\enspace.
    \end{align*}
    We want to show that for $\alpha\in \R$, $\beta\ge 0$, and $p \ge 2$, $h(\alpha, \beta) \geq 0$. We will break this proof into three cases: \textbf{A.} $\alpha \in \R$ and $\beta =0$; \textbf{B.} $\alpha \in (-\infty, -2]\cup[-2/3,\infty)$ and $\beta >0$; and \textbf{C.} $\alpha \in (-2, -2/3)$ and $\beta>0$. These cases together cover of the entire range of $\alpha\in \R$ and $\beta\geq 0$.
    
    \paragraph{\underline{Case A.}} When $\beta=0$, the proof simply follows from the statement of \Cref{lemma:adil_sc_lemma} by noting $|\alpha|^{p} = (\sqrt{\alpha^2})^{p} = (\alpha^2)^{p/2}$. 
    
    In the remaining two cases we will show that for any $\alpha\in \R$, increasing the value of $\beta$ still maintains $h(\alpha, \beta) \geq 0$. To see this, we first note that the derivative of $h(\alpha, \beta)$ w.r.t. $\beta$ is given by,
    \begin{align*}
        \nabla_\beta h(\alpha, \beta) &= p\beta\left(\inparen{(1+\alpha)^2+\beta^2}^{p/2-1} - \frac{4}{2^p}\inparen{\alpha^2+\beta^2}^{p/2-1}\right)\enspace.
    \end{align*}
    For $\beta > 0$, ensuring this derivative is positive is equivalent to the following,
    \begin{align}
        \nabla_\beta h(\alpha, \beta) > 0 &\equiv p\beta\inparen{(1+\alpha)^2+\beta^2}^{p/2-1} > p\beta \cdot \frac{4}{2^p}\inparen{\alpha^2+\beta^2}^{p/2-1}\enspace,\nonumber\\
        &\equiv^{(p\beta > 0)} (1+\alpha)^2 + \beta^2  > \left(\frac{1}{2^{p-2}}\right)^{2/(p-2)}\cdot \left(\alpha^2 + \beta^2\right)\enspace,\nonumber\\
        &\equiv (1+\alpha)^2 + \beta^2  > \frac{1}{4}\cdot \left(\alpha^2 + \beta^2\right)\enspace,\nonumber\\
        &\equiv (3\alpha^2 + 8\alpha + 4) + 3\beta^2  > 0\enspace,\nonumber\\
        &\equiv \beta^2 > -\left(\alpha^2 + \frac{8}{3}\alpha + \frac{4}{3}\right)\enspace.\label{eq:grad_pos_equiv}
    \end{align}
    
    \paragraph{\underline{Case B.}} Note that the roots of the quadratic function $3\alpha^2 + 8\alpha + 4$ are given by $\alpha_1 = -2$ and $\alpha_2 = -2/3$. This means that for $\alpha \in (-\infty, -2]\cup[-2/3,\infty)$ we have $3\alpha^2 + 8\alpha + 4 \geq 0$ which is \textbf{sufficient} to ensure using \eqref{eq:grad_pos_equiv} that $\nabla_\beta h(\alpha, \beta) > 0$, and hence $h(\alpha, \beta) > 0$. This takes care of Case B. 
    
    \paragraph{\underline{Case C.}} Now we only need to consider the range $\alpha \in (-2,-2/3)$ with $\beta>0$. In this range, the recall the equivalence \eqref{eq:grad_pos_equiv}, 
    \begin{align*}
        \nabla_\beta h(\alpha, \beta) > 0 &\equiv \beta > \sqrt{-\left(\alpha^2 + \frac{8}{3}\alpha + \frac{4}{3}\right)} =: \beta_0(\alpha)\enspace.
    \end{align*}
    Thus for all $\beta > \beta_0(\alpha)$ we know that $h(\alpha, \beta)$ is increasing in $\beta$ and vice-versa. This allows us for any given $\alpha\in (-2,-2/3)$ to further break Case C into two sub-cases:
    
    \paragraph{\underline{Case C.I}} For $\beta \in [0, \beta_0)$, since $h(\alpha, \beta)$ is decreasing in $\beta$ its lowest value is attained at $\beta=0$ and we only need to verify that $h(\alpha, 0) \geq 0$. We get this directly from \Cref{lemma:adil_sc_lemma}.
    
    \paragraph{\underline{Case C.II}} For $\beta \in [\beta_0, \infty)$, since $h(\alpha, \beta)$ is increasing in $\beta$ its lowest value is attained at $\beta=\beta_0$ and we only need to verify that $h(\alpha, \beta_0(\alpha)) \geq 0$. We first simplify the expression for $h(\alpha, \beta_0(\alpha))$,
    \begin{align*}
        h(\alpha, \beta_0(\alpha)) &=  \inparen{(1+\alpha)^2+\beta_0^2}^{p/2} - \left(1+p\alpha + K_p\inparen{\alpha^2+\beta_0^2}^{p/2}\right)\enspace,\\
        &= \inparen{-\frac{1}{3}-\frac{2}{3}\alpha}^{p/2} - \left(1+p\alpha + \frac{4}{2^p}\inparen{-\frac{8}{3}\alpha -\frac{4}{3}}^{p/2}\right)\enspace,\\
        &= \inparen{-\frac{1}{3}-\frac{2}{3}\alpha}^{p/2} - \left(1+p\alpha + 4\inparen{-\frac{2}{3}\alpha -\frac{1}{3}}^{p/2}\right)\enspace,\\
        &= -1 - p\alpha - 3\inparen{-\frac{2}{3}\alpha -\frac{1}{3}}^{p/2}\enspace,\\
        &= -1 - p\alpha - \frac{1}{3^{p/2-1}}(-2\alpha -1)^{p/2}\enspace,\\
        &= -(-2\alpha-1)^{p/2}\left(\frac{1+p\alpha}{(-2\alpha-1)^{p/2}} + \frac{1}{3^{p/2-1}}\right)\enspace. 
    \end{align*}
    Now since $\alpha\in (-2,-2/3)<-1/2$ we can use \Cref{lemma:gp_sc_galpha} to note that the first term is non-decreasing in $\alpha$ which means that its lowest value in this range can be lower bounded by its value at $\alpha = -2$, i.e., for $\alpha\in (-2,-2/3)$,
    \begin{align*}
        h(\alpha, \beta_0(\alpha)) &\geq h(-2, \beta_0(-2))\enspace,\\ 
        &= -3^{p/2}\left(\frac{1-2p}{3^{p/2}} + \frac{1}{3^{p/2-1}}\right)\enspace,\\
        &= 2p-1-3 = 2(p-2) > 0\enspace,
    \end{align*}
    which finishes the proof of Case C.II and also Case C. Together Cases A, B and C complete the proof of \Cref{lemma:gp_sc_alphaall}.  
\end{proof}

We are now ready to prove \Cref{lemma:strong_convexity_component}.

\begin{proof}[Proof of \Cref{lemma:strong_convexity_component}]
First, assume that $\norm{\vv}_2=1$. We will later extend the result to all $\vv$.

Since $\norm{\vv}_2=1$, we can write $\triangle = \alpha\vv+\beta\vw$ where $\ip{\vv,\vw} = 0$ and $\norm{\vw}_2=1$, so that we have $\norm{\triangle}_2^2=\alpha^2+\beta^2$. Without loss of generality, we have $\beta \ge 0$. Fixing $\vw$ and $\alpha$ for now, it is enough to show that for all $\beta \ge 0$, we have
\begin{align*}
    \norm{(1+\alpha)\vv+\beta\vw}_2^p = \inparen{(1+\alpha)^2+\beta^2}^{p/2} \overset{?}{\ge} 1+p\alpha + \frac{4}{2^p}\norm{\triangle}_2^p = 1+p\alpha + \frac{4}{2^p}\inparen{\alpha^2+\beta^2}^{p/2}.
\end{align*}
This follows immediately by \Cref{lemma:gp_sc_alphaall}.

We now extend the result for all $\vv$. Let $\bar{\vv} \coloneqq \vv/\norm{\vv}_2$ and note that
\begin{align*}
    \norm{\vv+\triangle}_2^p = \norm{\vv}_2^p\norm{\bar{\vv}+\frac{\triangle}{\norm{\vv}_2}}_2^p &\ge \norm{\vv}_2^p\inparen{1+\ip{\bar{\vv},\frac{\triangle}{\norm{\vv}_2}}+\frac{4}{2^p}\norm{\frac{\triangle}{\norm{\vv}_2}}_2^p} \\
    &= \norm{\vv}_2^p+p\norm{\vv}_2^{p-2}\ip{\vv,\triangle}+\frac{4}{2^p}\norm{\triangle}_2^p,
\end{align*}
completing the proof of \Cref{lemma:strong_convexity_component}.
\end{proof}

\subsubsection{Smoothness of the Objective}

The main result of this subsection is \Cref{lemma:gp_smooth}.

\begin{lemma}
\label{lemma:gp_smooth}
For all $\vx\in\R^d$, we have
\begin{align*}
    f(\vx)-f(\xstar) \le \frac{p(p-1)}{2}f(\vx)^{1-\frac{2}{p}}\gnorm{\mA(\vx-\xstar)}{p}^2.
\end{align*}
\end{lemma}
\begin{proof}[Proof of \Cref{lemma:gp_smooth}]
By Taylor's/mean-value theorem, we can write for some $\vy$ on the line connecting $\xstar$ and $\vx$,
\begin{align*}
    f(\vx) &= f(\xstar) + \ip{\nabla f(\xstar), \vx-\xstar} + \frac{1}{2}(\vx-\xstar)^{\top}\nabla^2 f(\vy)(\vx-\xstar) \\
    &\le^{\eqref{eq:f_quad_form}} f(\xstar) + \frac{p(p-1)}{2}\sum_{i=1}^m \norm{\mA_{S_i}\vy-\vb_{S_i}}_2^{p-2}\norm{\mA_{S_i}(\vx-\xstar)}_2^2 \\
    &\le f(\xstar) + \frac{p(p-1)}{2} \inparen{\sum_{i=1}^m \norm{\mA_{S_i}\vy-\vb_{S_i}}_2^{p}}^{\frac{p-2}{p}}\inparen{\sum_{i=1}^m \norm{\mA_{S_i}(\vx-\xstar)}_2^p}^{\frac{2}{p}} \\
    &\le f(\xstar) + \frac{p(p-1)}{2}f(\vx)^{1-\frac{2}{p}}\gnorm{\mA(\vx-\xstar)}{p}^2,
\end{align*}
completing the proof of \Cref{lemma:gp_smooth}.
\end{proof}

\subsection{Facts about the Iterates}
\label{sec:gp_iterate_facts}

The main result of this section is \Cref{lemma:fq_bounded}. In words, \Cref{lemma:fq_bounded} tells us that each proximal query we make in \Cref{alg:optimal_ms_acceleration} (see Line \ref{line:optimal_ms_acceleration_query} of \Cref{alg:optimal_ms_acceleration}) has bounded objective value. We will need this later when we argue about the convergence rates for the algorithms used to solve the proximal subproblems.

\begin{lemma}
\label{lemma:fq_bounded}
For all queries $\vq_t$, we have
\begin{align*}
    f(\vq_t) \le f(\vx_t) + \inparen{9p(p-1)}^{\frac{p}{2}}d^{\frac{p}{2}-1}.
\end{align*}
\end{lemma}
\begin{proof}[Proof of \Cref{lemma:fq_bounded}]
We establish the following upper bound on $f(\vv_t) - f(\xstar)$ using the ingredients developed so far:
\begin{align*}
    f(\vv_t) - f(\xstar) &\le \frac{p(p-1)}{2}f(\vv_t)^{1-\frac{2}{p}}\gnorm{\mA(\vv_t-\xstar)}{p}^2 & \text{(\Cref{lemma:gp_smooth})} \\
    &\le \frac{p(p-1)}{2}f(\vv_t)^{1-\frac{2}{p}}\norm{\vv_t-\xstar}_{\mM}^2 & \text{(\Cref{thm:gp_regression_blw_intro})} \\
    &\le p(p-1)f(\vv_t)^{1-\frac{2}{p}}\norm{\vx_0-\xstar}_{\mM}^2 & \text{(\Cref{lemma:ms_acceleration_iterate_diameter})}\\
    &\le p(p-1)f(\vv_t)^{1-\frac{2}{p}}2^2(2d)^{1-\frac{2}{p}} & \text{(\Cref{lemma:gp_regression_initialization})} \\
    &\leq 8d^{1-\frac{2}{p}}p(p-1)f(\vv_t)^{1-\frac{2}{p}}\enspace.
\end{align*}
Now, recall that we assume by rescaling that $f(\xstar)=1$. From this, it trivially follows that $1 \le d^{1-\frac{2}{p}}p(p-1)f(\vv_t)^{1-\frac{2}{p}}$. Combining these and re-arranging the above inequality leads to the following  polynomial inequality in $f(\vv_t)$,
\begin{align}
    0 &\ge f(\vv_t)-8d^{1-\frac{2}{p}}p(p-1)f(\vv_t)^{1-\frac{2}{p}} -1\enspace,\nonumber\\ 
    &= f(\vv_t)-9d^{1-\frac{2}{p}}p(p-1)f(\vv_t)^{1-\frac{2}{p}} +  d^{1-\frac{2}{p}}p(p-1)f(\vv_t)^{1-\frac{2}{p}} -1\enspace,\nonumber\\
    &\ge f(\vv_t)-9d^{1-\frac{2}{p}}p(p-1)f(\vv_t)^{1-\frac{2}{p}}\enspace,\label{eq:gp_fq_upper}
\end{align}
where in the last inequality we used the fact that the optimal value $f(\vx^\star)=1$ (due to our rescaling), which implies that for $p\geq 2$,
$$1 \leq f(\vv_t) \leq d^{1-\frac{2}{p}}p(p-1)f(\vv_t)^{1-\frac{2}{p}}\enspace.$$
Solving for $f(\vv_t)$ in \eqref{eq:gp_fq_upper}, we get
\begin{align*}
    f(\vv_t) \le \inparen{9p(p-1)}^{\frac{p}{2}}d^{\frac{p}{2}-1}\enspace.
\end{align*}
Using the definition of $\vq_t$ from \Cref{alg:optimal_ms_acceleration} (Line \ref{line:optimal_ms_query}) along with the convexity of $f$ (Jensen's inequality), and using our bound on $f(\vv_t)$ we note that,
\begin{align*}
    f(\vq_t) &\le f(\vx_t) + f(\vv_t)\enspace,\\
    &\leq f(\vx_t) + \inparen{9p(p-1)}^{\frac{p}{2}}d^{\frac{p}{2}-1}\enspace,
\end{align*}
which completes the proof of \Cref{lemma:fq_bounded}.
\end{proof}

\subsection{Proximal Subproblems -- Calculus, Algorithms, Proofs}
\label{sec:gp_prox_solver}

Let
\begin{align*}
    f_{\vq_t}(\xtilde) \coloneqq f(\xtilde)+ep^p\norm{\xtilde-\vq_t}_{\mM}^p\enspace.
\end{align*}
In this subsection, we design and analyze an algorithm (\Cref{alg:gp_prox_solver}) that approximately solves the subproblem
\begin{align*}
    \argmin{\xtilde\in\R^d} f_{\vq_t}(\xtilde).
\end{align*} 
Specifically, we will output $(\xtilde_{t+1},\lambda_{t+1})$ that satisfy the $\frac{1}{2}$-MS oracle condition (\Cref{defn:ms_oracle}) and an appropriate movement bound (\Cref{defn:movement_bound}).

This subproblem is the workhorse of \Cref{alg:gp_regression_alg}, and once we implement and analyze the solver, it is very straightforward to plug this into \Cref{alg:optimal_ms_acceleration} and \Cref{thm:optimal_ms_acceleration} to get our final iteration complexity.

\begin{algorithm}[H]
\caption{\textsf{GpRegressionProxOracle}: Implements $\frac{1}{2}$-MS oracle for $\gnorm{\cdot}{p}$ regression (see \Cref{lemma:prox_subproblem_ms} and \Cref{alg:mirror_descent}.}
\label{alg:gp_prox_solver}
\begin{algorithmic}[1]
\Require Query $\vq_t$, previous iterate $\vx_t$, intended parameter distance $\gamma$.
\State Define \begin{equation*}
    \begin{aligned}
    f_{\vq_t}(\xtilde) &\coloneqq f(\xtilde) + ep^p\norm{\xtilde-\vq_t}_{\mM}^p \\
    h_{\vq_t}(\xtilde) &\coloneqq \norm{\xtilde-\vq_t}_{\nabla^2 f(\vq_t)}^2 + ep^p\norm{\xtilde-\vq_t}_{\mM}^p \\
    D_{h_{\vq_t}}(\vx,\vy) &\coloneqq h_{\vq_t}(\vx)-h_{\vq_t}(\vy) - \ip{\nabla h_{\vq_t}(\vy),\vx-\vy} \\
    \xtilde_{\vq_t} &\coloneqq \argmin{\xtilde\in\R^d} f_{\vq_t}(\xtilde)
    \end{aligned}.
\end{equation*}
\State Let $T \ge Cp^{O(1)}e\logv{dpeh_{\vq_t}(\xtilde_{\vq_t})\inparen{\frac{4}{p\gamma}}^p}$.
\State Run \Cref{alg:mirror_descent} with input iteration count $T$, base function $f_{\vq_t}$, reference function $h_{\vq_t}$, and initialization $\vq_t$.
\end{algorithmic}
\end{algorithm}

The goal of the rest of this section is to analyze \Cref{alg:gp_prox_solver}. The analysis follows several steps:
\begin{enumerate}
    \item We find a reference function $h_{\vq_t}$ that depends on the query point $\vq_t$ for which the proximal objective $f_{\vq_t}$ is relatively smooth and relatively strongly convex with $O(p^{O(1)})$ condition number (see \Cref{sec:mirror_descent} for a sense of why this is useful). The main result here is \Cref{lemma:gp_hessian_stable}.
    \item We show that $f_{\vq_t}$ is strongly convex, following from \Cref{lemma:strong_convexity_component}. This will help us understand the argument suboptimality for any point that approximately optimizes $f_{\vq_t}$ in function value. We also show that the reference function $h_{\vq_t}$ is strongly convex, using the same tools, for the same reason.
    \item We show a form of smoothness for $f_{\vq_t}$. This helps us bound the gradient of any point that approximately optimizes $f_{\vq_t}$. Combining these later will tell us that an approximate solution to $f_{\vq_t}$ in argument value is also an approximate stationary point, i.e., it satisfies the $\frac{1}{2}$-MS condition (\Cref{defn:ms_oracle}).
    \item We solve the proximal subproblems. This solution itself follows a few steps:\begin{enumerate}
        \item We apply \Cref{thm:gp_mirror_descent}. This tells us that as long as we can approximately solve the Bregman proximal problems (approximately implementing Line \ref{line:mirror_descent_exact} in \Cref{alg:mirror_descent}), we will be in good shape.
        \item This means we have to figure out how to approximately solve problems of the form $\argmin{\vx\in\R^d} \ip{\vg,\vx} + Lh_{\vq_t}(\vx)$, where $L$ is the smoothness constant derived for $f_{\vq_t}$ with respect to $h_{\vq_t}$. We do this up to an accuracy that approximate mirror descent can handle (see \Cref{thm:gp_mirror_descent} for details on what we want this approximation to look like). For the approximation to work, we need to approximately solve this problem up to both argument accuracy and approximate stationarity. The main technical result of interest here is \Cref{lemma:gp_prox_relativesmoothsolve}.
    \end{enumerate}
    \item We use the smoothness and strong convexity guarantees to show that our solution from the previous step satisfies the $\frac{1}{2}$-MS oracle (\Cref{defn:ms_oracle}), which means we can plug-and-play into \Cref{thm:optimal_ms_acceleration}.
\end{enumerate}

\subsubsection{Hessian Stability}

Throughout this section, we adopt the following notation:
\begin{align*}
    C_p &\coloneqq ep^{p} \\
    f(\vx) &\coloneqq \sum_{i=1}^m \norm{\mA_{S_i}\vx-\vb_{S_i}}_2^p \\
    f_{\vq}(\vx) &\coloneqq f(\vx) + C_p\norm{\vx-\vq}_{\mM}^p \\
    h_{\vq}(\vx) &\coloneqq \norm{\vx-\vq}_{\nabla^2 f(\vq)}^2 + C_p\norm{\vx-\vq}_{\mM}^p
\end{align*}
We begin with proving our Hessian stability fact, which should also be equivalently viewed as showing that $f_{\vq_t}$ is relatively smooth and relatively strongly convex in $h_{\vq_t}$ with $O(p^{O(1)})$ condition number. Our main result is \Cref{lemma:gp_hessian_stable} which relies on analytical results \Cref{lem:small_lemma_for_hessian_stability_one} and \Cref{lem:small_lemma_for_hessian_stability_two} that we prove later.
\begin{lemma}
\label{lemma:gp_hessian_stable}
    For all $\vx \in\R^d$ and $p\geq 2$, we have $$\frac{1}{2p\cdot e}\nabla^2h_{\vq}(\vx)\preceq \nabla^2f_{\vq}(\vx) \preceq p\cdot e\nabla^2 h_{\vq}(\vx)\enspace.$$
\end{lemma}
\begin{proof}[Proof of \Cref{lemma:gp_hessian_stable}]
    Using an arbitrary $\vz\in\mathbb{R}^d$ we can write the following quadratic form of the hessian of $f$,
    \begin{align}
        \vz^{\top}\nabla^2 f(\vx)\vz &\le^{\text{(a)}} p\cdot(p-1)\sum_{i=1}^m\norm{\mA_{S_i}\vx - \vb_{S_i}}_2^{p-2}\norm{\mA_{S_i}\vz}_2^2\enspace,\nonumber\\
        &= p\cdot(p-1)\sum_{i=1}^m\norm{\mA_{S_i}(\vx - \vq) + \mA_{S_i}\vq - \vb_{S_i}}_2^{p-2}\norm{\mA_{S_i}\vz}_2^2\enspace,\nonumber\\
        &\leq^{\text{(b)}}  p\cdot(p-1)\sum_{i=1}^m\left(\alpha_p^{p-2}\norm{\mA_{S_i}(\vx - \vq)}_2^{p-2}\norm{\mA_{S_i}\vz}_2^2 + \beta_p^{p-2}\norm{\mA_{S_i}\vq - \vb_{S_i}}_2^{p-2}\norm{\mA_{S_i}\vz}_2^2\right)\enspace,\nonumber\\
        &\leq^{\text{(c)}}p\cdot(p-1)\cdot\alpha_p^{p-2}\sum_{i=1}^m\norm{\mA_{S_i}(\vx - \vq)}_2^{p-2}\norm{\mA_{S_i}\vz}_2^2 + (p-1)\cdot\beta_p^{p-2}\vz^{\top}\nabla^2f(\vq)\vz\enspace,\nonumber\\
        &\leq^{\text{(d)}}p\cdot(p-1)\cdot\alpha_p^{p-2}\left(\norm{\vx-\vq}^p_{\mM}\right)^{(p-2)/p}\left(\norm{\vz}^p_{\mM}\right)^{2/p}+ (p-1)\cdot\beta_p^{p-2}\vz^{\top}\nabla^2f(\vq)\vz\enspace,\nonumber\\
        &= p\cdot(p-1)\cdot\alpha_p^{p-2}\norm{\vx-\vq}^{p-2}_{\mM}\norm{\vz}^2_{\mM} + (p-1)\cdot\beta_p^{p-2}\vz^{\top}\nabla^2f(\vq)\vz\enspace,\nonumber\\
        &\leq^{(e)} \frac{(p-1)\cdot \alpha_p^{p-2}}{C_p} \vz^{\top}\nabla^2g_{\vq}(\vx)\vz + (p-1)\cdot\beta_p^{p-2}\vz^{\top}\nabla^2f(\vq)\vz\enspace,\label{eq:quad_form_f_ub}
    \end{align}
    where in (a) we apply the upper bound from \Cref{lemma:gp_regression_hessian_upperbound}, in (b) we pick $\alpha_p,\beta_p\geq 1$ such that $1/\alpha_p + 1/\beta_p = 1$ (we will choose them later), in (c) we apply the lower bound from \Cref{lemma:gp_regression_hessian_upperbound}, in (d) we use the choice of our weights in designing $\mM$ and \Cref{thm:gp_regression_blw_intro}\todo{write something more explicit for justifying (d).} and finally in (e) we use the following calculations for the regularizer term for some $\vz\in\mathbb{R}^d$,
    \begin{align*}
        g_{\vq}(\vx) &\coloneqq C_p\norm{\vx-\vq}^p_{\mM}\enspace,\\    
        \nabla g_{\vq}(\vx) &= pC_p\norm{\vx-\vq}_{\mM}^{p-2}{\mM}(\vx-\vq)\enspace,\\
        \nabla^2 g_{\vq}(\vx) &= pC_p\norm{\vx-\vq}_{\mM}^{p-2}\mM + p(p-2)C_p\norm{\vx-\vq}_{\mM}^{p-4}\mM(\vx-\vq)(\vx-\vq)^{\top}\mM\enspace,\\
        \vz^{\top}\nabla^2 g_{\vq}(\vx)\vz &= pC_p\norm{\vx-\vq}_{\mM}^{p-2}\norm{\vz}^2_{\mM}+ p(p-2)C_p\norm{\vx-\vq}_{\mM}^{p-4}\left((\vx-\vq)^{\top}\mM\vz\right)^2 \geq^{(p\geq 2)} 0\enspace.
    \end{align*}
    Combining \eqref{eq:quad_form_f_ub} with the definition of $f_{\vq}$ gives us,
    \begin{align*}
        \vz^{\top}\nabla^2f_{\vq}(\vx)\vz &= \vz^{\top}\nabla^2f(\vx)\vz + \vz^{\top}\nabla^2g_{\vq}(\vx)\vz\enspace,\\
        &\leq^{\text{using }\eqref{eq:quad_form_f_ub}} (p-1)\cdot\beta_p^{p-2}\vz^{\top}\nabla^2f(\vq)\vz + \left(1+\frac{(p-1)\cdot \alpha_p^{p-2}}{C_p}\right) \vz^{\top}\nabla^2g_{\vq}(\vx)\vz\enspace.
    \end{align*}
    Thus, in order to finish the proof for the upper bound we need to pick $\alpha_p, \beta_p$. We split the analysis here into two cases: \textbf{A.} $p>2$ and \textbf{B.} $p=2$.
    
    \paragraph{\underline{Case A.} ($p>2$)} For simplicity we will just pick $\alpha_p = p-1$ and $\beta_p = \frac{p-1}{p-2}$ which implies,
    \begin{align*}
        \vz^{\top}\nabla^2f_{\vq}(\vx)\vz &\leq (p-1)\cdot\left(1 + \frac{1}{p-2}\right)^{p-2}\vz^{\top}\nabla^2f(\vq)\vz + \left(1+\frac{(p-1)\cdot (p-1)^{p-2}}{C_p}\right) \vz^{\top}\nabla^2g_{\vq}(\vx)\vz\enspace,\\
        &\leq (p-1)\cdot e\vz^{\top}\nabla^2f(\vq)\vz + \left(1+\frac{(p-1)^{p-1}}{C_p}\right) \vz^{\top}\nabla^2g_{\vq}(\vx)\vz\enspace,\\
        &= \frac{(p-1)\cdot e}{2} \vz^{\top}\left(\nabla^2h_{\vq}(\vx) - \nabla^2g_{\vq}(\vx)\right)\vz + \left(1+\frac{(p-1)^{p-1}}{C_p}\right) \vz^{\top}\nabla^2g_{\vq}(\vx)\vz\enspace,\\
        &\leq^{(p\geq 2)} p\cdot e \vz^{\top}\nabla^2h_{\vq}(\vx)\vz + \left(1+\frac{(p-1)^{p-1}}{C_p} - \frac{(p-1)\cdot e}{2}\right)\vz^{\top}\nabla^2g_{\vq}(\vx)\vz\enspace,\\
        &= p\cdot e \vz^{\top}\nabla^2h_{\vq}(\vx)\vz + \left(1+\frac{(p-1)^{p-1}}{ep^p} - \frac{(p-1)\cdot e}{2}\right)\vz^{\top}\nabla^2g_{\vq}(\vx)\vz\enspace,\\
        &\leq^{\text{(\Cref{lem:small_lemma_for_hessian_stability_one})}} p\cdot e \vz^{\top}\nabla^2h_{\vq}(\vx)\vz\enspace,
    \end{align*}
    where in the final inequality we use \Cref{lem:small_lemma_for_hessian_stability_one} which tell us that for $p\geq 2$ the constant in front of $\vz^{\top}\nabla^2g_{\vq}(\vx)\vz$ is negative along with the fact that $\vz^{\top}\nabla^2g_{\vq}(\vx)\vz$ is non-negative. To get the lower bound we first exchange $\vx, \vq$ in \eqref{eq:quad_form_f_ub} (and use the values of $\alpha_p$ and $\beta_p$) to get,
    \begin{align*}
        &\vz^{\top}\nabla^2f(\vq)\vz \leq \frac{(p-1)\cdot(p-1){p-2}}{ep^p}\vz^{\top}\nabla^2g_{\vx}(\vq)\vz + (p-1)\left(1 + \frac{1}{p-2}\right)^{p-2}\vz^{\top}\nabla^2f(\vx)\vz\enspace,\\
        &\Rightarrow \vz^{\top}\nabla^2f(\vq)\vz \leq \frac{(p-1)^{p-1}}{ep^p}\vz^{\top}\nabla^2g_{\vx}(\vq)\vz + (p-1)e\vz^{\top}\nabla^2f(\vx)\vz\enspace,\\
        \Rightarrow &\frac{1}{(p-1)e}\vz^{\top}\nabla^2f(\vq)\vz - \frac{(p-1)^{p-2}}{e^2p^p}\vz^{\top}\nabla^2g_{\vx}(\vq)\vz \leq \vz^{\top}\nabla^2f(\vx)\vz\enspace.
    \end{align*}
    We can finally lower bound,
    \begin{align*}
        \vz^{\top}\nabla^2f_{\vq}(\vx)\vz &= \vz^{\top}\nabla^2f(\vx)\vz + \vz^{\top}\nabla^2g_{\vq}(\vx)\vz\enspace,\\
        &\geq \frac{1}{(p-1)e}\vz^{\top}\nabla^2f(\vq)\vz - \frac{(p-1)^{p-2}}{e^2p^p}\vz^{\top}\nabla^2g_{\vx}(\vq)\vz + \vz^{\top}\nabla^2g_{\vq}(\vx)\vz\enspace,\\
        &= \frac{1}{2(p-1)e}\vz^{\top}\left(\nabla^2h_{\vq}(\vx)- \nabla^2g_{\vq}(\vx)\right)\vz - \frac{(p-1)^{p-2}}{e^2p^p}\vz^{\top}\nabla^2g_{\vx}(\vq)\vz + \vz^{\top}\nabla^2g_{\vq}(\vx)\vz\enspace,\\
        &\geq^{(g_{\vq}(\vx) = g_{\vx}(\vq))} \frac{1}{2pe}\vz^{\top}\nabla^2h_{\vq}(\vx)\vz + \left(1 - \frac{1}{2(p-1)e}- \frac{(p-1)^{p-2}}{e^2p^p}\right)\vz^{\top}\nabla^2g_{\vq}(\vx)\vz\enspace,\\
        &\geq^{\text{(\Cref{lem:small_lemma_for_hessian_stability_two})}}\frac{1}{2pe}\vz^{\top}\nabla^2h_{\vq}(\vx)\vz\enspace, 
    \end{align*}
    where in the final inequality we use \Cref{lem:small_lemma_for_hessian_stability_two} and the fact that $\vz^{\top}\nabla^2g_{\vq}(\vx)\vz$ is non-negative. This finishes the proof for Case A.

    We finally consider the corner case with $p=2$. 

    \paragraph{\underline{Case B.} ($p=2$)} In this case the proof is trivial, and follows from simply writing the quadratic forms for $f_{\vq}$ and $h_{\vq}$. We do so below,
    \begin{align*}
        \vz^{\top}\nabla^2f_{\vq}(\vx)\vz &= \vz^{\top}\nabla^2f(\vx)\vz + \vz^{\top}\nabla^2g_{\vq}(\vx)\vz\enspace,\\
        &= \vz^{\top}\nabla^2f(\vx)\vz + 2C_2\norm{\vz}^2_{\mM}\enspace,\\
        &\leq 2\vz^{\top}\nabla^2f(\vx)\vz + 2C_2\norm{\vz}^2_{\mM} = \vz^{\top}\nabla^2h_{\vq}(\vx)\vz\enspace,
    \end{align*}
    which shows the relative smoothness with a constant of $1$ which is smaller (and hence better) than the claimed constant (for $p=2$) of $2e$ in the lemma. Now for the relative strong convexity we do the same,
    \begin{align*}
        \vz^{\top}\nabla^2f_{\vq}(\vx)\vz &= \vz^{\top}\nabla^2f(\vx)\vz + 2C_2\norm{\vz}^2_{\mM}\enspace,\\
        &\geq \frac{1}{2}\cdot\left(2\vz^{\top}\nabla^2f(\vx)\vz + 2C_2\norm{\vz}^2_{\mM}\right)\enspace,\\
        &= \frac{1}{2}\vz^{\top}\nabla^2h_{\vq}(\vx)\vz\enspace,
    \end{align*}
    which shows relative strong-convexity with a constant of $\frac{1}{2}$ which is larger (and hence better) than the claimed constant (for $p=2$) of $\frac{1}{4e}$ in the lemma. This finishes the proof for Case B. 
    
    This completes the proof of \Cref{lemma:gp_hessian_stable}.
\end{proof}
We prove two small technical lemmas that we used in the above proof now. 
\begin{lemma}\label{lem:small_lemma_for_hessian_stability_one}
    For all $p\geq 2$, $g(p) = 1+\frac{(p-1)^{p-1}}{ep^p} - \frac{(p-1)\cdot e}{2} \leq 0$.
\end{lemma}
\begin{proof}
     First note that at $p=2$ the function takes a strictly negative value,
     \begin{align*}
         g(2) = 1+\frac{(1}{e2^2} - \frac{e}{2} = \frac{4e+ 1 - 2e^2}{4e} < 0\enspace. 
     \end{align*}
     We will now show that the function is increasing in $p$ for $p\geq 2$,
     \begin{align*}
         g'(p) &= -\frac{(p-1)^{p-1}p^p(\ln(p) + 1)}{p^2p} + \frac{(p-1)^{p-1}(\ln(p-1) + 1)}{p^p} - \frac{e}{2}\enspace,\\
         &= -\frac{(p-1)^{p-1}\ln(p/(p-1))}{p^p} - \frac{e}{2} < 0\enspace. 
     \end{align*}
     Thus, the function attains its maximum value at $p=2$ in the range $p\geq 2$, implying it is strictly negative in that range.
\end{proof}

\begin{lemma}\label{lem:small_lemma_for_hessian_stability_two}
    For all $p\geq 2$, $g(p) = 1-\frac{1}{2(p-1)e} - \frac{(p-1)^{p-2}}{e^2p^p} \geq 0$.
\end{lemma}
\begin{proof}
    First note that at $p=2$ the function takes a strictly positive value,
    \begin{align*}
        g(2) = 1 - \frac{1}{2e} -\frac{1^0}{e^22^2} = 1 - \frac{1}{2e} - \frac{1}{4e^2} = \frac{4e^2 - 2e - 1}{4e^2} > 0\enspace.
    \end{align*}
    We will now show that the function is increasing in $p$ for $p\geq 2$,
    \begin{align*}
        g'(p) &= \frac{1}{2(p-1)^2e} + \frac{(p-1)^{p-2}p^p(\ln(p) + 1)}{e^2p^{2p}} - \frac{(p-1)^{p-2}(\ln(p-1) + (p-2)/(p-1))}{e^2p^p}\enspace,\\
        &= \frac{1}{2(p-1)^2e} + \frac{(p-1)^{p-2}(\ln(p) + 1)}{e^2p^{p}} - \frac{(p-1)^{p-2}(\ln(p-1) + 1 - 1/(p-1))}{e^2p^p}\enspace,\\
        &= \frac{1}{2(p-1)^2e} + \frac{(p-1)^{p-2}\left(\ln(p/(p-1)) + 1/(p-1)\right)}{e^2p^{p}} > 0\enspace.
    \end{align*}
    Thus, the function $g$ attains its minimum value at $p=2$ in the range $p\geq2$, implying that it is strictly positive in that range.
\end{proof}

\subsubsection{Strong Convexity of the Proximal Objective and Friends}

We begin by showing that the proximal objective enjoys a form of strong convexity.
\begin{lemma}
\label{lemma:fq_strongly_convex}
For all $\vx,\vd\in\R^d$, we have
\begin{align*}
    f_{\vq}(\vx+\vd) \ge f_{\vq}(\vx) + \ip{\nabla f_{\vq}(\vx), \vd} + \frac{4}{2^p}\inparen{\gnorm{\mA\vd}{p}^p + C_p\norm{\vd}_{\mM}^p}.
\end{align*}
\end{lemma}
\begin{proof}[Proof of \Cref{lemma:fq_strongly_convex}]
Let $K_p \coloneqq \frac{4}{2^p}$.

The plan is to apply \Cref{lemma:strong_convexity_component} to $f_{\vq}(\vx+\vd)$. We start with the regularizer. Notice that
\begin{align}
    \norm{\vx+\vd-\vq}_{\mM}^p &= \norm{\mM^{1/2}(\vx+\vd-\vq)}_2^p = \norm{\mM^{1/2}(\vx-\vq) + \mM^{1/2}\vd}_2^p\enspace,\nonumber \\
    &\ge^{\text{(\Cref{lemma:strong_convexity_component})}} \norm{\mM^{1/2}(\vx-\vq)}_2^p \\
    &\quad + \ip{p\norm{\mM^{1/2}(\vx-\vq)}_2^{p-2}\mM^{1/2}(\vx-\vq),\mM^{1/2}\vd} + K_p\norm{\mM^{1/2}\vd}_2^p\enspace,\nonumber\\
    &= \norm{\vx-\vq}_{\mM}^p + \ip{p\norm{\vx-\vq}_{\mM}^{p-2}\mM(\vx-\vq), \vd} + K_p\norm{\vd}_{\mM}^p\enspace,\nonumber\\
    &= \norm{\vx-\vq}_{\mM}^p + \ip{\nabla_{\vx} \inparen{\norm{\vx-\vq}_{\mM}^p}, \vd} + K_p\norm{\vd}_{\mM}^p\enspace.\label{eq:M_norm_sc}
\end{align}
We combine this with the conclusion of \Cref{lemma:strong_convexity_gp}, giving
\begin{align*}
    f_{\vq}(\vx+\vd) &= f(\vx+\vd) + C_p\|\vx+\vd - \vq\|_{\mM}^p\enspace,\\
    &\geq^{\text{(\Cref{lemma:strong_convexity_gp})}} f(\vx) + \ip{\nabla f(\vx),\vd} + K_p\gnorm{\mA\vd}{p}^p+ C_p\|\vx+\vd - \vq\|_{\mM}^p\enspace,\\
    &\geq^{\text{\eqref{eq:M_norm_sc}}} f(\vx) + \ip{\nabla f(\vx),\vd} + K_p\gnorm{\mA\vd}{p}^p + C_p\norm{\vx-\vq}_{\mM}^p\\ 
    &\quad + C_p\ip{\nabla_{\vx} \inparen{\norm{\vx-\vq}_{\mM}^p}, \vd} + K_pC_p\norm{\vd}_{\mM}^p\enspace,\\
    &= \textcolor{red}{f(\vx)}  + \textcolor{red}{C_p\norm{\vx-\vq}_{\mM}^p} + \ip{\textcolor{blue}{\nabla_{\vx} \left(f(\vx) + C_p\norm{\vx-\vq}_{\mM}^p\right)},\vd}\\
    &\quad + K_p\gnorm{\mA\vd}{p}^p + K_pC_p\norm{\vd}_{\mM}^p\enspace,\\
    &= \textcolor{red}{f_{\vq}(\vx)} + \ip{\textcolor{blue}{\nabla f_{\vq}(\vx)},\vd} + K_p\inparen{\gnorm{\mA\vd}{p}^p + C_p\norm{\vd}_{\mM}^p}\enspace.
\end{align*}
completing the proof of \Cref{lemma:fq_strongly_convex}.
\end{proof}

We also show that the subproblems we solve in Line \ref{line:mirror_descent_exact} of \Cref{alg:mirror_descent} are strongly convex.
\begin{lemma}
\label{lemma:gp_h_strongly_convex}
Fix $\vz, \vq, \vd \in \R^d$ and let $L > 0$. Consider the function
\begin{align*}
    g(\vx) &\coloneqq \ip{\vz,\vx} + L\inparen{\norm{\vx-\vq}_{\nabla^2 f(\vq)}^2 + C_p\norm{\vx-\vq}_{\mM}^p}\enspace.
\end{align*}
Then,
\begin{align*}
    g(\vx+\vd) \ge g(\vx) + \ip{\nabla g(\vx),\vd} + L\inparen{\norm{\vd}_{\nabla^2 f(\vq)}^2 + \frac{4C_p}{2^p}\norm{\vd}_{\mM}^p}\enspace.
\end{align*}
In particular, if $\vz$ is the minimizer for $g$, then for any $\vd\in\R^d$, we have
\begin{align*}
    \norm{\vd}_{\mM} \le \frac{2}{p \cdot (4e)^{1/p}}\inparen{\frac{g(\vz+\vd) - g(\vz)}{L}}^{1/p}\enspace.
\end{align*}
\end{lemma}
\begin{proof}[Proof of \Cref{lemma:gp_h_strongly_convex}]
This is pretty much the same proof as \Cref{lemma:fq_strongly_convex}. It is easy to check that
\begin{align}
    \norm{(\vx+\vd)-\vq}_{\nabla^2 f(\vq)}^2 = \norm{\vx-\vq}_{\nabla^2 f(\vq)}^2 + \ip{2\nabla^2 f(\vq)(\vx-\vq),\vd} + \norm{\vd}_{\nabla^2 f(\vq)}^2\enspace,\label{eq:f_norm_expansion}
\end{align}
and using \Cref{lemma:strong_convexity_component} in the same way as in the proof of \Cref{lemma:fq_strongly_convex}, we have
\begin{align*}
    \norm{(\vx+\vd)-\vq}_{\mM}^p \ge^{\text{\eqref{eq:M_norm_sc}}} \norm{\vx-\vq}_{\mM}^p + \ip{p\norm{\vx-\vq}_{\mM}^{p-2}\mM(\vx-\vq),\vd} + \frac{4}{2^p}\norm{\vd}_{\mM}^p\enspace.
\end{align*}
Combining this with the definition of $g$ gives the following,
\begin{align*}
    g(\vx+\vd) &= \ip{\vz,\vx + \vd} + L\inparen{\norm{\vx + \vd-\vq}_{\nabla^2 f(\vq)}^2 + C_p\norm{\vx + \vd -\vq}_{\mM}^p}\enspace,\\
    &\geq^{\text{\eqref{eq:f_norm_expansion}, \eqref{eq:M_norm_sc}}} \textcolor{red}{\ip{\vz,\vx}} +  \ip{\vz,\vd} + \textcolor{red}{L\norm{\vx-\vq}_{\nabla^2 f(\vq)}^2} + L\ip{2\nabla^2 f(\vq)(\vx-\vq),\vd}\\
    &\qquad + L\norm{\vd}_{\nabla^2 f(\vq)}^2 + LC_p\left(\textcolor{red}{\norm{\vx-\vq}_{\mM}^p} + \ip{p\norm{\vx-\vq}_{\mM}^{p-2}\mM(\vx-\vq),\vd} + \frac{4}{2^p}\norm{\vd}_{\mM}^p\right)\enspace,\\
    &=  \textcolor{red}{g(\vx)}  + \ip{\textcolor{blue}{\vz} + \textcolor{blue}{2L\nabla^2 f(\vq)(\vx-\vq)} + \textcolor{blue}{LC_pp\norm{\vx-\vq}_{\mM}^{p-2}\mM(\vx-\vq)},\vd} \\
    &\qquad  + L\left(\norm{\vd}_{\nabla^2 f(\vq)}^2 + \frac{4C_p}{2^p}\norm{\vd}_{\mM}^p\right)\enspace,\\
    &= g(\vx) + \ip{\textcolor{blue}{\nabla g(\vx)},\vd} + L\left(\norm{\vd}_{\nabla^2 f(\vq)}^2 + \frac{4C_p}{2^p}\norm{\vd}_{\mM}^p\right)\enspace,
\end{align*}
which proves the first result of the lemma.

To get the second result, we observe that $\nabla g(\vz) = 0$ by the optimality of $\vz$. Ignoring the $\norm{\vd}_{\nabla^2 f(\vq)}$ terms and rearranging gives the conclusion of \Cref{lemma:gp_h_strongly_convex}.
\end{proof}

\subsubsection{Smoothness of the Proximal Objective}

We first bound the operator norm of a matrix related to the Hessian of the proximal objective.

\begin{lemma}
\label{lemma:gp_prox_hessian_opnorm}
For all $\vq,\vy\in\R^d$, we have
\begin{align*}
    \opnorm{\mM^{-1/2}\inparen{\nabla^2 f_{\vq}(\vy)}\mM^{-1/2}} &\le ep^2(p-1)\inparen{2f(\vq)^{1-\frac{2}{p}}+C_p\norm{\vy-\vq}_{\mM}^{p-2}}\enspace.
\end{align*}
\end{lemma}
\begin{proof}[Proof of \Cref{lemma:gp_prox_hessian_opnorm}]
Recall from the proof of \Cref{lemma:gp_hessian_stable} the definition of the regularization term $g_{\vq}(\vy) \coloneqq C_p\norm{\vy-\vq}_{\mM}^p$ for $C_p = ep^p$ as well as the following calculations,
 \begin{align*}
        g_{\vq}(\vy) &\coloneqq C_p\norm{\vy-\vq}^p_{\mM}\enspace,\\    
        \nabla g_{\vq}(\vy) &= pC_p\norm{\vy-\vq}_{\mM}^{p-2}{\mM}(\vy-\vq)\enspace,\\
        \nabla^2 g_{\vq}(\vy) &= pC_p\norm{\vy-\vq}_{\mM}^{p-2}\mM + p(p-2)C_p\norm{\vy-\vq}_{\mM}^{p-4}\mM(\vy-\vq)(\vy-\vq)^{\top}\mM\enspace.
\end{align*}
By \Cref{lemma:gp_hessian_stable}, we know that
\begin{align*}
    \nabla^2 f_{\vq}(\vy) \preceq ep\inparen{2\nabla^2 f(\vq) + \nabla^2 g_{\vq}(\vy)}.
\end{align*}
Observe that
\begin{align*}
    \quad \mM^{-1/2} \inparen{\nabla^2 g_{\vq}(\vy)} \mM^{-1/2} &= pC_p\inparen{\norm{\vy-\vq}_{\mM}^{p-2}+(p-2)\norm{\vy-\vq}_{\mM}^{p-4}\mM^{1/2}(\vy-\vq)(\vy-\vq)^{\top}\mM^{1/2}}\enspace,\\
    &\preceq pC_p\norm{\vy-\vq}_{\mM}^{p-2}\mI + (p-2)\norm{\vy-\vq}_{\mM}^{p-4}\opnorm{\mM^{1/2}(\vy-\vq)(\vy-\vq)^{\top}\mM^{1/2}}\mI\enspace,\\
    &\preceq pC_p\norm{\vy-\vq}_{\mM}^{p-2}\mI + (p-2)\norm{\vy-\vq}_{\mM}^{p-4}\norm{\mM^{1/2}(\vy-\vq)}_2^2\mI\enspace,\\ 
    &\preceq p(p-1)C_p \norm{\vy-\vq}_{\mM}^{p-2}\mI\enspace,
\end{align*}
and, applying \Cref{lemma:gp_regression_hessian_upperbound} (with $\mM^{-1/2}\vz$ as the vectors in the quadratic form) and H\"older inequality with norms $\|\cdot\|_{p/(p-2)},\ \|\cdot\|_{p/2}$, for $\vz\in\R^d$ we have\todo{add a reference for the third inequality. I keep forgetting which result this is in the previous sections.}
\begin{align*}
    \vz^{\top}\mM^{-1/2}\inparen{\nabla^2 f(\vq)} \mM^{-1/2}\vz &\le p(p-1)\sum_{i=1}^m \norm{\mA_{S_i}\vq-\vb_{S_i}}_2^{p-2}\norm{\mA_{S_i}\mM^{-1/2}\vz}_2^2 \\
    &\le p(p-1)\inparen{\sum_{i=1}^m \norm{\mA_{S_i}\vq-\vb_{S_i}}_2^p}^{\frac{p-2}{p}}\inparen{\sum_{i=1}^m \norm{\mA_{S_i}\mM^{-1/2}\vz}_2^p}^{\frac{2}{p}} \\
    &\le p(p-1)f(\vq)^{1-\frac{2}{p}}\norm{\mM^{-1/2}\vz}_{\mM}^2 = p(p-1)f(\vq)^{1-\frac{2}{p}}\norm{\vz}_2^2.
\end{align*}
Combining gives
\begin{align*}
    \mM^{-1/2}\inparen{\nabla^2f_{\vq}(\vy)}\mM^{-1/2} &\preceq ep\mM^{-1/2}\inparen{2\nabla^2f(\vq) + \nabla^2g_{\vq(\vy)}}\mM^{-1/2}\enspace,\\    
    &\preceq 2ep^2(p-1)f(\vq)^{1-\frac{2}{p}} + ep^2(p-1)C_p\norm{\vy-\vq}_{\mM}^{p-2}\enspace, \\
    &\preceq ep^2(p-1)\inparen{2f(\vq)^{1-\frac{2}{p}}+C_p\norm{\vy-\vq}_{\mM}^{p-2}},
\end{align*}
completing the proof of \Cref{lemma:gp_prox_hessian_opnorm}.
\end{proof}

Next, we show a bound on the norm of the gradient of any solution $\vx$ that is approximately optimal for $f_{\vq}$.

\begin{lemma}
\label{lemma:self_gradient_norm_small}
For all $\vq,\vx\in\R^d$, we have
\begin{align*}
    \norm{\mM^{-1}\nabla f_{\vq}(\vx)}_{\mM} &\le ep^2(p-1)\inparen{f(\vq)^{1-\frac{2}{p}}+C_p\max\inbraces{\norm{\vx-\vq}_{\mM},\norm{\vx_{\vq}-\vq}_{\mM}}^{p-2}}\norm{\vx-\vx_{\vq}}_{\mM}\enspace.
\end{align*}
\end{lemma}
\begin{proof}[Proof of \Cref{lemma:self_gradient_norm_small}]
We use a continuity argument. By Taylor's theorem, we know for some $\vy$ along the line connecting $\vx$ and $\vx_{\vq}$ (minimizer of $f_{\vq}$) that
\begin{align*}
    \nabla f_{\vq}(\vx) = \nabla f_{\vq}(\vx_{\vq}) + \nabla^2 f_{\vq}(\vy)(\vx-\vx_{\vq}) = \nabla^2 f_{\vq}(\vy)(\vx-\vx_{\vq})\enspace.
\end{align*}
Taking $\mM^{-1}$-norm of both sides gives,
\begin{align*}
    \norm{\nabla f_{\vq}(\vx)}_{\mM^{-1}} &= \norm{\mM^{-1/2}\nabla f_{\vq}(\vx)}_{2}\enspace,\\
     &= \norm{\mM^{-1/2}\nabla^2 f_{\vq}(\vy)(\vx-\vx_{\vq})}_{2}\enspace,\\
     &= \norm{\mM^{-1/2}\nabla^2 f_{\vq}(\vy)\mM^{-1/2}\mM^{1/2}(\vx-\vx_{\vq})}_{2}\enspace,\\
    &\le \opnorm{\mM^{-1/2}\inparen{\nabla^2 f_{\vq}(\vy)}\mM^{-1/2}} \cdot \norm{\vx-\vx_{\vq}}_{\mM}\enspace.
\end{align*}
The rest of the proof involves bounding the operator norm term. This follows directly from \Cref{lemma:gp_prox_hessian_opnorm}, from which we get (using convexity of $\|\cdot\|_{\mM}$),
\begin{align*}
    \opnorm{\mM^{-1/2}\nabla^2 f_{\vq}(\vy)\mM^{-1/2}} &\le ep^2(p-1)\inparen{2f(\vq)^{1-\frac{2}{p}}+C_p\norm{\vy-\vq}_{\mM}^{p-2}} \\
    &\le ep^2(p-1)\inparen{2f(\vq)^{1-\frac{2}{p}}+C_p\max\inbraces{\norm{\vx-\vq}_{\mM},\norm{\vx_{\vq}-\vq}_{\mM}}^{p-2}}.
\end{align*}
Putting everything together, we get
\begin{align*}
    \norm{\mM^{-1}\nabla f_{\vq}(\vx)}_{\mM}&= \norm{\nabla f_{\vq}(\vx)}_{\mM^{-1}}\enspace, \\
    &\le ep^2(p-1)\inparen{2f(\vq)^{1-\frac{2}{p}}+C_p\max\inbraces{\norm{\vx-\vq}_{\mM},\norm{\vx_{\vq}-\vq}_{\mM}}^{p-2}}\norm{\vx-\vx_{\vq}}_{\mM},
\end{align*}
completing the proof of \Cref{lemma:self_gradient_norm_small}.
\end{proof}

\subsubsection{Solving the Proximal Subproblems}

We begin by showing that the optimal solution to the proximal problem $\vx_{\vq_t}\coloneqq \argmin{\vx\in\R^d} f_{\vq_t}(\vx)$ is not too far from $\xstar$.

\begin{lemma}
\label{lemma:gp_regression_prox_diameter}
For all proximal queries $\vq_t$, we have
\begin{align*}
    \norm{\vx_{\vq_t}-\xstar}_{\mM} \le d^{\frac{1}{2}-\frac{1}{p}}\inparen{2^{\frac{3}{2}}f(\vx_t)+4}.
\end{align*}
\end{lemma}
\begin{proof}
In the rest of this proof, we omit the subscript $t$ wherever it is clear which iterates we are working with. 

We first show that
\begin{align*}
    \norm{\vx_{\vq}-\vq}_{\mM} \le \norm{\xstar-\vq}_{\mM}\enspace.
\end{align*}
To see this, suppose this is not the case. Then, we have
\begin{align*}
    f(\xstar) + C_p\norm{\xstar-\vq}_{\mM}^p < f(\vx_{\vq}) + C_p\norm{\vx_{\vq}-\vq}_{\mM}^p\enspace,
\end{align*}
which contradicts the optimality of $\vx_{\vq}$ for $f_{\vq}$.

We now write
\begin{align*}
    \norm{\vx_{\vq_t}-\xstar}_{\mM} &\le \norm{\vx_{\vq_t}-\vq_t}_{\mM} + \norm{\xstar-\vq_t}_{\mM}\enspace, \\
    &\le 2\norm{\xstar-\vq_t}_{\mM}\enspace,\\
    &\le 2\inparen{\norm{\vx_t-\xstar}_{\mM}+\norm{\vv_t-\xstar}_{\mM}}\enspace,
\end{align*}
where in the last inequality, we used the definition of $\vq_t$ from Line 6 in \Cref{alg:optimal_ms_acceleration} and the convexity of $\|\cdot\|_{\mM}$. The required control on $\norm{\vv_t-\xstar}_{\mM}$ comes from \Cref{lemma:ms_acceleration_iterate_diameter} and \Cref{lemma:gp_regression_initialization} (along with re-scaling assumption to make the optimal value $1$) -- we have
\begin{align*}
    \norm{\vv_t-\xstar}_{\mM} \le \sqrt{2}\norm{\vx_0-\xstar}_{\mM} \le 4d^{\frac{1}{2}-\frac{1}{p}}\enspace.
\end{align*}
For the other term, we apply \Cref{lemma:strong_convexity_gp} and get
\begin{align*}
    \norm{\vx_t-\xstar}_{\mM} \le 2^{\frac{3}{2}}d^{\frac{1}{2}-\frac{1}{p}}\inparen{f(\vx_t)-f(\xstar)}^{\frac{1}{p}} < 2^{\frac{3}{2}}d^{\frac{1}{2}-\frac{1}{p}}f(\vx_t)^{\frac{1}{p}}.
\end{align*}
Adding gives us the conclusion of \Cref{lemma:gp_regression_prox_diameter}.
\end{proof}

The next few lemmas are targeted at solving the proximal subproblems. We begin with a calculation that we will use in showing that the initial Bregman divergence between our initialization and the optimum is small.

\begin{lemma}
\label{lemma:initial_bregman_div}
In the same setting as \Cref{lemma:gp_hessian_stable}, for all $\vx,\vy \in \R^d$, we have
\begin{align*}
    h_{\vq}(\vx_{\vq}) \le p(p-1)f(\vq)^{1-\frac{2}{p}}\norm{\vx_{\vq}-\vq}_{\mM}^2 + C_p\norm{\vx_{\vq}-\vq}_{\mM}^p < f(\vq)+C_p\norm{\vx_{\vq}-\vq}_{\mM}^p \le 2f(\vq).
\end{align*}
\end{lemma}
\begin{proof}[Proof of \Cref{lemma:initial_bregman_div}]
By optimality of $\vx_{\vq}$ for the subproblem, we have
\begin{align*}
    f(\vx_{\vq})+C_p\norm{\vx_{\vq}-\vq}_{\mM}^p \le f(\vq)+C_p\norm{\vq-\vq}_{\mM}^p=f(\vq).
\end{align*}
Rearranging gives,
\begin{align}
    \norm{\vx_{\vq}-\vq}_{\mM}^p \le \frac{f(\vq)-f(\vx_{\vq})}{C_p} \le \frac{f(\vq)}{C_p}\enspace.\label{eq:norm_M_ub_f}
\end{align}
We now use the definition of $h_{\vq}$ and \Cref{lemma:gp_regression_hessian_upperbound} to write
\begin{align*}
    h_{\vq}(\vx_{\vq}) &= \norm{\vx_{\vq}-\vq}_{\nabla^2 f(\vq)}^2 + C_p\norm{\vx_{\vq}-\vq}_{\mM}^p\enspace,\\
    &\le^{\text{\Cref{lemma:gp_regression_hessian_upperbound}}} p(p-1)\sum_{i=1}^m \norm{\mA_{S_i}\vq-\vb_{S_i}}_2^{p-2}\norm{\mA_{S_i}(\vx_{\vq}-\vq)}_2^2 + C_p\norm{\vx_{\vq}-\vq}_{\mM}^p\enspace,\\
    &\le^{\text{(a)}} p(p-1)\inparen{\sum_{i=1}^m \norm{\mA_{S_i}\vq-\vb_{S_i}}_2^p}^{1-\frac{2}{p}}\inparen{\sum_{i=1}^m \norm{\mA_{S_i}(\vx_{\vq}-\vq)}_2^p}^{\frac{2}{p}}+ C_p\norm{\vx_{\vq}-\vq}_{\mM}^p\enspace,\\
    &\le^{\text{(b)}} p(p-1)f(\vq)^{1-\frac{2}{p}}\norm{\vx_{\vq}-\vq}_{\mM}^{2} + C_p\norm{\vx_{\vq}-\vq}_{\mM}^p\enspace,\\
    &\le^{\text{\eqref{eq:norm_M_ub_f}}} p(p-1)f(\vq)^{1-\frac{2}{p}}\inparen{\frac{f(\vq)}{C_p}}^{\frac{2}{p}} + C_p\norm{\vx_{\vq}-\vq}_{\mM}^p\enspace,\\
    &=^{\text{($C_p = ep^p$)}} \frac{(p-1)}{ep}f(\vq)+ C_p\norm{\vx_{\vq}-\vq}_{\mM}^p\enspace,\\
    &< f(\vq) + C_p\norm{\vx_{\vq}-\vq}_{\mM}^p\enspace,\\
    &<^{\text{\eqref{eq:norm_M_ub_f}}} 2f(\vq)\enspace,
\end{align*}
where in (a) we used H\"older inequality with norms $\|\cdot\|_{p/(p-2)},\ \|\cdot\|_{p/2}$ and in (b) we used \Cref{thm:gp_regression_blw_intro} \todo{again what's the reference for this one?}. 

This completes the proof for the series of inequalities in \Cref{lemma:initial_bregman_div}.
\end{proof}
We now have the tools to show how to approximately solve problems in Line \ref{line:mirror_descent_exact} of \Cref{alg:mirror_descent} when applied in our setting. Although this and future complexity bounds depend on $f(\vx_t)$, we will later be able to use \Cref{thm:optimal_ms_acceleration} to ``bootstrap'' and get an unconditional upper bound below.
\begin{lemma}
\label{lemma:gp_prox_relativesmoothsolve}
Let $\alpha \le 1/2$. In the context of \Cref{alg:gp_regression_alg}, there exists an algorithm that approximately solves subproblems of the form (for $p\geq 2$ and $L=pe$),
\begin{align*}
    \vz \coloneqq \argmin{\vx\in\R^d} \ip{\vg,\vx} + L\inparen{\norm{\vx-\vq}_{\nabla^2 f(\vq)}^2 + C_p\norm{\vx-\vq}_{\mM}^p}\enspace,
\end{align*}
in the sense that we output $\vx$ for which,
\begin{equation*}
    \begin{aligned}
        \max\left\{\norm{\vx-\vz}_{\mM}, \norm{\mM^{-1}\vg + 2L\inparen{\mM^{-1}\nabla^2 f(\vq)(\vx-\vq) + C_p\norm{\vx-\vq}_{\mM}^{p-2}(\vx-\vq)}}_{\mM}\right\} \le \alpha
    \end{aligned}\enspace.
\end{equation*}
The algorithm takes $p^{O(1)}\logv{\frac{pd \cdot f(\vq)}{\alpha}}$ linear-system-solves in matrices of the form $\mA^{\top}\mB\mA$ for block-diagonal $\mB$, where each block in $\mB$ has size $\abs{S_i} \times \abs{S_i}$.
\end{lemma}
\begin{proof}[Proof of \Cref{lemma:gp_prox_relativesmoothsolve}]
This proof is lengthy, and splitting it into lemmas would disrupt the intended reading flow. So we break it up into several key components here.

\smallpar{Motivation for the lemma.} First, let us see why this lemma is even useful. In each iteration of \Cref{alg:gp_prox_solver}, which in turn calls \Cref{alg:mirror_descent}, the main primitive is computing\todo{I think this part should appear before the lemma because it leads to confusion within the proof about what sub-problem we are solving. If it is motivation, we don't need a separate lemma we can just put it before the lemma.}
\begin{align*}
    \xtilde_i &= \argmin{\xtilde \in \R^d} f_{\vq_t}(\xtilde_{i-1}) + \ip{\nabla f_{\vq_t}(\xtilde_{i-1}), \xtilde-\xtilde_{i-1}} + peD_{h_{\vq_t}}(\xtilde,\xtilde_{i-1})\enspace, \\
    &= \argmin{\xtilde \in \R^d} f_{\vq_t}(\xtilde_{i-1}) + \ip{\nabla f_{\vq_t}(\xtilde_{i-1}), \xtilde-\xtilde_{i-1}} + pe \left(h_{\vq_t}(\xtilde) - h_{\vq_t}(\xtilde_{i-1}) - \ip{\nabla h_{\vq_t}(\xtilde_{i-1}), \xtilde - \xtilde_{i-1}}\right)\enspace, \\
    &= \argmin{\xtilde \in \R^d} f_{\vq_t}(\xtilde_{i-1})-peh_{\vq_t}(\xtilde_{i-1}) + \ip{\nabla f_{\vq_t}(
    \xtilde_{i-1})-pe\nabla h_{\vq_t}(\xtilde_{i-1}), \xtilde-\xtilde_{i-1}} + peh_{\vq_t}(\xtilde)\enspace, \\
    &= \argmin{\xtilde\in\R^d} \ip{\nabla f_{\vq_t}(
    \xtilde_{i-1})-pe\nabla h_{\vq_t}(\xtilde_{i-1}), \xtilde} + peh_{\vq_t}(\xtilde)\enspace.
\end{align*}
Observe that the subproblem is of the form
\begin{align}
    \vz &= \argmin{\vx\in\R^d} \ip{\vg, \vx} + peh_{\vq}(\vx)\enspace,\nonumber\\
    &= \argmin{\vx\in\R^d} \ip{\vg,\vx}+pe\inparen{\norm{\vx-\vq}_{\nabla^2 f(\vq)}^2 + C_p\norm{\vx-\vq}_{\mM}^p}\enspace,\label{eq:equiv_prox_problem}
\end{align}
and so our goal is to show how to solve these types of problems.

\smallpar{The general algorithm.} 
Consider solving the related subproblem (instead of \eqref{eq:equiv_prox_problem}),
\begin{align*}
    \argmin{\vx\in\R^d} \ip{\vg,\vx}+L\inparen{\norm{\vx-\vq}_{\nabla^2 f(\vq)}^2 + C_p\tau\norm{\vx-\vq}_{\mM}^2}
\end{align*}
for some fixed $\tau \ge 0$. This is a quadratic problem, and we can therefore solve it in $1$ linear-system-solve. It is easy to check that at optimality, we have
\begin{align*}
    \vg + 2pe\inparen{\nabla^2 f(\vq)(\vx-\vq) + C_p\tau\mM(\vx-\vq)} = 0\enspace,
\end{align*}
which rearranges to\footnote{Recall that $\nabla^2 f(\vq)=\mA^{\top}\mB_1\mA$ for block-diagonal $\mB_1$ and by construction, $\mM = \mA^{\top}\mW^{1-\frac{2}{p}}\mA$ where $\mW$ consists of the block Lewis weights on the diagonal. Thus, $\nabla^2 f(\vq) + C_p\tau\mM=\mA^{\top}\mB_2\mA$ for block-diagonal $\mB_2$.}
\begin{align*}
    \vx-\vq = -\frac{1}{2pe}\inparen{\nabla^2 f(\vq)+C_p\tau\mM}^{-1}\vg\enspace.
\end{align*}

Note that at optimality for our original subproblem \eqref{eq:equiv_prox_problem}, we have $\tau^\star := \norm{\vz-\vq}_{\mM}^{p-2}$ where $\vz$ is the solution of subproblem \eqref{eq:equiv_prox_problem}. Also note that $\norm{\vx-\vq}_{\mM}$ is a decreasing function in $\tau$ because,
\begin{align*}
    \norm{\vx-\vq}_{\mM}^2 = \frac{1}{4p^2e^2}\|\vg\|^2_{\inparen{\nabla^2 f(\vq)+C_p\tau\mM}^{-1}\mM\inparen{\nabla^2 f(\vq)+C_p\tau\mM}^{-1}}\enspace, 
\end{align*}
and for $\tau_1 \leq \tau_2$,
\begin{align*}
    \inparen{\nabla^2 f(\vq)+C_p\tau_1\mM}^{-1}\mM\inparen{\nabla^2 f(\vq)+C_p\tau_1\mM}^{-1} \succeq \inparen{\nabla^2 f(\vq)+C_p\tau_2\mM}^{-1}\mM\inparen{\nabla^2 f(\vq)+C_p\tau_2\mM}^{-1}\enspace.
\end{align*}
We therefore see that if $\tau > \norm{\vx-\vq}_{\mM}^{p-2}$ --- where $\vx$ is the optimal solution for a fixed $\tau$ --- then we are over-regularizing and need to decrease $\tau$ and vice-versa. 
This means we can binary search for the appropriate value of $\tau$. To execute this, we first need to establish the accuracy up to which we have to identify $\tau$. 

\smallpar{Convergence in Argument.} 
\todo{I can't follow this. I tried writing the function values, but it is certainly not obvious from looking at the expression; we need more steps and words here.}
By \Cref{lemma:gp_h_strongly_convex} (setting $\vd = \vx-\vz$), recall that it is enough to solve sub-problem \eqref{eq:equiv_prox_problem} up to additive accuracy $(p/2)^pL\alpha^p$ to get $\norm{\vx-\vz}_{\mM} \le \alpha$. Suppose we find $\tau$ for which $\tau^{\star} \le \tau \le \tau^{\star} + \delta$. By writing the objectives and comparing, we see that the $\vx$ we find from using $\tau$ gives us at most a $\delta \cdot d$-suboptimal solution compared to $\vz$. Plugging this into the bound from \Cref{lemma:gp_h_strongly_convex} tells us that we should choose $\delta = (p/2)^pL\alpha^p/d$, and plugging this into the binary search over $\tau \in [0,d^p(1+f(\vq))]$ gives us $p^{O(1)}\logv{\frac{pd \cdot f(\vq)}{\alpha}}$ steps, as needed.\todo{There is a typo here, but I think we discussed we anyways want to make the statement of the lemma iterate dependent; should discuss.}

\smallpar{First-order stationary point.} We first claim that it is enough to get
\begin{align*}
    \norm{\mM^{-1}\nabla h_{\vq}(\vx)-\mM^{-1}\nabla h_{\vq}(\vz)}_{\mM} \le \frac{\alpha}{L}.
\end{align*}
Indeed, let $\vz$ be the optimal solution for the subproblem. This means that it must satisfy the first order stationary condition, namely,
\begin{align*}
    \vg + L\nabla h_{\vq}(\vz) = 0.
\end{align*}
Multiplying both sides by $\mM^{-1}$, subtracting, and dividing both sides by $L$ gives us the expression we are interested in.

Writing first order stationary conditions gives both
\begin{equation*}
    \begin{aligned}
        \vg + 2L\inparen{\nabla^2 f(\vq)(\vx-\vq)+C_p\tau\mM(\vx-\vq)} &= 0 \\
        \vg + 2L\inparen{\nabla^2 f(\vq)(\vz-\vq)+C_p\tau^{\star}\mM(\vz-\vq)} &= 0
    \end{aligned}.
\end{equation*}
Multiplying both sides of both equalities by $\mM^{-1}$ and subtracting these gives
\begin{align*}
    2L\inparen{\mM^{-1}\nabla^2 f(\vq)(\vx-\vz) + C_p\inparen{\tau(\vx-\vq)-\tau^{\star}(\vz-\vq)}} = 0.
\end{align*}
Expanding out $L(\mM^{-1}\nabla h_{\vq}(\vx)-\mM^{-1}h_{\vq}(\vz))$ and subtracting the above gives the desired condition
\begin{align*}
    2L\abs{\tau-\norm{\vx-\vq}_{\mM}^{p-2}} \cdot \norm{\vx-\vq}_{\mM} \overset{?}{\le} \alpha.
\end{align*}
Next, let us run the binary search from above so that we get argument convergence, i.e. $\norm{\vx-\vz}_{\mM} \le \alpha^C \ll 0.1\alpha$ for some constant $C$. Using the fact that the approximate mirror descent step using $\vz$ decreases the objective value (\Cref{lemma:gp_md_descent}), observe that
\begin{align*}
    \norm{\vx-\vq}_{\mM} \le \norm{\vz-\vq}_{\mM} + \norm{\vx-\vz}_{\mM} \le \norm{\vq-\vz}_{\mM}+0.1\alpha \lesssim \sqrt{d}(1+f(\vq)). 
\end{align*}
It then follows that binary searching $\tau$ to additive accuracy $\alpha(\sqrt{d}(1+f(\vq)))^{-1}/L$ is sufficient. By the same argument as above, this takes $p^{O(1)}\logv{\frac{pd \cdot f(\vq_t)}{\alpha}}$ steps, completing the proof of \Cref{lemma:gp_prox_relativesmoothsolve}.
\end{proof}
We now combine \Cref{lemma:gp_prox_relativesmoothsolve} with \Cref{thm:gp_mirror_descent} and \Cref{alg:mirror_descent} to obtain approximate argument optimality for each proximal subproblem.
\begin{lemma}
\label{lemma:prox_subproblem_base}
Let $\gamma > 0$ and $\vx_{\vq} \coloneqq \argmin{\vx\in\R^d} f_{\vq}(\vx)$. There exists an algorithm that returns $\vx$ for which
\begin{align*}
    \norm{\vx-\vx_{\vq}}_{\mM} &\le \gamma.
\end{align*}
The algorithm takes at most $O\inparen{p^{O(1)}\logv{ph_{\vq}(\vx_{\vq})\inparen{\frac{4}{p\gamma}}^{p}}}$ iterations of solving subproblems of the form $\argmin{\vx\in\R^d} \ip{\vg,\vx}+eph_{\vq}(\vx)$ for fixed vectors $\vg$ and $\vq$.
\end{lemma}
\begin{proof}[Proof of \Cref{lemma:prox_subproblem_base}]
This proof resembles \cite[Lemma 4.5]{jls21}, which uses an exact version of mirror descent arising from \citet{lfn18}. The main difference between our argument and that of \cite[Lemma 4.5]{jls21} is that we rigorously identify a concrete upper bound on the complexity needed to satisfy the MS condition and argue that the mirror descent algorithm can handle the inexact Bregman proximal problem solves.

First, we use \Cref{lemma:fq_strongly_convex} on the approximate solution $\vx$ and true solution $\vx_{\vq}$ and get,
\begin{align*}
    f_{\vq}(\vx) &\ge f_{\vq}(\vx_{\vq}) + \frac{4}{2^p}\inparen{\gnorm{\mA(\vx-\vx_{\vq})}{p}^p+C_p\norm{\vx_{\vq}-\vx}_{\mM}^p}\enspace,\\
    &\geq f_{\vq}(\vx) + \frac{4C_p}{2^p}\norm{\vx_{\vq}-\vx}_{\mM}^p\enspace.
\end{align*}
Rearranging, we get
\begin{align*}
    \norm{\vx_{\vq}-\vx}_{\mM} &\le \inparen{\frac{2^p}{4C_p}}^{1/p}\inparen{f_{\vq}(\vx)-f_{\vq}(\vx_{\vq})}^{1/p}\enspace,\\
    &= \inparen{\frac{2^p}{4ep^p}}^{1/p}\inparen{f_{\vq}(\vx)-f_{\vq}(\vx_{\vq})}^{1/p}\enspace,\\
    &< \frac{2}{p}\inparen{f_{\vq}(\vx)-f_{\vq}(\vx_{\vq})}^{1/p}\enspace.
\end{align*}
Using the notation from \cite{lfn18}, for convex $\myfunc{h}{\R^d}{\R}$, let
\begin{align*}
    D_h(\vx,\vy) \coloneqq h(\vx) - h(\vy) - \ip{\nabla h(\vy), \vx-\vy}.
\end{align*}
Recall the conclusion of \Cref{lemma:gp_hessian_stable} -- we have for $\mu = 1/(2pe)$ and $L = pe$ that
\begin{align*}
    \mu\nabla^2 h_{\vq}(\vx) \preceq \nabla^2 f_{\vq}(\vx) \preceq L\nabla^2 h_{\vq}(\vx).
\end{align*}
By \Cref{thm:gp_mirror_descent} and \Cref{lemma:gp_hessian_stable}, using the same notation from \Cref{lemma:gp_hessian_stable}, we have for all iterations $t$ of \Cref{alg:mirror_descent} (with $f = f_{\vq}$ and $h = h_{\vq}$) that,\todo{not sure about the second equality...}
\begin{align*}
    f_{\vq}(\vx_t) - f_{\vq}(\vx_{\vq}) &\le L\inparen{1-\frac{\mu}{L}}^tD_{h_{\vq}}(\vx_{\vq},\vq)+\max_{1\le i \le t} \ip{\triangle_i, \vx_t-\vx_{\vq}}\enspace,\\
    &= 2L\inparen{1-\frac{\mu}{L}}^th_{\vq}(\vx_{\vq})+\max_{1\le i \le t} \ip{\triangle_i, \vx_t-\vx_{\vq}}\enspace.
\end{align*}
Hence, for $t \ge \frac{L}{\mu}\logv{Lh_{\vq}(\vx_{\vq})\inparen{\frac{4}{p\gamma}}^{p}}$, it is easy to check that for $p\geq 2$,
\begin{align*}
    f_{\vq}(\vx_t) - f_{\vq}(\vx_{\vq}) &\leq 2L\inparen{\frac{1}{e}}^{\logv{Lh_{\vq}(\vx_{\vq})\inparen{\frac{4}{p\gamma}}^{p}}}h_{\vq}(\vx_{\vq}) + \max_{1\le i \le t} \ip{\triangle_i, \vx_t-\vx_{\vq}}\enspace,\\
    &= 2\inparen{\frac{p\gamma}{4}}^{p} + \max_{1\le i \le t} \ip{\triangle_i, \vx_t-\vx_{\vq}}\enspace,\\
    &\leq \inparen{\frac{p\gamma}{2}}^{p} + \max_{1\le i \le t} \ip{\triangle_i, \vx_t-\vx_{\vq}}\enspace,
\end{align*}
and combining this with \Cref{lemma:gp_prox_relativesmoothsolve} to make the error term on the order of our accuracy, we get $\norm{\vx_{\vq}-\vx}_{\mM} \lesssim \gamma$. We thus conclude the proof of \Cref{lemma:prox_subproblem_base}. \todo{I am not sure what is happening in the last line....}
\end{proof}

The last step is to use our proximal problem solver to build a valid MS oracle.

\begin{lemma}
\label{lemma:prox_subproblem_ms}
In the context of \Cref{alg:optimal_ms_acceleration}, there exists an algorithm $(\xtilde_{t+1},\lambda_{t+1})=\oprox(\vq_t)$ that approximately solves
\begin{align*}
    \argmin{\xtilde\in\R^d} f(\xtilde)+ep^p\norm{\xtilde-\vq_t}_{\mM}^p
\end{align*}
using $O\inparen{p^{O(1)}\logv{\frac{pd \cdot f(\vx_t)}{\eps}}}$ linear-system-solves in $\mA^{\top}\mB\mA$, in the sense that
\begin{align*}
    \norm{\frac{1}{ep^{p+1}\norm{\xtilde_{t+1}-\vq_t}_{\mM}^{p-2}}\mM^{-1}\nabla f(\xtilde_{t+1})+(\xtilde_{t+1}-\vq_t)}_{\mM} \le \frac{1}{2}\norm{\xtilde_{t+1}-\vq_t}_{\mM}.
\end{align*}
\end{lemma}
\begin{proof}[Proof of \Cref{lemma:prox_subproblem_ms}]
The point of this proof is to give an analysis of \Cref{alg:gp_prox_solver}.

For notational simplicity, let $\vx=\xtilde_{t+1}$ and $\lambda = \lambda_{t+1}$. We will reintroduce the indices when it is essential to clarify the iterations we are discussing.

First, it is helpful to see why the stated notion of approximation is useful. Let $C_p \coloneqq ep^p$. Observe that at exact optimality, we have
\begin{align}
    \nabla f(\vx_{\vq}) + \underbrace{ep^{p+1}\norm{\vx_{\vq}-\vq}_{\mM}^{p-2}}_{\lambda^{\star}}\mM(\vx-\vq)=0\enspace.\label{eq:prox_problem_stationarity_condition}
\end{align}
This motivates the approximation in our lemma statement, with us asking for a $\frac{1}{2}$-approximate MS oracle (\Cref{defn:ms_oracle}) for $f$. This also tells us that at optimality in \eqref{eq:prox_problem_stationarity_condition}, we have,
\begin{align*}
    &\nabla f(\vx_{\vq}) + ep^{p+1}\norm{\vx_{\vq}-\vq}_{\mM}^{p-2}\mM(\vx-\vq)=0\enspace,\\
    &\Leftrightarrow \mM^{-1/2}f(\vx_{\vq}) = -pC_p\norm{\vx_{\vq}-\vq}_{\mM}^{p-2}\mM^{1/2}(\vx-\vq)\enspace,\\
    &\Rightarrow \norm{\mM^{-1/2}f(\vx_{\vq})}_{2} = pC_p\norm{\vx_{\vq}-\vq}_{\mM}^{p-2}\norm{\mM^{1/2}(\vx-\vq)}_2\enspace,\\
    &\Leftrightarrow\norm{\vx_{\vq}-\vq}_{\mM} = \inparen{\frac{\norm{\mM^{-1}\nabla f(\vx_{\vq})}_{\mM}}{pC_p}}^{\frac{1}{p-1}}\enspace.
\end{align*}
We now break up our analysis into two cases. In the first, suppose that $\norm{\mM^{-1}\nabla f(\vx_{\vq})}_{\mM} \le \eps/\norm{\vx_{\vq}-\xstar}_{\mM}$. Then, by convexity, we have
\begin{align*}
    f(\vx_{\vq})-f(\xstar) \le \ip{\nabla f(\vx_{\vq}), \vx_{\vq}-\xstar} \le \norm{\mM^{-1}\nabla f(\vx_{\vq})}_{\mM}\norm{\vx_{\vq}-\xstar}_{\mM} \le \eps.   
\end{align*}
Hence, for the rest of the proof, assume that $\norm{\mM^{-1}\nabla f(\vx_{\vq})} \ge \eps/\norm{\vx_{\vq}-\xstar}_{\mM}$ (because if this is not the case, in the algorithm we can simply check whether the MS condition is satisfied -- if not, then we know this assumption was violated and we are done anyway)\todo{I am not sure I follow the statement in the parenthesis, maybe I need to discuss with you the hierarchy of the algorithms and see where specifically this termination condition comes in}. We run the algorithm implied by \Cref{lemma:prox_subproblem_base} and obtain an approximate solution $\vx$ for which
\begin{align}
    \norm{\vx-\vx_{\vq}}_{\mM} \le \alpha\norm{\vx_{\vq}-\vq}_{\mM} \text{ for } \alpha = \frac{1}{5}\min\inbraces{\frac{C_p}{ep(p-1)}\inparen{\frac{\norm{\vx_{\vq}-\vq}_{\mM}}{f(\vq)^{\frac{1}{p}}}}^{p-2}, 1}\enspace.\label{eq:lemma_D19_guarantee}
\end{align}
Since $\alpha < 1$ the guarantee in \eqref{eq:lemma_D19_guarantee} gives us,
\begin{align}
    \norm{\vx-\vx_{\vq}}_{\mM} \le \alpha\norm{\vx-\vq}_{\mM} \leq \frac{\alpha}{1-\alpha}
    \norm{\vx-\vq}_{\mM}\enspace,\label{eq:xq_x_UB_1}
\end{align}
and further applying triangle inequality gives us
\begin{align}
    \norm{\vx_{\vq}-\vq}_{\mM} &\le \norm{\vx-\vq}_{\mM} +  \norm{\vx_{\vq}-\vx}_{\mM}\enspace,\nonumber\\ 
    &\le \frac{1-\alpha}{1-\alpha}\norm{\vx-\vq}_{\mM} + \frac{\alpha}{1-\alpha}\norm{\vx-\vq}_{\mM}\enspace,\nonumber\\
    &\le \frac{1}{1-\alpha}\norm{\vx-\vq}_{\mM}\enspace.\label{eq:xq_x_UB_2}
\end{align}
Hence, we get
\begin{align}
    \frac{ep(p-1)f(\vq)^{1-\frac{2}{p}}}{C_p\norm{\vx-\vq}_{\mM}^{p-2}} \cdot \norm{\vx-\vx_{\vq}}_{\mM} &= \frac{ep(p-1)}{C_p} \cdot \inparen{\frac{f(\vq)^{\frac{1}{p}}}{\norm{\vx-\vq}_{\mM}}}^{p-2} \cdot \norm{\vx-\vx_{\vq}}_{\mM}\enspace,\nonumber\\
    &\le^{\eqref{eq:lemma_D19_guarantee}} \frac{1}{5}\norm{\vx_{\vq}-\vq}_{\mM}\enspace,\nonumber\\ &\le^{\eqref{eq:xq_x_UB_2}} \frac{1}{5}\cdot\frac{1}{1-\alpha}\norm{\vx-\vq}_{\mM}\enspace,\nonumber\\
    &\leq \frac{1}{4}\norm{\vx-\vq}_{\mM}\enspace,\label{eq:xq_x_UB_3}
\end{align}
where in the last inequality, we used that $\alpha \leq \frac{1}{5}$ due to our choice in \eqref{eq:lemma_D19_guarantee}. We now call \Cref{lemma:self_gradient_norm_small}, divide both sides by $\lambda$, and get\todo{I tried to make sense of the steps here, unclear how the first and last inequalities are working.}
\begin{align*}
    &\norm{\frac{1}{ep^{p+1}\norm{\vx-\vq}_{\mM}^{p-2}}\mM^{-1}\nabla f(\vx)+(\vx-\vq)}_{\mM} \\
    &\le^{\text{(\Cref{lemma:self_gradient_norm_small})}} ep(p-1)\inparen{\frac{f(\vq)^{1-\frac{2}{p}}}{C_p\norm{\vx-\vq}_{\mM}^{p-2}}+\max\inbraces{1, \inparen{\frac{\norm{\vx_{\vq}-\vq}_{\mM}}{\norm{\vx-\vq}_{\mM}}}^{p-2}}}\norm{\vx-\vx_{\vq}}_{\mM}\enspace,\\
    &\le^{\text{\eqref{eq:xq_x_UB_2}}} ep(p-1)\inparen{\frac{f(\vq)^{1-\frac{2}{p}}}{C_p\norm{\vx-\vq}_{\mM}^{p-2}}+\frac{1}{(1-\alpha)^{p-2}}}\norm{\vx-\vx_{\vq}}_{\mM}\enspace,\\
    &\le^{\eqref{eq:xq_x_UB_1}} \frac{ep(p-1)f(\vq)^{1-\frac{2}{p}}}{C_p\norm{\vx-\vq}_{\mM}^{p-2}} \cdot \norm{\vx-\vx_{\vq}}_{\mM} + \frac{ep(p-1)\alpha}{(1-\alpha)^{p-1}}\norm{\vx-\vq}_{\mM}\enspace,\\ 
    &\leq^{\eqref{eq:xq_x_UB_2},\ \eqref{eq:lemma_D19_guarantee}} \frac{1}{4}\norm{\vx-\vq}_{\mM} + \frac{ep(p-1)5^{p-2}}{4^{p-1}}\norm{\vx-\vq}_{\mM}\enspace,\\
    &\le \frac{1}{2}\norm{\vx-\vq}_{\mM}\enspace,
\end{align*}
giving us the approximation guarantee.

It remains to understand the complexity of solving the proximal subproblem to the accuracy required in \eqref{eq:lemma_D19_guarantee}. Plugging in $\gamma = \alpha\norm{\vx_{\vq}-\vq}_{\mM}$  into \Cref{lemma:prox_subproblem_base} and using our bound on $h_{\vq}(\vx_{\vq})$ from \Cref{lemma:initial_bregman_div} gives an iteration complexity of (ignoring the constant in front of the big-$O$)
\begin{align*}
    &\quad p^{O(1)}\logv{ph_{\vq}(\vx_{\vq})\inparen{\frac{2}{p\alpha\norm{\vx_{\vq}-\vq}_{\mM}}}^{p}} \\
    &\le p^{O(1)}\logv{p\inparen{p(p-1)f(\vq)^{1-\frac{2}{p}}\norm{\vx_{\vq}-\vq}_{\mM}^2 + C_p\norm{\vx_{\vq}-\vq}_{\mM}^p}\inparen{\frac{2}{p\alpha\norm{\vx_{\vq}-\vq}_{\mM}}}^{p}} \\
    &= p^{O(1)}\logv{\inparen{\frac{2}{p}}^{p} p\inparen{\frac{p(p-1)f(\vq)^{1-\frac{2}{p}}\norm{\vx_{\vq}-\vq}_{\mM}^2 + C_p\norm{\vx_{\vq}-\vq}_{\mM}^p}{\alpha^p\norm{\vx_{\vq}-\vq}_{\mM}^p}}} \\
    &= p^{O(1)}\logv{\inparen{\frac{2}{p}}^{p} p\inparen{\frac{p(p-1)f(\vq)^{1-\frac{2}{p}}}{\alpha^p\norm{\vx_{\vq}-\vq}_{\mM}^{p-2}} + \frac{C_p}{\alpha^p}}}
\end{align*}
We have two cases to analyze for the value of $\alpha$. In the first, suppose we get $\alpha = \frac{1}{5}$. By the definition of $\alpha$, this means we have
\begin{align*}
    \frac{C_p}{ep(p-1)}\inparen{\frac{\norm{\vx_{\vq}-\vq}_{\mM}}{f(\vq)^{\frac{1}{p}}}}^{p-2} \ge 1,
\end{align*}
which means the complexity we get is $p^{O(1)}\log p$. We now handle the other case, i.e., $\alpha=\frac{C_p}{5ep(p-1)}\inparen{\frac{\norm{\vx_{\vq}-\vq}_{\mM}}{f(\vq)^{\frac{1}{p}}}}^{p-2}$. Here, it will be useful to keep track of the timestep $t$ that we are working with. Recall that
\begin{align}
    \norm{\vx_{\vq_t}-\vq_t}_{\mM}^p = \inparen{\frac{\norm{\mM^{-1}\nabla f(\vx_{\vq_t})}_{\mM}}{pC_p}}^{\frac{p}{p-1}} \ge \inparen{\frac{\eps}{pC_p\norm{\vx_{\vq_t}-\xstar}_{\mM}}}^{\frac{p}{p-1}}\enspace,\label{eq:xq_lb}
\end{align}
so the complexity we want to control is given by
\begin{align*}
    &p^{O(1)}\logv{\inparen{\frac{2}{p}}^p p\inparen{\frac{2f(\vq_t)}{\alpha^p\norm{\vx_{\vq_t}-\vq_t}_{\mM}^{p}}}}\\ &\qquad\lesssim^{\eqref{eq:lemma_D19_guarantee}} p^{O(1)}\logv{\inparen{\frac{2}{p}}^p p\inparen{\frac{2\inparen{5ep(p-1)}^{p}f(\vq_t)^{p-1}}{C_p^p\norm{\vx_{\vq_t}-\vq_t}_{\mM}^{p(p-2)}\norm{\vx_{\vq_t}-\vq_t}_{\mM}^{p}}}}\enspace,\\
    &\qquad\lesssim p^{O(1)}\logv{ p\inparen{\frac{2\inparen{10(p-1)}^{p}f(\vq_t)^{p-1}}{p^{p^2}\norm{\vx_{\vq_t}-\vq_t}_{\mM}^{p(p-1)}}}}\enspace,\\
    &\qquad\lesssim^{\eqref{eq:xq_lb}} p^{O(1)}\logv{ p\inparen{\frac{2\inparen{10e(p-1)}^{p}p^{p(p+1)}f(\vq_t)^{p-1}}{p^{p^2}\epsilon^p}}\norm{\vx_{\vq_t}-\xstar}_{\mM}^p}\enspace,\\
    &\qquad\lesssim^{\eqref{eq:xq_lb}} p^{O(1)}\logv{\inparen{\frac{2\inparen{10e(p-1)}^{p}p^{p+1}f(\vq_t)^{p-1}}{\epsilon^p}}\norm{\vx_{\vq_t}-\xstar}_{\mM}^p}\enspace,\\
    &\qquad\lesssim p^{O(1)}\logv{\frac{pf(\vq_t)\norm{\vx_{\vq_t}-\xstar}_{\mM}}{\eps}}\enspace,\\
    &\qquad\lesssim^{\text{(\Cref{lemma:gp_regression_prox_diameter})}} p^{O(1)}\logv{\frac{pf(\vq_t)df(\vx_t)}{\eps}}\enspace,\\
    &\qquad\lesssim^{\text{(\Cref{lemma:fq_bounded})}} p^{O(1)}\logv{\frac{pf(\vx_t)}{\eps}}\enspace,
\end{align*}
completing the proof of \Cref{lemma:prox_subproblem_ms}.
\end{proof}

\subsection{The Algorithm}
\label{sec:gp_interpolating_alg}

We are now ready to combine the results from the previous two subsections to build our algorithm for $\cG_p$-regression and prove \Cref{mainthm:gp_regression_iteration_complexity}. The main algorithmic object here is \Cref{alg:gp_regression_alg}.

\begin{algorithm}[H]
\caption{\textsf{GpRegression}: Optimizes \eqref{eq:interpolation} up to $(1+\eps)$-multiplicative error}
\label{alg:gp_regression_alg}
\begin{algorithmic}[1]
\Require Regression problems $(\mA_{S_1},\vb_{S_1}),\dots,(\mA_{S_m},\vb_{S_m})$, accuracy $\eps > 0$
\State Using \cite[Algorithm 2]{mo23} with input $\insquare{\mA \vert \vb}$, find nonnegative diagonal $\mW$ such that for all $\vx\in\R^{d}$ and $c \in \R$,
\begin{align*}
    \gnorm{\mA\vx-c\vb}{\infty} \le \norm{\mW^{\frac{1}{2}-\frac{1}{p}}\mA\vx-c\mW^{1/2}\vb}_2 \le (2(d+1))^{\frac{1}{2}-\frac{1}{p}}\gnorm{\mA\vx-c\vb}{\infty}.
\end{align*}
\State Let $\vx_0 = \inparen{\mA^{\top}\mW^{1-\frac{2}{p}}\mA}^{-1}\mA^{\top}\mW^{1-\frac{2}{p}}\vb$. \Comment{$\vx_0 \coloneqq \argmin{\vx\in\R^d} \norm{\mW^{\frac{1}{2}-\frac{1}{p}}\mA\vx-\mW^{\frac{1}{2}-\frac{1}{p}}\vb}_2$.}
\State Using \Cref{alg:gp_prox_solver} and \Cref{lemma:prox_subproblem_ms}, implement a $\frac{1}{2}$-MS oracle for $f$ (\Cref{defn:ms_oracle})
\State Run \Cref{alg:optimal_ms_acceleration} with the oracle from the previous line and with $\vx_0$ as the initialization for $O\inparen{\mathsf{poly}(p)\min\inbraces{\rank{\mA},m}^{\frac{p-2}{3p-2}}\logv{\frac{d}{\eps}}^3}$ iterations.
\State \Return $\xhat$ the output of the previous step.
\end{algorithmic}
\end{algorithm}

\begin{proof}[Proof of \Cref{mainthm:gp_regression_iteration_complexity}]
By writing the stationary condition of the proximal problem, it makes sense to choose $\lambda_{t+1} = ep^{p+1}\norm{\xtilde_{t+1}-\vq_t}_{\mM}^{p-2}$.

It is easy to check that
\begin{align*}
    \norm{\xtilde_{t+1}-\vq_t}_{\mM} = \inparen{\frac{ep^{p+1}\norm{\xtilde_{t+1}-\vq_t}_{\mM}^{p-2}}{\inparen{(ep^{p+1})^{\frac{1}{p-1}}}^{p-1}}}^{\frac{1}{(p-1)-1}},
\end{align*}
and therefore the triple $(\xtilde_{t+1},\vq_t,ep^{p+1}\norm{\xtilde_{t+1}-\vq_t}_{\mM}^{p-2})$ always satisfies a $(p-1, (ep^{p+1})^{1/(p-1)})$-movement bound (\Cref{defn:movement_bound}).

Next, we calculate the iteration complexity we need to reduce the error to half of what we started with. For an arbitrary initial iterate $\vx$, let $\delta = 0.5(f(\vx)-f(\xstar))$. By \Cref{lemma:strong_convexity_gp}, we have
\begin{align*}
    \norm{\vx-\xstar}_{\mM}^{s+1} = \norm{\vx-\xstar}_{\mM}^p \le 2^{3p/2}d^{p/2-1}(f(\vx)-f(\xstar)),
\end{align*}
so combining this along with the fact that $c^s=ep^{p+1}$ and applying \Cref{thm:optimal_ms_acceleration} with our proximal solver \Cref{lemma:prox_subproblem_ms} yields
\begin{align*}
    T_{\mathsf{min}} &= \frac{p-1}{3}\inparen{pC_p\cdot2^{3p/2+1}d^{p/2-1}}^{\frac{2}{3p-2}} \lesssim p^{5/3}d^{\frac{p-2}{3p-2}}.
\end{align*}
Next, we initialize $\vx_0 \coloneqq \inparen{\mA^{\top}\mW^{1-2/p}\mA}^{-1}\mA^{\top}\mW^{1-2/p}\vb$. Using \Cref{thm:blw_to_l2} and \Cref{thm:gp_regression_blw_alg}, we have
\begin{align*}
    f(\vx_0) \le (2d)^{p/2-1}f(\xstar),
\end{align*}
so reaching an iterate $\vx$ for which $f(\vx)-f(\xstar) \le \eps f(\xstar)$ takes $T_{\mathsf{min}} \cdot \logv{d^{p/2-1}/\eps} = p^{8/3}d^{\frac{p-2}{3p-2}}\logv{\frac{d}{\eps}}$ calls to $\oprox$.

We now resolve the full iteration complexity, including the bootstrapping step to show that $f(\vx_t)$ is reasonably bounded so that we get an unconditional upper bound from \Cref{lemma:prox_subproblem_ms}. At the end of iteration $t$, from (loosely) inverting the bound in \Cref{thm:optimal_ms_acceleration}, we know that
\begin{align*}
    f(\vx_t)-f(\xstar)\le \frac{(Cp^3)^{\frac{3p-2}{2}}(2d)^{\frac{p}{2}-1}}{t^{\frac{3p-2}{2}}}.
\end{align*}
Since $\xtilde_{t+1}$ only depends on $\vq_t$, which in turn only depends on $\vx_t$ and $\vv_t$, it suffices to use the above bound for $f(\vx_t)$, which gives us an iteration complexity of $p^{O(1)}\logv{\frac{pd}{\eps}}$ to compute $\xtilde_{t+1}$ (which we get from plugging into \Cref{lemma:prox_subproblem_ms}).

Combining this with the iteration complexity of $\oprox$ gives us the result of \Cref{mainthm:gp_regression_iteration_complexity}.
\end{proof}

%% file: experiments.tex
\section{Empirical Evaluation}\label{app:exps}
In this section, we thoroughly compare our method against the baselines mentioned in \Cref{table:compare}. We begin with a synthetic regression experiment in \Cref{app:data_generation}, in which the synthetic design allows us to control the heterogeneity in the data. This is followed in \Cref{app:acs} by another experiment on a real-world group-structured regression task drawn from the American Community Survey (ACS). 

\subsection{Synthetic Heterogeneous Regression Construction}
\label{app:data_generation}

We construct synthetic group-structured regression problems designed to test optimization under severe group heterogeneity. The data consist of $m$ groups. Each group $i$ has its own design matrix $\mA_{S_i}$ and target vector $\vb_{S_i}$, and defines a quadratic loss
\[
\ell_i(x) = \frac{1}{n_i} \|\mA_{S_i} \vx - \vb_{S_i}\|_2^2 \enspace.
\]
The objective of interest is the worst-group loss (as in \eqref{eq:main_objective})
\[
F(x) = \max_{i \in [m]} \ell_i(\vx) \enspace.
\]

We generate two qualitatively distinct group types to separate average performance from worst-group performance. The majority of groups are benign and geometrically aligned, whereas a small number have very high curvature, are geometrically misaligned, and are far from the population center. This construction ensures that minimizing the average loss does not coincide with minimizing the worst-group loss, and that curvature heterogeneity strongly affects optimization behavior.

We construct all group covariances relative to a shared orthonormal coordinate system. This allows us to control curvature direction-by-direction while keeping the ambient geometry comparable across groups. Each group, therefore, has a quadratic loss whose Hessian shares eigenvectors with the others but whose eigenvalues vary across groups.

\paragraph{Normal groups.}
Most groups are generated with a moderate condition number. Their Hessians have eigenvalues that vary across coordinates but remain within a controlled range. The corresponding optimal parameters are sampled from a distribution concentrated around a common center in parameter space. Independent noise is added to each group so that these losses are smooth and moderately curved. As a result, normal groups are geometrically aligned: their curvature structure is similar, and their optima lie in a relatively small region of parameter space.

\paragraph{Outlier groups.}
A small subset of groups is constructed to be adversarial in two distinct ways. First, each adversarial group has one direction with extremely large curvature, while the remaining directions retain moderate curvature. These sharp directions differ across adversarial groups. Second, the optimal parameter of each adversarial group lies far from the population center along its corresponding high-curvature direction. In addition, the noise level in these groups is set to be very small, so their losses are sharply concentrated around their optima. Together, these properties ensure that deviations along the sharp directions incur very large increases in the worst-group loss.

\paragraph{Implications of our construction.}
This construction produces three key phenomena:

\begin{enumerate}
    \item The stacked design matrix has a large condition number.
    \item The curvature directions that dominate different groups are misaligned.
    \item The empirical risk minimizer (which minimizes the average loss) can perform well on most groups while incurring substantial loss on a small number of adversarial groups.
\end{enumerate}

Because most groups share similar geometry, gradient averaging implicitly emphasizes their curvature structure. Therefore, first-order methods, reduce the average loss efficiently but make comparatively slow progress along the sharp directions that control the worst-group objective. In contrast, methods that adapt to local curvature or explicitly control worst-case behavior can continue to decrease the max-loss objective.

For the default instance used in \Cref{fig:iteration_runtime}, the problem dimension is $d=10$ with $100$ groups, of which $5$ are adversarial. The stacked Gram matrix has a condition number on the order of $10^5$, and the empirical risk minimizer exhibits a clear gap relative to the robust optimum computed via convex programming. For more details on data generation, see the included Jupyter notebook.

\subsubsection{Computing the Robust Optimum via Convex Programming}
To obtain a reliable reference point for evaluation, we compute the exact optimum value of the worst-group objective using a convex solver. Concretely, we solve the robust regression problem that minimizes the maximum group mean-squared error. Although the objective is a pointwise maximum over groups, it remains a convex function of the parameter vector because each group loss is a convex quadratic.

We programmatically formulate this problem using the epigraph trick. We introduce an auxiliary scalar variable $t$ that upper-bounds every group loss, and then minimize this upper bound. If we write the group loss as the mean squared residual for that group, then the epigraph reformulation becomes
\[
\min_{x,\, t}\ \ t
\quad\text{subject to}\quad
\ell_i(\vx) \le t
\ \ \text{for every group } i\in[m]\enspace.
\]
Each constraint is a convex quadratic inequality, so the resulting problem is a convex quadratically constrained program. We implement this formulation in \texttt{CVXPY}~\citep{diamond2016cvxpy,agrawal2018rewriting} by declaring decision variables for the model parameters and the epigraph variable, adding one quadratic constraint per group, and calling a standard convex solver. We treat the returned epigraph value as the robust optimum value.

In all plots, we report the gap between an algorithm's current worst-group loss and this robust optimum value. This subtraction makes the figures directly comparable across instances and highlights whether a method continues to reduce the true robust suboptimality, rather than merely decreasing a surrogate objective.
\subsubsection{Baselines}

We compare the following methods.

\paragraph{Subgradient method.}
We run subgradient descent directly on the nonsmooth max-loss objective \ref{eq:main_objective}, using both fixed and diminishing step-size schedules, and report the best variant.

\paragraph{Smoothed gradient methods.}
We apply log-sum-exp smoothing (see \ref{eq:intro_smooth_fair_objective}) to approximate the max operator and optimize the resulting smooth objective using:
\begin{itemize}
    \item Gradient descent,
    \item Heavy-Ball (Polyak) momentum,
    \item Nesterov acceleration.
\end{itemize}

\paragraph{Interior-point method.}
We implement a log-barrier-based interior-point method that solves a sequence of smooth approximations using Newton steps.

\paragraph{Ball-oracle methods (ours).}
We implement two trust-region style methods that repeatedly solve the smoothed objective using a damped Newton solver (\Cref{sec:iter_prox_calls}):
\begin{itemize}
    \item Euclidean geometry (naive ball),
    \item Lewis-weight geometry, where the trust region is defined using a data-dependent positive definite matrix constructed from block Lewis weights.
\end{itemize}

After each outer step, the center is updated to the new solution, and the trust-region radius is optionally shrunk. For simplicity, we do not consider the acceleration of the ball-oracle method.

\subsubsection{Hyperparameter Tuning}

We tune every method via grid search over its relevant hyperparameters:

\begin{itemize}
    \item Step sizes for gradient and subgradient methods;
    \item Smoothing parameters and momentum coefficients for smoothed methods;
    \item Barrier parameters and inner iteration counts for the interior-point method;
    \item Initial trust-region radius, smoothing strength, and radius decay factors for the ball-oracle methods.
\end{itemize}

All methods use the same warm start. For each configuration, we run a fixed number of outer iterations and select the configuration that achieves the lowest worst-group loss within this budget.

\subsubsection{Empirical Behavior}

Notably, the meaning of an iteration differs across algorithms:

\begin{itemize}
    \item For subgradient and smoothed gradient methods, one iteration corresponds to one full gradient or subgradient update using all groups.
    \item For the interior-point method, one iteration corresponds to one outer Newton step of the barrier procedure.
    \item For the ball-oracle methods, one iteration corresponds to one call to the trust-region Newton solver (i.e., one outer iteration).
\end{itemize}

In all iteration-complexity plots, we compare methods using their own natural outer-iteration count.

\paragraph{Iteration complexity.}
On the adversarial instances, first-order methods make limited progress. Subgradient descent improves briefly but quickly plateaus. Even the best-tuned smoothed-gradient variant stalls far above the optimum. In contrast, both ball-oracle methods steadily decrease the worst-group loss across outer iterations, whereas the IPM converges rapidly. We also note that the IPM achieves the best final loss among all methods, which is unsurprising, as \texttt{CVXPY} natively uses the same algorithm to compute the maximum-loss optimum.

\paragraph{Runtime complexity.} While we do not spend significant effort simulating a time-complexity model or hyper-optimizing our code, we also plot the runtime complexity of all the algorithms to control for the different meanings of an iteration. Because Newton steps are computationally more expensive, first-order methods initially appear competitive in wall-clock time. However, they plateau early and fail to approach high accuracy. Interior-point and ball-oracle methods continue to reduce the worst-group loss gap and eventually achieve near-optimal solutions, whereas first-order methods remain stuck far from the optimum. We observe a very slight benefit from using the Lewis geometry in our ball oracle method. 

\begin{figure}[t]
    \centering
    \begin{subfigure}{0.49\textwidth}
        \centering
        \includegraphics[width=\linewidth]{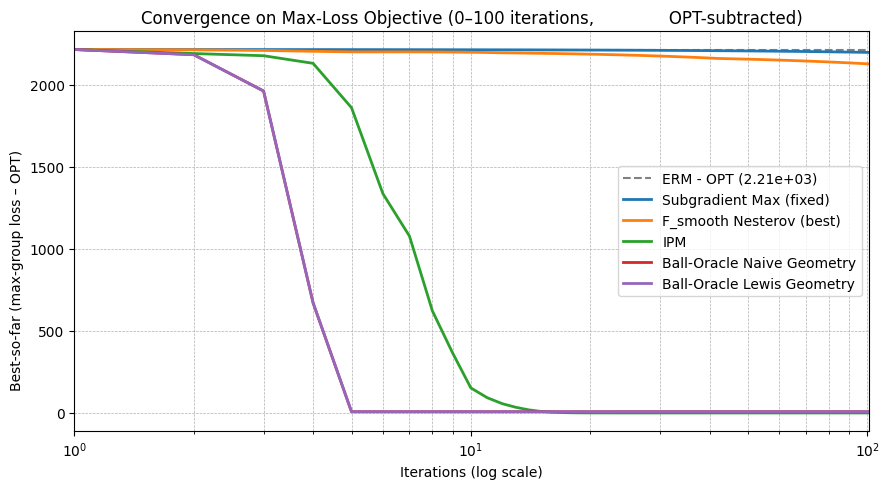}
        \caption{Suboptimality versus iteration count (log--log).}
    \end{subfigure}
    \hfill
    \begin{subfigure}{0.49\textwidth}
        \centering
        \includegraphics[width=\linewidth]{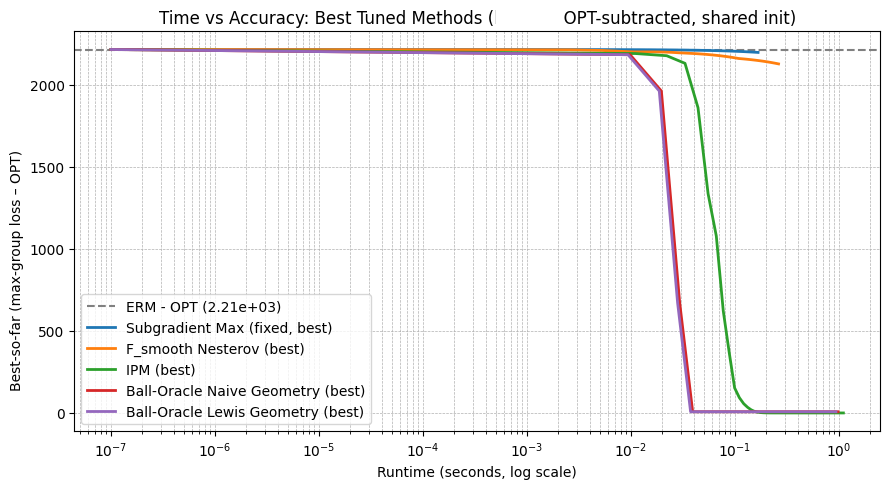}
        \caption{Suboptimality versus wall-clock time (seconds).}
    \end{subfigure}
    \caption{Comparison of first-order, interior-point, and ball-oracle methods on adversarial heterogeneous regression instances. All curves report the worst-group loss minus the optimal value computed via \texttt{CVXPY}.}
    \label{fig:iteration_runtime}
\end{figure}


\subsection{Real-world Experiment: ACS Income}
\label{app:acs}

To validate the computational behavior of our methods beyond synthetic instances, we evaluate them on a real-world group-structured regression task drawn from the American Community Survey (ACS), accessed through the \texttt{folktables} package~\citep{ding2021retiring}. The task is to predict log personal income from $d=10$ standardized demographic and employment features (age, education, occupation, hours worked, etc.) for employed US adults. We group the data by region of residence and use all $m=51$ regions (the $50$ states together with Puerto Rico), subsampling $200$ individuals per region for a total of $10{,}200$ samples. Each group loss $\ell_i$ is the mean squared prediction error on region $i$, and the objective is again the worst-group loss $F(\vx) = \max_{i\in[m]} \ell_i(\vx)$. This is a natural distributionally-robust setting: income distributions differ substantially across regions, so the model that minimizes the average error can perform poorly on individual states. All methods are warm-started at the empirical risk minimizer (ERM), and we compute the robust optimum $\opt$ with \texttt{CVXPY} as the reference value.

\paragraph{Convergence and runtime.} Our primary interest is computational: how quickly each method drives the worst-group suboptimality $F(\vx_t) - \opt$ to zero, both in outer iterations and in wall-clock time. \Cref{fig:acs_convergence} plots this gap against iteration count and against runtime, and \Cref{tab:acs_runtime} reports the iterations and wall-clock time each method needs to reach a relative suboptimality of $1\%$.

The two ball-oracle methods are the fastest on both axes. Starting from the warm start, each reaches the $1\%$ target in a \emph{single} outer iteration and in under $20$~ms of wall-clock time --- roughly $3\times$ faster than the interior-point method ($66$~ms, $8$ iterations) and the best-tuned smoothed Heavy-Ball method ($62$~ms, $47$ iterations). The interior-point method converges in a few iterations but incurs a higher per-iteration cost, whereas the smoothed first-order method requires an order-of-magnitude more iterations to reach the same accuracy. The subgradient method never reaches the $1\%$ target within the budget: it remains essentially pinned at the ERM gap (\Cref{fig:acs_convergence}). The Lewis-weight geometry is marginally faster than the Euclidean ball, consistent with the synthetic results. Overall, on a genuine $51$-group regression problem, the ball-oracle methods reach high-accuracy robust solutions fastest in both iteration count and wall-clock time.

\begin{figure}[t]
    \centering
    \begin{subfigure}{0.49\textwidth}
        \centering
        \includegraphics[width=\linewidth]{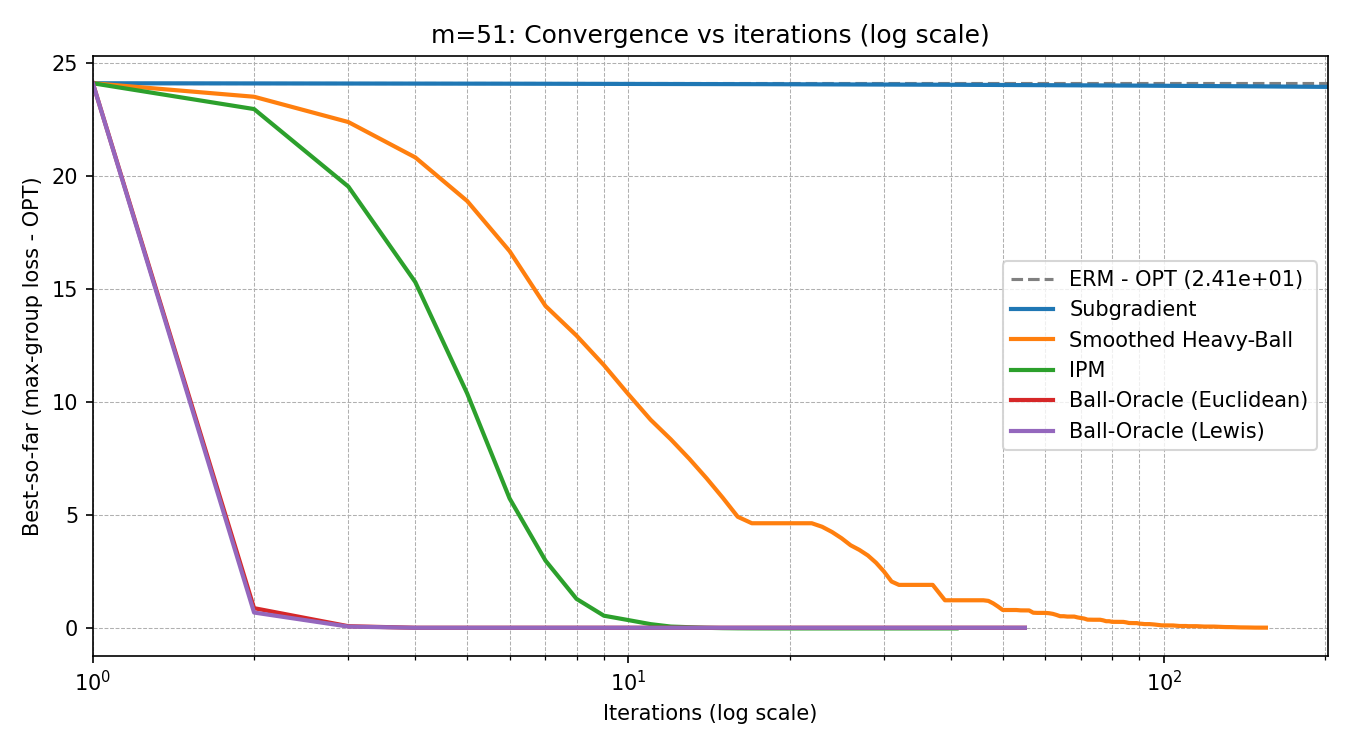}
        \caption{Suboptimality versus iteration count (log--log).}
    \end{subfigure}
    \hfill
    \begin{subfigure}{0.49\textwidth}
        \centering
        \includegraphics[width=\linewidth]{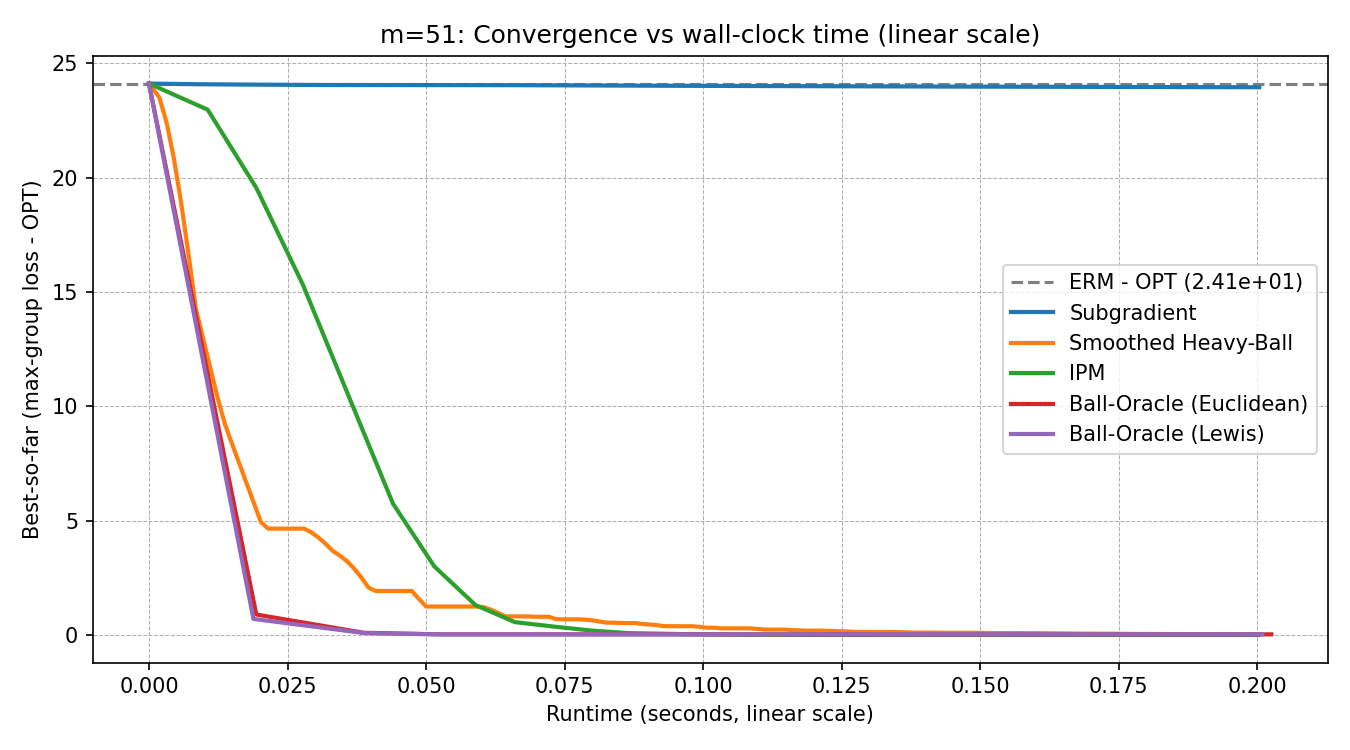}
        \caption{Suboptimality versus wall-clock time (seconds).}
    \end{subfigure}
    \caption{Convergence on the ACS Income task ($m=51$ regions, $d=10$, $10{,}200$ samples). Each curve reports the best-so-far worst-group suboptimality $F(\vx_t) - \opt$; the dashed line marks the ERM gap. The ball-oracle methods reach the robust optimum in a single outer iteration and the least wall-clock time, while the subgradient method stalls at the ERM gap.}
    \label{fig:acs_convergence}
\end{figure}

\begin{table}[t]
    \centering
    \small
    \begin{tabular}{lrr}
        \toprule
        Method & Iterations to $1\%$ gap & Wall-clock time (s) \\
        \midrule
        Subgradient              & --- (not reached) & --- \\
        Smoothed Heavy-Ball      & 47 & 0.062 \\
        IPM                      &  8 & 0.066 \\
        Ball-Oracle (Euclidean)  &  1 & 0.019 \\
        Ball-Oracle (Lewis)      &  1 & 0.019 \\
        \bottomrule
    \end{tabular}
    \caption{Cost to reach a $1\%$ relative worst-group suboptimality on the ACS Income task ($m=51$). Iterations are each method's own natural outer iterations. The subgradient method does not reach the target within the iteration budget.}
    \label{tab:acs_runtime}
\end{table}

\paragraph{Statistical context.} For completeness, we note that the optimization gains correspond to meaningful redistribution of prediction accuracy across regions. The ERM solution attains the lowest \emph{average} error ($108.2$) but a large spread ($\sigma = 11.8$) and a worst-group loss of $138.1$, attained by California; other heavily penalized states are New York, Hawaii, Nevada, Florida, and New Jersey, all large or high-income regions on which a $10$-feature linear model underfits. The robust optimum instead nearly equalizes the group losses into a narrow band around $107$--$114$ (Max/Mean drops from $1.28$ to $1.02$), decreasing the loss on the previously neglected states (e.g.\ California by $24.3$~MSE) at a moderate cost on the homogeneous low-income states that ERM happened to fit well.

%% file: ref.bib
@book{nw06,
  title={Numerical optimization},
  author={Nocedal, Jorge and Wright, Stephen J},
  year={2006},
  publisher={Springer}
}

@InProceedings{cpo20,
  title = 	 {Distributionally Robust Formulation and Model Selection for the Graphical Lasso},
  author =       {Cisneros-Velarde, Pedro and Petersen, Alexander and Oh, Sang-Yun},
  booktitle = 	 {Proceedings of the Twenty Third International Conference on Artificial Intelligence and Statistics},
  pages = 	 {756--765},
  year = 	 {2020},
  editor = 	 {Chiappa, Silvia and Calandra, Roberto},
  volume = 	 {108},
  series = 	 {Proceedings of Machine Learning Research},
  month = 	 {08},
  publisher =    {PMLR},
  eprint={1905.08975},
  archivePrefix={arXiv},
  primaryClass={stat.ML},
}

@article{bkm19,
  title={Robust Wasserstein profile inference and applications to machine learning},
  author={Blanchet, Jose and Kang, Yang and Murthy, Karthyek},
  journal={Journal of Applied Probability},
  volume={56},
  number={3},
  pages={830--857},
  year={2019},
  eprint={1610.05627},
  archivePrefix={arXiv},
  primaryClass={math.ST},
  publisher={Cambridge University Press}
}

@misc{ksj18,
      title={Global linear convergence of Newton's method without strong-convexity or Lipschitz gradients}, 
      author={Sai Praneeth Karimireddy and Sebastian U. Stich and Martin Jaggi},
      year={2018},
      eprint={1806.00413},
      archivePrefix={arXiv},
      primaryClass={cs.LG},
      url={https://arxiv.org/abs/1806.00413}, 
}

@InProceedings{syls16,
  title = 	 {AdaDelay: Delay Adaptive Distributed Stochastic Optimization},
  author = 	 {Sra, Suvrit and Yu, Adams Wei and Li, Mu and Smola, Alex},
  booktitle = 	 {Proceedings of the 19th International Conference on Artificial Intelligence and Statistics},
  pages = 	 {957--965},
  year = 	 {2016},
  editor = 	 {Gretton, Arthur and Robert, Christian C.},
  volume = 	 {51},
  series = 	 {Proceedings of Machine Learning Research},
  address = 	 {Cadiz, Spain},
  month = 	 {05},
  publisher =    {PMLR},
eprint={1508.05003},
      archivePrefix={arXiv},
      primaryClass={stat.ML},
}

@misc{ajk24,
      title={Acceleration Meets Inverse Maintenance: Faster $\ell_{\infty}$-Regression}, 
      author={Deeksha Adil and Shunhua Jiang and Rasmus Kyng},
      year={2024},
      eprint={2409.20030},
      archivePrefix={arXiv},
      primaryClass={cs.DS},
      url={https://arxiv.org/abs/2409.20030}, 
}

@inproceedings{bcll18,
author = {Bubeck, S\'{e}bastien and Cohen, Michael B. and Lee, Yin Tat and Li, Yuanzhi},
title = {An homotopy method for lp regression provably beyond self-concordance and in input-sparsity time},
year = {2018},
isbn = {9781450355599},
publisher = {Association for Computing Machinery},
address = {New York, NY, USA},
booktitle = {Proceedings of the 50th Annual ACM SIGACT Symposium on Theory of Computing},
pages = {1130–1137},
numpages = {8},
location = {Los Angeles, CA, USA},
series = {STOC 2018},
eprint={1711.01328},
      archivePrefix={arXiv},
      primaryClass={math.OC},
}

@incollection{john1948,
	title        = {Extremum problems with inequalities as subsidiary conditions},
	author       = {John, Fritz},
	year         = 1948,
	booktitle    = {Studies and Essays Presented to R. Courant on his 60th Birthday},
	publisher    = {Interscience Publishers, Inc},
	pages        = {187--204}
}

@book{nn94,
  title={Interior Point Polynomial Algorithms in Convex Programming},
  author={Nesterov, Y. and Nemirovskii, A.},
  isbn={9780898715156},
  lccn={93038912},
  series={SIAM studies in applied and numerical mathematics: Society for Industrial and Applied Mathematics},
  year={1994},
  publisher={Society for Industrial and Applied Mathematics}
}

@inproceedings{chjjs22,
author = {Carmon, Yair and Hausler, Danielle and Jambulapati, Arun and Jin, Yujia and Sidford, Aaron},
title = {Optimal and adaptive monteiro-svaiter acceleration},
year = {2022},
isbn = {9781713871088},
publisher = {Curran Associates Inc.},
address = {Red Hook, NY, USA},
booktitle = {Proceedings of the 36th International Conference on Neural Information Processing Systems},
articleno = {1479},
numpages = {13},
location = {New Orleans, LA, USA},
series = {NIPS '22},
eprint={2205.15371},
archivePrefix={arXiv},
primaryClass={math.OC},
}

@inbook{bjlls19,
author = {Bubeck, S\'{e}bastien and Jiang, Qijia and Lee, Yin Tat and Li, Yuanzhi and Sidford, Aaron},
title = {Complexity of highly parallel non-smooth convex optimization},
year = {2019},
publisher = {Curran Associates Inc.},
address = {Red Hook, NY, USA},
booktitle = {Proceedings of the 33rd International Conference on Neural Information Processing Systems},
articleno = {1247},
numpages = {10},
eprint={1906.10655},
archivePrefix={arXiv},
primaryClass={math.OC},
}

@inbook{akps19,
author = {Deeksha Adil and Rasmus Kyng and Richard Peng and Sushant Sachdeva},
title = {Iterative Refinement for {$\ell_p$}-norm Regression},
booktitle = {Proceedings of the 2019 Annual ACM-SIAM Symposium on Discrete Algorithms (SODA)},
chapter = {},
pages = {1405-1424},
year = {2019},
doi = {10.1137/1.9781611975482.86},
URL = {https://epubs.siam.org/doi/abs/10.1137/1.9781611975482.86},
eprint = {https://epubs.siam.org/doi/pdf/10.1137/1.9781611975482.86},
    eprint={1901.06764},
    archivePrefix={arXiv},
    primaryClass={cs.DS},
}

@article{lfn18,
  title={Relatively smooth convex optimization by first-order methods, and applications},
  author={Lu, Haihao and Freund, Robert M and Nesterov, Yurii},
  journal={SIAM Journal on Optimization},
  volume={28},
  number={1},
  pages={333--354},
  year={2018},
  publisher={SIAM},
eprint={1610.05708},
      archivePrefix={arXiv},
      primaryClass={math.OC},
}

@article{berk2021fairness,
  title={Fairness in criminal justice risk assessments: The state of the art},
  author={Berk, Richard and Heidari, Hoda and Jabbari, Shahin and Kearns, Michael and Roth, Aaron},
  journal={Sociological Methods \& Research},
  volume={50},
  number={1},
  pages={3--44},
  year={2021},
  publisher={Sage Publications Sage CA: Los Angeles, CA}
}

@inproceedings{selbst2019fairness,
  title={Fairness and abstraction in sociotechnical systems},
  author={Selbst, Andrew D and Boyd, Danah and Friedler, Sorelle A and Venkatasubramanian, Suresh and Vertesi, Janet},
  booktitle={Proceedings of the conference on fairness, accountability, and transparency},
  pages={59--68},
  year={2019}
}

@book{conn2000trust,
  title={Trust region methods},
  author={Conn, Andrew R and Gould, Nicholas IM and Toint, Philippe L},
  year={2000},
  publisher={SIAM}
}

@inproceedings{veale2018fairness,
  title={Fairness and accountability design needs for algorithmic support in high-stakes public sector decision-making},
  author={Veale, Michael and Van Kleek, Max and Binns, Reuben},
  booktitle={Proceedings of the 2018 chi conference on human factors in computing systems},
  pages={1--14},
  year={2018}
}

@article{data2014seizing,
  title={Seizing Opportunities, Preserving Values},
  author={Data, Big},
  journal={The White House Report Washington},
  year={2014}
}

@misc{chouldechova2016fair,
  title={Fair prediction with disparate impact: A study of bias in recidivism prediction instruments. Big Data, 5 (2), 153-163},
  author={Chouldechova, Alexandra},
  year={2016}
}

@inproceedings{agarwal2019fair,
  title={Fair regression: Quantitative definitions and reduction-based algorithms},
  author={Agarwal, Alekh and Dud{\'\i}k, Miroslav and Wu, Zhiwei Steven},
  booktitle={International Conference on Machine Learning},
  pages={120--129},
  year={2019},
  organization={PMLR}
}

@article{barocas2016big,
  title={Big data's disparate impact},
  author={Barocas, Solon and Selbst, Andrew D},
  journal={Calif. L. Rev.},
  volume={104},
  pages={671},
  year={2016},
  publisher={HeinOnline}
}

@article{corbett2023measure,
  title={The measure and mismeasure of fairness},
  author={Corbett-Davies, Sam and Gaebler, Johann D and Nilforoshan, Hamed and Shroff, Ravi and Goel, Sharad},
  journal={The Journal of Machine Learning Research},
  volume={24},
  number={1},
  pages={14730--14846},
  year={2023},
  publisher={JMLRORG}
}

@article{cohen2024fairness,
  title={Fairness of linear regression in decision making},
  author={Cohen-Addad, Vincent and Gavva, Surya Teja and Karthik, CS and Mathieu, Claire and Namrata},
  journal={International journal of data science and analytics},
  volume={18},
  number={3},
  pages={337--347},
  year={2024},
  publisher={Springer}
}

@article{duchi2016statistics,
  title={Statistics of robust optimization: A generalized empirical likelihood approach. arXiv},
  author={Duchi, J and Glynn, P and Namkoong, Hongseok},
  journal={Machine Learning},
  year={2016}
}

@article{chen2022fair,
  title={Fair assortment planning},
  author={Chen, Qinyi and Golrezaei, Negin and Susan, Fransisca and Baskoro, Edy},
  journal={arXiv preprint arXiv:2208.07341},
  year={2022}
}

@article{zhang2024stochastic,
  title={Stochastic approximation approaches to group distributionally robust optimization},
  author={Zhang, Lijun and Zhao, Peng and Zhuang, Zhen-Hua and Yang, Tianbao and Zhou, Zhi-Hua},
  journal={Advances in Neural Information Processing Systems},
  volume={36},
  year={2024}
}

@inproceedings{kleinberg2018algorithmic,
  title={Algorithmic fairness},
  author={Kleinberg, Jon and Ludwig, Jens and Mullainathan, Sendhil and Rambachan, Ashesh},
  booktitle={Aea papers and proceedings},
  volume={108},
  pages={22--27},
  year={2018}
}

@article{soma2022optimal,
  title={Optimal algorithms for group distributionally robust optimization and beyond},
  author={Soma, Tasuku and Gatmiry, Khashayar and Jegelka, Stefanie},
  journal={arXiv preprint arXiv:2212.13669},
  year={2022}
}

@article{golrezaei2024online,
  title={Online Combinatorial Optimization with Group Fairness Constraints},
  author={Golrezaei, Negin and Niazadeh, Rad and Patel, Kumar Kshitij and Susan, Fransisca},
  journal={Available at SSRN 4824251},
  year={2024}
}

@InProceedings{aakmrz22,
  title = 	 {Active Sampling for Min-Max Fairness},
  author =       {Abernethy, Jacob D and Awasthi, Pranjal and Kleindessner, Matth{\"a}us and Morgenstern, Jamie and Russell, Chris and Zhang, Jie},
  booktitle = 	 {Proceedings of the 39th International Conference on Machine Learning},
  pages = 	 {53--65},
  year = 	 {2022},
  editor = 	 {Chaudhuri, Kamalika and Jegelka, Stefanie and Song, Le and Szepesvari, Csaba and Niu, Gang and Sabato, Sivan},
  volume = 	 {162},
  series = 	 {Proceedings of Machine Learning Research},
  month = 	 {07},
  publisher =    {PMLR},
  url = 	 {https://proceedings.mlr.press/v162/abernethy22a.html}
}

@MISC {var_stackexchange,
    TITLE = {Bound the variance of the product of two random varables.},
    AUTHOR = {Davide Giraudo},
    HOWPUBLISHED = {Mathematics Stack Exchange},
    URL = {https://math.stackexchange.com/q/1044864},
    YEAR={2014},
    MONTH={11},
}

@article{ms13,
author = {Monteiro, Renato D. C. and Svaiter, B. F.},
title = {An Accelerated Hybrid Proximal Extragradient Method for Convex Optimization and Its Implications to Second-Order Methods},
journal = {SIAM Journal on Optimization},
volume = {23},
number = {2},
pages = {1092-1125},
year = {2013},
doi = {10.1137/110833786},
URL = { 
        https://doi.org/10.1137/110833786
},
eprint = { 
        https://doi.org/10.1137/110833786
}
}

@misc{
svwz24,
title={On Socially Fair Regression and Low-Rank Approximation},
author={Zhao Song and Ali Vakilian and David Woodruff and Samson Zhou},
year={2024},
url={https://openreview.net/forum?id=KJHUYWviZ6}
}

@misc{ob18,
      title={Finite-sample analysis of M-estimators using self-concordance}, 
      author={Dmitrii Ostrovskii and Francis Bach},
      year={2020},
      eprint={1810.06838},
      archivePrefix={arXiv},
      primaryClass={math.ST},
      url={https://arxiv.org/abs/1810.06838}, 
}

@inproceedings{msbacon,
author = {Carmon, Yair and Jambulapati, Arun and Jiang, Qijia and Jin, Yujia and Lee, Yin Tat and Sidford, Aaron and Tian, Kevin},
title = {Acceleration with a ball optimization oracle},
year = {2020},
isbn = {9781713829546},
publisher = {Curran Associates Inc.},
address = {Red Hook, NY, USA},
booktitle = {Proceedings of the 34th International Conference on Neural Information Processing Systems},
articleno = {1599},
numpages = {12},
series = {NIPS '20},
eprint={2003.08078},
      archivePrefix={arXiv},
        primaryClass={math.OC}
}

@inbook{mo23,
      title={The Change-of-Measure Method, Block Lewis Weights, and Approximating Matrix Block Norms}, 
      author={Naren Sarayu Manoj and Max Ovsiankin},
      year={2025},
      eprint={2311.10013},
      archivePrefix={arXiv},
      primaryClass={math.FA},
      booktitle = {Proceedings of the 2025 Annual ACM-SIAM Symposium on Discrete Algorithms (SODA)}
}

@inproceedings{jlls23,
  title={Sparsifying sums of norms},
  author={Jambulapati, Arun and Lee, James R and Liu, Yang P and Sidford, Aaron},
  booktitle={2023 IEEE 64th Annual Symposium on Foundations of Computer Science (FOCS)},
  pages={1953--1962},
  year={2023},
  organization={IEEE},
  eprint={2305.09049},
  archivePrefix={arXiv},
  primaryClass={cs.DS}
}

@article{blm89,
  title={Approximation of zonoids by zonotopes},
  author={Bourgain, Jean and Lindenstrauss, Joram and Milman, Vitali},
  year={1989}
}

@inproceedings{mmwy21,
  title={Active Linear Regression for {$\ell_p$} Norms and Beyond}, 
  author={Musco, Cameron and Musco, Christopher and Woodruff, David P and Yasuda, Taisuke},
  booktitle={2022 IEEE 63rd Annual Symposium on Foundations of Computer Science (FOCS)},
  pages={744--753},
  year={2022},
  organization={IEEE},
  eprint={2111.04888},
  archivePrefix={arXiv},
  primaryClass={cs.LG}
}

@misc{ls19,
      title={Solving Linear Programs with Sqrt(rank) Linear System Solves}, 
      author={Yin Tat Lee and Aaron Sidford},
      year={2019},
      eprint={1910.08033},
      archivePrefix={arXiv},
      primaryClass={cs.DS}
}

@inproceedings{jls21,
  title={Improved iteration complexities for overconstrained p-norm regression},
  author={Jambulapati, Arun and Liu, Yang P and Sidford, Aaron},
  booktitle={Proceedings of the 54th Annual ACM SIGACT Symposium on Theory of Computing},
  pages={529--542},
  year={2022},
  eprint={2111.01848},
  archivePrefix={arXiv},
  primaryClass={cs.DS}
}

@inproceedings{jls22,
  title={Chaining, group leverage score overestimates, and fast spectral hypergraph sparsification},
  author={Jambulapati, Arun and Liu, Yang P and Sidford, Aaron},
  booktitle={Proceedings of the 55th Annual ACM Symposium on Theory of Computing},
  pages={196--206},
  year={2023},
  eprint={2209.10539},
  archivePrefix={arXiv},
  primaryClass={cs.DS}
}

@article{chen2018tight,
  title={Tight bounds for collaborative pac learning via multiplicative weights},
  author={Chen, Jiecao and Zhang, Qin and Zhou, Yuan},
  journal={Advances in neural information processing systems},
  volume={31},
  year={2018}
}

@article{nguyen2018improved,
  title={Improved algorithms for collaborative PAC learning},
  author={Nguyen, Huy and Zakynthinou, Lydia},
  journal={Advances in Neural Information Processing Systems},
  volume={31},
  year={2018}
}

@inproceedings{zhang2024optimal,
  title={Optimal multi-distribution learning},
  author={Zhang, Zihan and Zhan, Wenhao and Chen, Yuxin and Du, Simon S and Lee, Jason D},
  booktitle={The Thirty Seventh Annual Conference on Learning Theory},
  pages={5220--5223},
  year={2024},
  organization={PMLR}
}

@article{haghtalab2022demand,
  title={On-demand sampling: Learning optimally from multiple distributions},
  author={Haghtalab, Nika and Jordan, Michael and Zhao, Eric},
  journal={Advances in Neural Information Processing Systems},
  volume={35},
  pages={406--419},
  year={2022}
}

@inproceedings{mohri2019agnostic,
  title={Agnostic federated learning},
  author={Mohri, Mehryar and Sivek, Gary and Suresh, Ananda Theertha},
  booktitle={International conference on machine learning},
  pages={4615--4625},
  year={2019},
  organization={PMLR}
}

@article{hanneke2019value,
  title={On the value of target data in transfer learning},
  author={Hanneke, Steve and Kpotufe, Samory},
  journal={Advances in Neural Information Processing Systems},
  volume={32},
  year={2019}
}

@article{bullins2021stochastic,
  title={A stochastic newton algorithm for distributed convex optimization},
  author={Bullins, Brian and Patel, Kshitij and Shamir, Ohad and Srebro, Nathan and Woodworth, Blake E},
  journal={Advances in Neural Information Processing Systems},
  volume={34},
  pages={26818--26830},
  year={2021}
}

@article{mcmahan2016s,
  title={S. Hampsonet al.,“Communication-efficient learning of deep networks from decentralizeddata,”},
  author={McMahan, HB and Moore, E and Ramage, D},
  journal={arXiv preprint arXiv:1602.05629},
  year={2016}
}

@inproceedings{woodworth2020local,
  title={Is local SGD better than minibatch SGD?},
  author={Woodworth, Blake and Patel, Kumar Kshitij and Stich, Sebastian and Dai, Zhen and Bullins, Brian and Mcmahan, Brendan and Shamir, Ohad and Srebro, Nathan},
  booktitle={International Conference on Machine Learning},
  pages={10334--10343},
  year={2020},
  organization={PMLR}
}

@article{patel2024limits,
  title={The limits and potentials of local sgd for distributed heterogeneous learning with intermittent communication},
  author={Patel, Kumar Kshitij and Glasgow, Margalit and Zindari, Ali and Wang, Lingxiao and Stich, Sebastian U and Cheng, Ziheng and Joshi, Nirmit and Srebro, Nathan},
  journal={arXiv preprint arXiv:2405.11667},
  year={2024}
}

@article{valiant1984theory,
  title={A theory of the learnable},
  author={Valiant, Leslie G},
  journal={Communications of the ACM},
  volume={27},
  number={11},
  pages={1134--1142},
  year={1984},
  publisher={ACM New York, NY, USA}
}

@book{vapnik2013nature,
  title={The nature of statistical learning theory},
  author={Vapnik, Vladimir},
  year={2013},
  publisher={Springer science \& business media}
}

@article{bertsimas2011price,
  title={The price of fairness},
  author={Bertsimas, Dimitris and Farias, Vivek F and Trichakis, Nikolaos},
  journal={Operations research},
  volume={59},
  number={1},
  pages={17--31},
  year={2011},
  publisher={INFORMS}
}

@article{asadpour2022sequential,
  title={Sequential Submodular Maximization and Applications to Ranking an Assortment of Products},
  author={Asadpour, Arash and Niazadeh, Rad and Saberi, Amin and Shameli, Ali},
  journal={Operations Research},
  year={2022},
  publisher={INFORMS}
}

@article{chouldechova2018frontiers,
  title={The frontiers of fairness in machine learning},
  author={Chouldechova, Alexandra and Roth, Aaron},
  journal={arXiv preprint arXiv:1810.08810},
  year={2018}
}

@article{ben2013robust,
  title={Robust solutions of optimization problems affected by uncertain probabilities},
  author={Ben-Tal, Aharon and Den Hertog, Dick and De Waegenaere, Anja and Melenberg, Bertrand and Rennen, Gijs},
  journal={Management Science},
  volume={59},
  number={2},
  pages={341--357},
  year={2013},
  publisher={INFORMS}
}

@article{blum2017collaborative,
  title={Collaborative PAC learning},
  author={Blum, Avrim and Haghtalab, Nika and Procaccia, Ariel D and Qiao, Mingda},
  journal={Advances in Neural Information Processing Systems},
  volume={30},
  year={2017}
}

@article{levy2020large,
  title={Large-scale methods for distributionally robust optimization},
  author={Levy, Daniel and Carmon, Yair and Duchi, John C and Sidford, Aaron},
  journal={Advances in Neural Information Processing Systems},
  volume={33},
  pages={8847--8860},
  year={2020}
}

@article{sagawa2019distributionally,
  title={Distributionally robust neural networks for group shifts: On the importance of regularization for worst-case generalization},
  author={Sagawa, Shiori and Koh, Pang Wei and Hashimoto, Tatsunori B and Liang, Percy},
  journal={arXiv preprint arXiv:1911.08731},
  year={2019}
}

@book{boyd2004convex,
  title={Convex optimization},
  author={Boyd, Stephen P and Vandenberghe, Lieven},
  year={2004},
  publisher={Cambridge university press}
}

@article{rahmattalabi2019exploring,
  title={Exploring algorithmic fairness in robust graph covering problems},
  author={Rahmattalabi, Aida and Vayanos, Phebe and Fulginiti, Anthony and Rice, Eric and Wilder, Bryan and Yadav, Amulya and Tambe, Milind},
  journal={Advances in neural information processing systems},
  volume={32},
  year={2019}
}

@inproceedings{chierichetti2019matroids,
  title={Matroids, matchings, and fairness},
  author={Chierichetti, Flavio and Kumar, Ravi and Lattanzi, Silvio and Vassilvtiskii, Sergei},
  booktitle={The 22nd International Conference on Artificial Intelligence and Statistics},
  pages={2212--2220},
  year={2019},
  organization={PMLR}
}

@inproceedings{dwork2012fairness,
  title={Fairness through awareness},
  author={Dwork, Cynthia and Hardt, Moritz and Pitassi, Toniann and Reingold, Omer and Zemel, Richard},
  booktitle={Proceedings of the 3rd innovations in theoretical computer science conference},
  pages={214--226},
  year={2012}
}

@article{loi2019philosophical,
  title={A Philosophical Theory of Fairness for Prediction-Based Decisions.},
  author={Loi, Michele and Herlitz, Anders and Heidari, Hoda},
  journal={Available at SSRN 3450300},
  year={2019}
}

@article{hardt2016equality,
  title={Equality of opportunity in supervised learning},
  author={Hardt, Moritz and Price, Eric and Srebro, Nati},
  journal={Advances in neural information processing systems},
  volume={29},
  year={2016}
}

@inproceedings{kearns2018preventing,
  title={Preventing fairness gerrymandering: Auditing and learning for subgroup fairness},
  author={Kearns, Michael and Neel, Seth and Roth, Aaron and Wu, Zhiwei Steven},
  booktitle={International conference on machine learning},
  pages={2564--2572},
  year={2018},
  organization={PMLR}
}

@inproceedings{kearns2019empirical,
  title={An empirical study of rich subgroup fairness for machine learning},
  author={Kearns, Michael and Neel, Seth and Roth, Aaron and Wu, Zhiwei Steven},
  booktitle={Proceedings of the conference on fairness, accountability, and transparency},
  pages={100--109},
  year={2019}
}

@inproceedings{singh2018fairness,
  title={Fairness of exposure in rankings},
  author={Singh, Ashudeep and Joachims, Thorsten},
  booktitle={Proceedings of the 24th ACM SIGKDD International Conference on Knowledge Discovery \& Data Mining},
  pages={2219--2228},
  year={2018}
}

@article{singh2019policy,
  title={Policy learning for fairness in ranking},
  author={Singh, Ashudeep and Joachims, Thorsten},
  journal={Advances in neural information processing systems},
  volume={32},
  year={2019}
}

@inproceedings{biega2018equity,
  title={Equity of attention: Amortizing individual fairness in rankings},
  author={Biega, Asia J and Gummadi, Krishna P and Weikum, Gerhard},
  booktitle={The 41st international acm sigir conference on research \& development in information retrieval},
  pages={405--414},
  year={2018}
}

@article{bertsimas2012efficiency,
  title={On the efficiency-fairness trade-off},
  author={Bertsimas, Dimitris and Farias, Vivek F and Trichakis, Nikolaos},
  journal={Management Science},
  volume={58},
  number={12},
  pages={2234--2250},
  year={2012},
  publisher={INFORMS}
}

@inproceedings{manshadi2021fair,
  title={Fair dynamic rationing},
  author={Manshadi, Vahideh and Niazadeh, Rad and Rodilitz, Scott},
  booktitle={Proceedings of the 22nd ACM Conference on Economics and Computation},
  pages={694--695},
  year={2021}
}

@inproceedings{balseiro2021regularized,
  title={Regularized online allocation problems: Fairness and beyond},
  author={Balseiro, Santiago and Lu, Haihao and Mirrokni, Vahab},
  booktitle={International Conference on Machine Learning},
  pages={630--639},
  year={2021},
  organization={PMLR}
}

@book{miettinen1999nonlinear,
  title={Nonlinear multiobjective optimization},
  author={Miettinen, Kaisa},
  volume={12},
  year={1999},
  publisher={Springer Science \& Business Media}
}

@article{hooker2012combining,
  title={Combining equity and utilitarianism in a mathematical programming model},
  author={Hooker, John N and Williams, H Paul},
  journal={Management Science},
  volume={58},
  number={9},
  pages={1682--1693},
  year={2012},
  publisher={INFORMS}
}

@article{mulvany2021fair,
  title={Fair scheduling of heterogeneous customer populations},
  author={Mulvany, Justin and Randhawa, Ramandeep S},
  journal={Available at SSRN 3803016},
  year={2021}
}

@article{kusner2017counterfactual,
  title={Counterfactual fairness},
  author={Kusner, Matt J and Loftus, Joshua and Russell, Chris and Silva, Ricardo},
  journal={Advances in neural information processing systems},
  volume={30},
  year={2017}
}

@inproceedings{ustun2019fairness,
  title={Fairness without harm: Decoupled classifiers with preference guarantees},
  author={Ustun, Berk and Liu, Yang and Parkes, David},
  booktitle={International Conference on Machine Learning},
  pages={6373--6382},
  year={2019},
  organization={PMLR}
}

@inproceedings{donahue2020fairness,
  title={Fairness and utilization in allocating resources with uncertain demand},
  author={Donahue, Kate and Kleinberg, Jon},
  booktitle={Proceedings of the 2020 conference on fairness, accountability, and transparency},
  pages={658--668},
  year={2020}
}

@book{ehrgott2005multicriteria,
  title={Multicriteria optimization},
  author={Ehrgott, Matthias},
  volume={491},
  year={2005},
  publisher={Springer Science \& Business Media}
}

@inproceedings{goel2018non,
  title={Non-discriminatory machine learning through convex fairness criteria},
  author={Goel, Naman and Yaghini, Mohammad and Faltings, Boi},
  booktitle={Proceedings of the 2018 AAAI/ACM Conference on AI, Ethics, and Society},
  pages={116--116},
  year={2018}
}

@article{gupta2022socially,
  title={Socially fair and hierarchical facility location problems},
  author={Gupta, Swati and Moondra, Jai and Singh, Mohit},
  journal={arXiv preprint arXiv:2211.14873},
  year={2022}
}

@article{ma2023fairness,
  title={Fairness maximization among offline agents in online-matching markets},
  author={Ma, Will and Xu, Pan and Xu, Yifan},
  journal={ACM Transactions on Economics and Computation},
  volume={10},
  number={4},
  pages={1--27},
  year={2023},
  publisher={ACM New York, NY}
}

@inproceedings{calders2009building,
  title={Building classifiers with independency constraints},
  author={Calders, Toon and Kamiran, Faisal and Pechenizkiy, Mykola},
  booktitle={2009 IEEE international conference on data mining workshops},
  pages={13--18},
  year={2009},
  organization={IEEE}
}

@inproceedings{kasy2021fairness,
  title={Fairness, equality, and power in algorithmic decision-making},
  author={Kasy, Maximilian and Abebe, Rediet},
  booktitle={Proceedings of the 2021 ACM Conference on Fairness, Accountability, and Transparency},
  pages={576--586},
  year={2021}
}

@inproceedings{abebe2020roles,
  title={Roles for computing in social change},
  author={Abebe, Rediet and Barocas, Solon and Kleinberg, Jon and Levy, Karen and Raghavan, Manish and Robinson, David G},
  booktitle={Proceedings of the 2020 conference on fairness, accountability, and transparency},
  pages={252--260},
  year={2020}
}

@article{ding2021retiring,
  title={Retiring Adult: New Datasets for Fair Machine Learning},
  author={Ding, Frances and Hardt, Moritz and Miller, John and Schmidt, Ludwig},
  journal={Advances in Neural Information Processing Systems},
  volume={34},
  pages={6478--6490},
  year={2021}
}

@article{diamond2016cvxpy,
  title={{CVXPY}: A {Python}-embedded modeling language for convex optimization},
  author={Diamond, Steven and Boyd, Stephen},
  journal={Journal of Machine Learning Research},
  volume={17},
  number={83},
  pages={1--5},
  year={2016}
}

@article{agrawal2018rewriting,
  title={A rewriting system for convex optimization problems},
  author={Agrawal, Akshay and Verschueren, Robin and Diamond, Steven and Boyd, Stephen},
  journal={Journal of Control and Decision},
  volume={5},
  number={1},
  pages={42--60},
  year={2018},
  publisher={Taylor \& Francis}
}
